\documentclass[nohyperref]{article}

\RequirePackage{silence}
\usepackage{microtype}
\usepackage{graphicx}
\usepackage{subfigure}
\usepackage{booktabs} 

\usepackage{hyperref}
\usepackage{amsthm}
\usepackage{thmtools}

\usepackage[accepted]{icml2022}

\usepackage{amsmath}
\usepackage{amssymb}
\usepackage{mathtools}

\usepackage[capitalize,noabbrev]{cleveref}

\theoremstyle{plain}
\newtheorem{theorem}{Theorem}[section]
\newtheorem{proposition}[theorem]{Proposition}
\newtheorem{lemma}[theorem]{Lemma}

\theoremstyle{definition}
\newtheorem{definition}[theorem]{Definition}

\theoremstyle{remark}
\newtheorem{remark}[theorem]{Remark}

\usepackage[textsize=tiny]{todonotes}

\usepackage{centernot}  
\usepackage{xfrac}      
\usepackage{bbm}        

\usepackage{enumitem}   

\usepackage{parskip}    

\def\bfI{\mathbf{I}}

\def\calD{\mathcal{D}}

\def\calF{\mathcal{F}}
\def\calG{\mathcal{G}}
\def\calH{\mathcal{H}}
\def\calI{\mathcal{I}}
\def\calJ{\mathcal{J}}
\def\calK{\mathcal{K}}
\def\calL{\mathcal{L}}

\def\calN{\mathcal{N}}
\def\calO{\mathcal{O}}

\def\calS{\mathcal{S}}

\def\calU{\mathcal{U}}

\def\calX{\mathcal{X}}

\def\calZ{\mathcal{Z}}

\newcommand{\abs}[1]{\left\vert #1\right\vert}

\newcommand{\rbr}[1]{\left(#1\right)}
\newcommand{\sbr}[1]{\left[#1\right]}
\newcommand{\cbr}[1]{\left\{#1\right\}}
\newcommand{\abr}[1]{\left\langle#1\right\rangle}

\def\norm#1{\left\|#1\right\|}

\def\argmax{\mathop{\rm arg\,max}}
\def\argmin{\mathop{\rm arg\,min}}

\def\half{\frac 1 2}

\newcommand{\R}{\mathbb{R}}

\newcommand{\into}{\rightarrow}

\newcommand{\y}{y}
\newcommand{\yk}{y_k}
\newcommand{\ykk}{y_{k+1}}

\newcommand{\x}{u}
\newcommand{\xk}{u_k}

\newcommand{\xkk}{u_{k+1}}

\newcommand{\xbar}{\xbar{w}}

\newcommand{\etak}{\eta_k}

\newcommand{\grad}{\nabla f}

\newcommand{\diag}{\text{diag}}

\usepackage{tikz}
\usepackage{pgfplots}

\usepgfplotslibrary{fillbetween}
\usetikzlibrary{patterns}

\tikzset{
    font={\fontsize{12pt}{12}\selectfont},
}

\pgfplotsset{
    compat=1.5.1,
    primary/.style={color=black, style=solid, line width=1.5pt}, 
    secondary/.style={color=red, style=solid, line width=1.5pt}, 
}

\icmltitlerunning{Fast Convex Optimization for Two-Layer ReLU Networks: Equivalent Model Classes and Cone Decompositions}

\def\smallPDF{}

\begin{document}

\twocolumn[
	\icmltitle{Fast Convex Optimization for Two-Layer ReLU Networks: Equivalent Model Classes and Cone Decompositions}



	\icmlsetsymbol{equal}{*}

	\begin{icmlauthorlist}
		\icmlauthor{Aaron Mishkin}{cs_stanford}
		\icmlauthor{Arda Sahiner}{ee_stanford}
		\icmlauthor{Mert Pilanci}{ee_stanford}
	\end{icmlauthorlist}

	\icmlaffiliation{cs_stanford}{Department of Computer Science, Stanford University}
	\icmlaffiliation{ee_stanford}{Department of Electrical Engineering, Stanford University}

	\icmlcorrespondingauthor{Aaron Mishkin}{amishkin@cs.stanford.edu}

	\icmlkeywords{Machine Learning, ICML}

	\vskip 0.3in
]



\printAffiliationsAndNotice{} 

\begin{abstract}
    We develop fast algorithms and robust software for convex optimization of
    two-layer neural networks with ReLU activation functions.
    Our work leverages a convex reformulation of the standard weight-decay
    penalized training problem as a set of group-$\ell_1$-regularized
    \emph{data-local} models, where locality is enforced by polyhedral cone
    constraints.
    In the special case of zero-regularization, we show that this problem is
    exactly equivalent to \emph{unconstrained} optimization of a convex ``gated
    ReLU'' network with non-singular gates.
    For problems with non-zero regularization, we show that convex gated ReLU
    models obtain data-dependent approximation bounds for the ReLU training
    problem.
    To optimize the convex reformulations, we develop an accelerated proximal
    gradient method and a practical augmented Lagrangian solver.
    We show that these approaches are faster than standard training heuristics
    for the non-convex problem, such as SGD, and outperform commercial
    interior-point solvers.
    Experimentally, we verify our theoretical results, explore the
    group-$\ell_1$ regularization path, and scale convex optimization for
    neural networks to image classification on MNIST and CIFAR-10.
\end{abstract}





\section{Introduction}\label{sec:intro}


\begin{figure}[t]
	\centering
	\ifdefined\smallPDF
		\includegraphics[width=\linewidth]{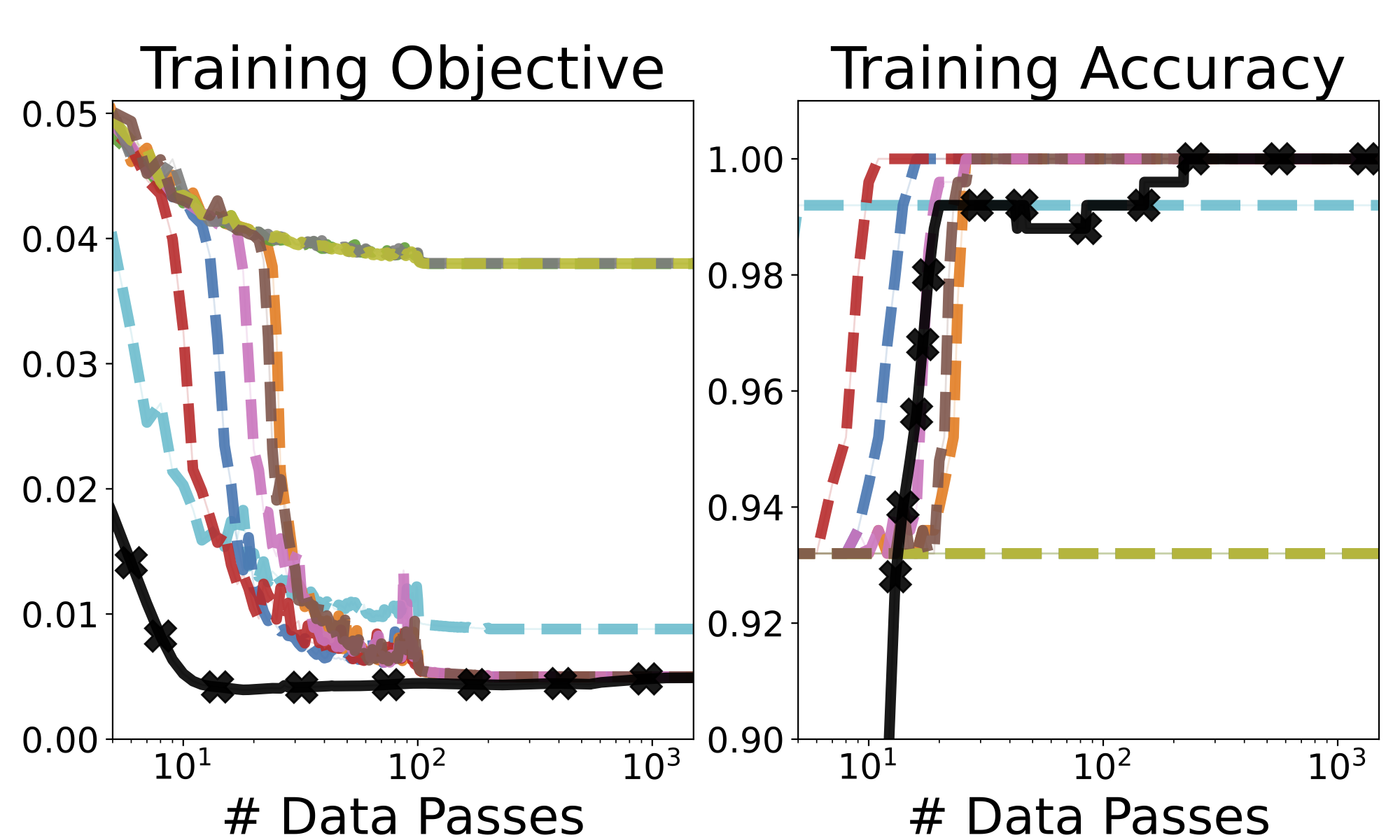}
	\else
		\includegraphics[width=\linewidth]{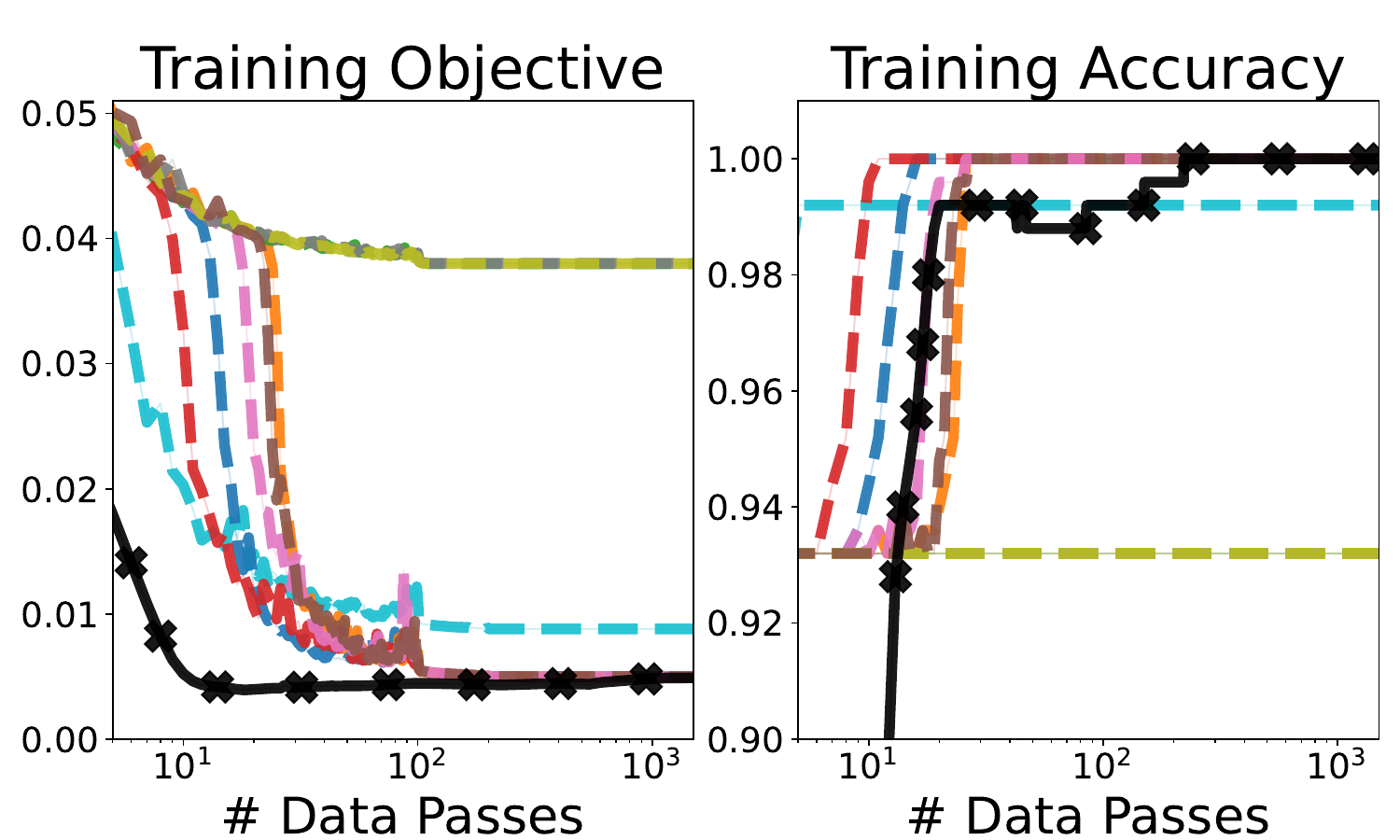}
	\fi
	\vspace{-1.5em}
	\caption{Convex (solid line) and non-convex (dashed) optimization of a two-layer ReLU network for a realizable synthetic classification problem.
		We plot only one run of the convex solver since they are nearly identical and all reach perfect accuracy.
		In contrast, \( 4/10 \) runs of SGD on the non-convex problem converge to sub-optimal stationary points.
	}%
	\label{fig:synthetic-classification}
	\vspace{-0.5cm}
\end{figure}

It is well-known that global optimization of neural networks is NP-Hard~\citep{blum1988npcomplete}.
Despite the theoretical difficulty, highly accurate models are trained in practice using stochastic gradient methods (SGMs)~\citep{bengio2012practical}.
Unfortunately, SGMs cannot guarantee convergence to a local optimum of the non-convex training loss~\citep{ge2015escaping} and existing methods rarely certify convergence to a stationary point of any type~\citep{goodfellow2016deeplearning}.
SGMs are also sensitive to hyper-parameters;
they converge slowly, to different stationary points~\citep{neyshabur2017implicit}, or even diverge depending on the choice of step-size.
Parameters like the random seed complicate replications and can produce model churn, where networks learned using the same procedure give different predictions for the same inputs~\citep{henderson2018deep, bhojanapalli2021reproducibility}.
See \autoref{fig:synthetic-classification} for an example.
For most applications, practitioners use domain knowledge and costly hyper-parameter search to cope with these challenges.

In contrast, we propose to optimize shallow models via \emph{convex reformulations} of the training objective.
Recent work by \citet{pilanci2020convexnn} uses duality theory to show two-layer neural networks with ReLU activations and weight decay regularization may be re-expressed as a linear model with a group-$\ell_1$ penalty and polyhedral cone constraints.
Subsequent research extends this model space, deriving convex formulations for convolutions~\citep{sahiner2020convex,ergen2021implicit,gupta2021exact}, vector-outputs~\citep{sahiner2021vector}, batch normalization \cite{ergen2021demystifying}, generative models \cite{sahiner2021hidden} and deeper networks~\citep{ergen2021beyond, ergen2021revealing}.
However, existing work is largely focused on model classes, rather than leveraging convexification to train neural networks.

This paper develops fast optimization algorithms for two-layer ReLU models by carefully studying the space of equivalent models.
We show that unregularized ReLU networks can be trained by decomposing the solution to an unconstrained generalized linear model (GLM) onto a difference of polyhedral cones.
With non-zero regularization, the same unconstrained problem yields a data-dependent approximation of the optimal solution which differs only in the norm of the model weights.
To fit this GLM, we develop a proximal-gradient method that combines the convex optimization toolbox with GPU acceleration.
We also use this optimizer as a sub-routine for an augmented Lagrangian method that quickly and robustly trains ReLU networks via the (constrained) convex reformulation.
Our deterministic optimizers give both convergence and optimality guarantees.

To summarize, our main contributions are the following:
\begin{itemize}
	\item A new class of \emph{unconstrained} convex optimization problems which are equivalent to training an unregularized ReLU model and approximation guarantees for the case of non-zero regularization.

	\item An accelerated proximal-gradient method for this unconstrained problem that improves the complexity of computing a global optimum from \(O(1/\epsilon^2)\) to \(O(1/\sqrt{\epsilon})\) iterations compared to subgradient methods.

	\item An augmented Lagrangian method for the constrained convex reformulation which uses our unconstrained solver as a sub-routine and outperforms commercial interior-point software such as MOSEK~\citep{mosek}.

	\item Extensive experiments which validate our theoretical results, carefully explore the properties of our optimization methods, and scale convex optimization for ReLU networks to MNIST and CIFAR-10.
\end{itemize}

Quality software is key to practical use of our methods.
As such, we also provide \texttt{scnn}, an open-source package for training neural networks by convex optimization.\footnote{\small \url{https://github.com/pilancilab/scnn}}

\subsection{Related Work}

Our work combines ideas from the literature on convex neural networks, accelerated methods, and constrained solvers.

\textbf{Convex Neural Networks}:
There have been repeated attempts to develop convex neural networks.
\citet{bengio2006convex} view two-layer neural networks as convex models, but their work requires the first-layer weights to be fixed.
Similarly, extreme learning machines \citep{huang2005elm} obtain a convex problem by using a random first-layer;
these models can obtain zero training error for over-parameterized problems
\citep{woodworth2020rich}, but do not learn parsimonious latent representations as in our approach.

\citet{bach2017breaking} analyze infinite-width two-layer networks; these methods are not implementable, but may be viewed as convex problems.
Other research considers the separate problem of neural networks for which the prediction function is convex~\citep{amos2017input, sivaprasad2021curious}.

In concurrent work, \citet{bai2022efficient} consider training two-layer
ReLU networks via convex reformulations using ADMM.
Their approach requires solving a linear system at each iteration,
or uses coordinate descent to solve the ADMM sub-problems.
In practice, our solvers scale to larger datasets
and allow for more activation patterns.

\textbf{Accelerated Proximal Gradient}:
\citet{beck2009fista, nesterov2007proximalgradient} were the first to extend optimal gradient methods~\citep{nesterov1983method} to composite problems.
Work since then includes extensions to stochastic~\citep{schmidt2011convergence} and non-convex~\citep{li2015accelerated} optimization.
See~\citet{parikh2014proximalsurvery} for a survey of proximal algorithms, including proximal gradient.

\textbf{Augmented Lagrangian Methods}: The convergence theory was initially developed by Rockafellar~\citep{rockafellar1976augmented, rockafellar1976monotone}.
More recent work includes practical guidelines~\citep{birgin2014pratical} and acceleration techniques~\citep{kang2015inexact}.
See \citet{bertsekas2014constrained} for exhaustive theoretical developments.

\section{Convex Reformulations}\label{sec:convex-formulations}

Let \( X \in \R^{n \times d} \) be a data matrix and \( y \in \R^n \) the associated targets.
We are interested in two-layer ReLU networks,
\begin{equation*}
	h_{W_1, w_2}(X) = \sum_{i=1}^m \rbr{X W_{1i}}_+ w_{2i},
\end{equation*}
where \( W_1 \in \R^{m \times d} \), \( w_2 \in \R^m \) are the weights of the first and second layers, \( m \) is the number of hidden units, and \( \rbr{\cdot}_+ = \max\cbr{\cdot, 0} \) is the ReLU activation.
Fitting \( h_{W_1, w_2} \) by minimizing convex loss \( \calL \) with weight decay (\( \ell_2 \)) regularization leads to the optimization problem (NC-ReLU),\vspace{-1em}
\begin{equation}\label{convex-forms:eq:non-convex-relu-mlp}
	\min_{W_1, w_2} \! \calL\Big(h_{W_1, w_2}(X), y \Big) + \frac{\lambda}{2}\sum_{i=1}^m \norm{W_{1i}}_2^2 + |w_{2i}|^2,\!
\end{equation}
where \( \lambda \geq 0 \) is the regularization strength.
While Problem~\ref{convex-forms:eq:non-convex-relu-mlp} is non-convex,
\citet{pilanci2020convexnn} show that there is an equivalent convex optimization problem with the same optimal value if \( m \geq m^* \) for some \( m^* \leq n + 1 \). Furthermore, \citet{wang2021hidden} showed that all optimal solutions to \eqref{convex-forms:eq:non-convex-relu-mlp} can be found via the convex problem.

\subsection{Sub-Sampled ReLU Convex Programs}\label{sec:relu-models}

The convex reformulation for the NC-ReLU objective is based on ``enumerating'' the possible activations of a single neuron in the hidden layer.
The activation patterns a ReLU neuron \( \rbr{X w}_{+} \) can take for fixed \( X \) are described by
\[ \calD_X = \cbr{D = \diag(\mathbbm{1}(X u \geq 0)) : u \in \R^d}, \]
which grows as \(|\calD_X| \in O(r(n/r)^r)\) for \(r := \text{rank}(X)\) \citep{pilanci2020convexnn}.
For \( D_i \in \calD_{X} \), the set of vectors \( u \) which achieve the corresponding activation pattern, meaning \( D_i X u = \rbr{X u}_+ \), is the convex cone,
\[
	\calK_i = \cbr{u \in \R^d : (2D_i - I) X u \succeq 0}.
\]
For any subset \( \tilde \calD \subseteq \calD_X \), we define the sub-sampled convex optimization problem (C-ReLU):
\begin{equation}\label{convex-forms:eq:convex-relu-mlp}
	\begin{aligned}
		\!\!\! \min_{v, w} \, & \calL\Big(\!\!\sum_{D_i \in \tilde \calD}\!\! D_i X (v_i \!-\! w_i), y\Big)\!+\!\lambda\!\sum_{D_i \in \tilde \calD}\!\norm{v_i}_2\!+\!\norm{w_i}_2. \\
		                      & \text{s.t.} \quad v_i, w_i \in \calK_i
	\end{aligned}
\end{equation}

\citet{pilanci2020convexnn} prove NC-ReLU and C-ReLU are equivalent using linear semi-infinite duality theory~\citep{ goberna2002semi}.
However, this result requires \( m \geq m^* \) and the full enumeration of the activations of a neuron: \( \tilde \calD = \calD_X \).
In practice, learning with \( \calD_{X} \) is computationally infeasible except for special cases where the data are low rank.
By introducing sub-sampled models, we relax the dependencies on \( m^* \) and \( \calD_X \) to simple inclusions involving \( \tilde \calD \).





\begin{restatable}{theorem}{convexDualityFree}\label{thm:duality-free}
	Suppose \( \rbr{W_1^*, w_2^*} \) and \( \rbr{v^*, w^*} \) are global minima of the NC-ReLU~\eqref{convex-forms:eq:mi-reformulation} and C-ReLU~\eqref{convex-forms:eq:convex-relu-mlp} problems, respectively.
	If the number of hidden units satisfies
	\[ m \geq b := \sum_{D_i \in \tilde \calD} \abs{\cbr{v^*_i : v^*_i \neq 0} \cup \cbr{w^*_i : w^*_i \neq 0} }, \]
	and the optimal activations are in the convex model,
	\[ \cbr{\text{\emph{diag}}\rbr{X W_{1i}^* \geq 0 : i \in [m]} } \subseteq \tilde \calD, \]
	then the two problems have same the optimal value.
\end{restatable}
See \cref{app:convex-forms} for proof.
The advantages of this theorem over existing results are (i) the simple and duality-free proof, and (ii) the dependence on \( \tilde \calD \), which we show in \cref{sec:experiments} can be much smaller than \( \calD_X \) while still performing comparably to NC-ReLU.
\cref{thm:duality-free} also reveals that \( m^* \) is determined by the number of active ``neurons'' at the optimal solution of the full C-ReLU problem with \( \tilde \calD = \calD_X \).




\subsection{Unconstrained Relaxation: Gated ReLUs}

Solving C-ReLU using scalable first-order methods typically requires projecting on \( \calK_i \), which is an expensive quadratic program in the general case.
To circumvent this, we consider the following unconstrained relaxation (C-GReLU):
\begin{equation}\label{convex-forms:eq:unconstrained-relu-mlp}
	\min_{u} \; \calL\Big(\sum_{D_i \in \tilde \calD} D_i X u_i, y\Big)+ \lambda \sum_{D_i \in \tilde \calD} \norm{u_i}_2.
\end{equation}
At first look, this problem is a high-dimensional GLM with group-\( \ell_1 \) regularization.
In fact, C-GReLU is the convex re-formulation of another neural network optimization problem.
Let \( \calG \subset \R^d \) and consider the model,
\[
	h_{W_1, w_2}(X) = \sum_{g_i \in \calG} \phi_{g}(X, W_{1i}) w_{2i},
\]
where \( \phi_{g}(X, u) = \text{diag}(\mathbbm{1}(Xg \geq 0)) X u \) is a ``gated ReLU'' activation function with fixed gate vector \( g \) \citep{fiat2019decoupling}.
C-GReLU is equivalent to training this gated ReLU network.

\begin{restatable}{theorem}{unconstrainedEquivalence}\label{thm:unconstrained-equivalence}
	Let \( g_i \in R^d \) such that \(\text{diag}(X g_i \geq 0) = D_i \) and \( \tilde \calG = \cbr{g_i : D_i \in \tilde \calD} \).
	Then, C-GReLU is equivalent to the following gated ReLU problem (NC-GReLU):
	\begin{equation}\label{convex-forms:eq:gated-relu-mlp}
		\!\min_{W_1, w_2}\! \calL\Big(\!\!\sum_{g_i \in \tilde \calG }\!\!\phi_{g_i}(X, W_{1i})w_{2i}, y\Big)\!+\!\frac{\lambda}{2} \!\! \sum_{g_i \in \tilde \calG}\!\!\norm{W_{1i}}_2^2 + w_{2i}^2. \!\!
	\end{equation}
\end{restatable}
See \cref{app:convex-forms} for proof.
We can use \cref{thm:unconstrained-equivalence} to fit Gated ReLU networks by (much easier) unconstrained minimization.
However, as the next section shows, we can also leverage C-GReLU to approximate or exactly solve the original ReLU problem.


\section{Equivalence of ReLU and Gated ReLU}\label{sec:relu-grelu-equivalence}

This section builds upon our convex re-formulations to show that the full C-ReLU problem is equivalent to a sub-sampled C-GReLU problem up to the norm of their optimal solutions.
As a consequence, we give an approximation algorithm for ReLU networks that first computes the solution to C-GReLU and then solves an auxiliary \emph{cone decomposition} problem.
The cone decomposition can be formulated as a linear program (LP) or second-order cone program (SOCP), and admits a closed form solution when \( X \) is full row-rank.
Before presenting these fully-general results, we study the unregularized setting, where we show the approximation is exact.
All proofs are deferred to \cref{app:relu-grelu-equivalence}.

\begin{figure}
    \centering


\begin{tikzpicture}[scale=1,
		declare function={
				cone_u(\x)= -\x/3;
				cone_l(\x)= -4*\x/5;
			}
	]
    \begin{axis}[width=1.1\linewidth, height=5cm,
			axis lines=center, yticklabels={,,}, xticklabels={,,},
			ymin=-4, ymax=4, ytick={-5,...,5}, ylabel=$$, x axis line style={-},
				xmin=-6, xmax=6, xtick={-5,...,5}, xlabel=$$, y axis line style={-},
		]
		\addplot[name path=cone_u, domain=-6:6, samples=100, line width=1pt]{cone_u(x)};
		\addplot[name path=cone_l, domain=-6:6, samples=200, line width=1pt]{cone_l(x)};

		\addplot fill between[
				of = cone_u and cone_l,
				split, 
				every even segment/.style = {fill=blue, fill opacity=0.3},
				every odd segment/.style  = {fill=teal, fill opacity=0.3}
			];

		\node[circle, fill, inner sep=1pt] at (axis cs:0,0) {};

		\node[label={0:$u_i$}, circle, fill, inner sep=1.8pt] (u) at (axis cs:2,1) {};
		\node[label={90:$v_i$}, circle, fill, inner sep=1.8pt] (v) at (axis cs:25/7+2, -25/21 - 2/3) {};
		\node[label={90:$w_i$}, circle, fill, inner sep=1.8pt] (w) at (axis cs:-25/7, 20/7) {};

		\node[label={0:$\calK_i$}] at (axis cs:4,-2.5) {};
		\node[label={180:$-\calK_i$}] at (axis cs:-3.75,2.5) {};
        
		\draw [->, dashed, draw=red, line width = 0.4mm] (u) edge (w);
		\draw [->, dashed, draw=red, line width = 0.4mm] (u) edge (v);
	\end{axis}

\end{tikzpicture}%
    \caption{An illustration of the Cone Decomposition (CD) procedure: \( u_i \) is decomposed onto the Minkowski difference \( \calK_i - \calK_i \).}%
    \label{fig:cone-decomp}
    \vspace{-0.3cm}
\end{figure}
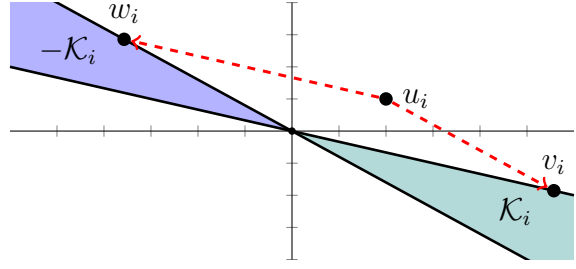

Let \( \lambda = 0 \) and consider the C-GReLU problem~\eqref{convex-forms:eq:unconstrained-relu-mlp}.
For each \( D_i \in \tilde \calD \), we seek to decompose the optimal data-local models as \( u^*_i = v_i - w_i \in \calK_i - \calK_i \).
If these decompositions exist, collecting them into \( (v, w) = \cbr{(v_i, w_i)} \) gives a feasible point for the C-ReLU problem with the same optimal objective value as C-GReLU.
The next proposition gives sufficient conditions on the data for this to happen.
\begin{restatable}{proposition}{nonSingularCones}\label{prop:non-singular-cones}
    If \( X \) is full row-rank, then \( \calK_i - \calK_i = \R^d  \) for every \( D_i \in \calD_{\calX} \).
    As a result, the C-ReLU, C-GReLU, NC-ReLU, and NC-GReLU problems are all equivalent.
\end{restatable}
Unfortunately, \cref{prop:non-singular-cones} does not extended to \( n > d \);
in \cref{prop:col-rank-counter-example}, we give full-rank \( X \) for which some \( \calK_i \) is contained in a subspace of \( \R^d \), implying \( \calK_i - \calK_i \subset \R^d \).
We call such cones (and associated gate vectors) \emph{singular}.
\begin{restatable}{proposition}{coneContainment}\label{prop:cone-containment}
    Suppose \( \calK_i \) is singular for \( D_i \in \calD_{X} \).
    Then \( \exists D_j \in \calD_{X} \) such that \( \calK_j - \calK_j = \R^d \) and \( \calK_i \subset \calK_j \).
\end{restatable}
That is, every singular cone is contained within a non-singular cone.
As a result, we show that these ``bad'' cones,
\begin{equation}\label{eq:singular-cones}
    \calS(\tilde \calD_X) = \cbr{ D_i \in \tilde \calD_X : \calK_i - \calK_i \subset \R^d },
\end{equation}
can be safely ignored when forming the convex programs.
\begin{restatable}{theorem}{coneElimination}\label{thm:cone-elimination}
    Let \( \tilde \calD \subseteq \calD_X \) and \( \lambda \geq 0 \).
    Then the C-ReLU problem with \( \tilde \calD \) is equivalent
    to the C-ReLU problem with \( \tilde \calD \setminus \calS(\tilde \calD) \).
    If \( \lambda = 0 \), then both problems are equivalent to the sub-sampled
    C-GReLU problem with \( \tilde \calD \setminus \calS(\tilde \calD) \).
\end{restatable}
Choosing \( \tilde \calD = \calD_X \) shows that the full C-ReLU
problem is exactly equivalent to the unconstrained C-GReLU problem without
singular gate vectors.
\cref{alg:cone-decomp} provides a template for training ReLU networks by leveraging cone decompositions and \cref{thm:cone-elimination}.
Note that \( \tilde \calD \) is generated by randomly sampling gate vectors \( g_i \sim \calN(0, I) \).
This is sufficient to recover the C-ReLU problem as \cref{thm:cone-elimination} implies singular cones, for which the sampling probability is zero, don't contribute to the solution.

\subsection{Approximating ReLU by Cone Decompositions}

We have seen that decomposing \( u_i^* = v_i - w_i \) allows us to map the C-GReLU problem into the C-ReLU problem.
However, triangle inequality shows \( \norm{u_i^*}_2 \leq \norm{v_i}_2 + \norm{w_i}_2 \), meaning the cone decomposition can only increase the norm of the model (see \cref{fig:cone-decomp}).
For \( \lambda > 0 \), this increases the penalty term in the objective (Eq.~\ref{convex-forms:eq:convex-relu-mlp}), although the loss \( \calL \) is unchanged.
This section develops cone decomposition algorithms for which we know the blow-up of the norm is not too large.
As a result, we obtain approximation guarantees for solving C-ReLU by solving C-GReLU.

In what follows, \( \calK = \cbr{w : (2 D - I) X w \succeq 0} \) denotes a non-singular cone, \( \tilde X = (2 D - I) X \), and \( \kappa(A) \) is the ratio of the largest and smallest \emph{non-zero} singular values of \( A \).
Our first result gives conditions for the existence of a closed-form
decomposition.
\begin{restatable}{proposition}{closedFormDecomp}\label{prop:closed-form-decomp}
    Suppose \( X \) is full row-rank.
    If \( \calI = \cbr{i \in [n] : \abr{\tilde x_i, u} < 0 } \), then for every \( u \in \R^d \),
    \vspace{-1ex}
    \[
        u = \rbr{u + w} - w, \text{where} \,  w = - \tilde X_{\calI}^\dagger \tilde X_{\calI} u,
    \]
    \vspace{-1ex}
    is a valid decomposition onto \( \calK - \calK \) satisfying,
    \[
        \norm{u + w}_2 + \norm{w}_2 \leq 2 \norm{u}_2.
    \]
\end{restatable}
In general, we cannot hope for constant approximations since \( n \gg d \) implies
the cones \( \calK \) are very ``narrow''.
\begin{restatable}{proposition}{worstCaseNorm}\label{prop:worst-case-norm}
    There does not exist a decomposition \( u = v - w \), where \( v, w \in \calK \), such that
    \[
        \norm{v}_2 + \norm{w}_2 \leq C \norm{u},
    \]
    holds for an absolute constant \( C \).
\end{restatable}
\begin{algorithm}[tb]
    \caption{Solving C-ReLU by Cone Decomposition}
    \label{alg:cone-decomp}
    \begin{algorithmic}
        \STATE {\bfseries Input:} data \( (X, y) \), \( \lambda \geq 0 \), num. samples \( p \), objective \( R \).
        \STATE \textbf{Sample}: \( \tilde \calD \!=\! \cbr{\diag(\mathbbm{1}(X g_i \geq 0)) :  g_i \!\sim \!\calN(0, I), i \in [p] } \)
        \STATE \textbf{Solve C-GReLU}:
        \STATE \hspace{1em} \( u^* \in \argmin_{u} \calL(\sum_{\tilde \calD} D_i X u_i, y) + \lambda \sum_{\tilde \calD} \norm{u_i}_2\)
        \STATE \textbf{Solve Cone Decomposition}:
        \STATE \hspace{1em} \( \bar v, \bar w \in \argmin_{v, w} \cbr{ R(v, w) : u_i^* = v_i - w_i, i \in [p] } \)
        \STATE \textbf{Return}: \((\bar v, \bar w)\)
    \end{algorithmic}
\end{algorithm}
\vspace{-1ex}
When \( X \) is not full row-rank, we can solve
\begin{equation}\label{eq:decomposition-program}
    \begin{aligned}
        \textbf{CD}: \quad \min_{v, w \in \calK} \cbr{ R(v, w) : v - w = u },
    \end{aligned}
\end{equation}
where \( R : \R^{d \times d} \mapsto \R \) is some loss function.
Taking \( R(v, w) = 0 \) reduces to a linear feasibility problem which can be handled by off-the-shelf LP solvers.
Choosing \( R(v,w) = \norm{v}_2 + \norm{w}_2 \) yields a second-order cone program (SOCP) for which we have the following guarantee.
\begin{restatable}{proposition}{socpDecomp}\label{prop:socp-decomp}
    For every \( u \in \R^d \), if \( \rbr{\bar v, \bar w} \) is a solution to the cone-decomposition program~\eqref{eq:decomposition-program} with
    \( R(v,w) = \norm{v}_2 + \norm{w}_2 \), then there exists \( \calJ \subseteq [n] \) such that
    \[
        \norm{\bar v}_2 + \norm{\bar w}_2 \leq \rbr{1 + 2 \kappa(\tilde X_{\calJ})} \norm{u}_2.
    \]
\end{restatable}

Note that the general setting incurs a penalty of \( \kappa(\tilde X_{\calJ}) \) compared to \cref{prop:closed-form-decomp}.
Intuitively, this term measures the narrowness of \( \calK \) and the difficulty of the decomposition.
Combining \cref{prop:socp-decomp} with \cref{thm:cone-elimination} gives our main approximation result.
\begin{restatable}{theorem}{approxResult}\label{thm:approx-result}
    Let \( \lambda \geq 0 \) and let \( p^* \) be the optimal value of the full
    C-ReLU problem with training set \( (X, y) \).
    There exists \( \calJ \subseteq [n] \) such that the C-GReLU problem with
    patterns \( \calD_X \setminus \calS(\calD_X) \), minimizer \( u^* \),
    and optimal value \( d^* \) satisfies,
    \[
        d^*
        \leq p^*
        \leq d^*
        + 2 \lambda \kappa(\tilde X_{\calJ}) \sum_{D_i \in \tilde \calD} \norm{u_i^*}_2.
    \]
\end{restatable}
As a consequence of \cref{thm:approx-result}, \cref{alg:cone-decomp} is guaranteed to approximate the C-ReLU problem if \( R(v,w) = \norm{v}_2 + \norm{w}_2 \) and \( p \) is sufficiently large.
As \( \lambda \into 0 \), this result smoothly recovers \cref{thm:cone-elimination}, implying we can control the approximation by adjusting the regularization.

\begin{figure}
    \centering


\begin{tikzpicture}[scale=1,
	]
	\begin{axis}[width=1.1\linewidth, height=5cm,
			axis lines=none,  
			yticklabels={,,}, xticklabels={,,},
			ymin=-0.2, ymax=10.2, x axis line style={-},
			xmin=-0.2, xmax=20.2, y axis line style={-},
		]

		\filldraw[color=blue!60, fill=blue!5, line width=0.4mm](axis cs:0,5.8) rectangle (axis cs:20, 10);
		\filldraw[color=red!60, fill=red!5, line width=0.4mm](axis cs:0,0) rectangle (axis cs:20, 4.2);

		\filldraw[line width=0.4mm, fill=white](axis cs:1,1) rectangle (axis cs:8, 3.2) node[pos=.5] {NC-GReLU};
		\filldraw[line width=0.4mm, fill=white](axis cs:12,1) rectangle (axis cs:19, 3.2) node[pos=.5] {NC-ReLU};

		\filldraw[line width=0.4mm, fill=white](axis cs:1,6.8) rectangle (axis cs:8, 9) node[pos=.5] {C-GReLU};
		\filldraw[line width=0.4mm, fill=white](axis cs:12,6.8) rectangle (axis cs:19, 9) node[pos=.5] {C-ReLU};

		\draw [<->, solid, draw=black, line width = 0.6mm] (axis cs:4.5,3.2) -- (axis cs:4.5,6.8) node[right, pos=0.5] {\small Sol. Map};

		\draw [<->, solid, draw=black, line width = 0.6mm] (axis cs:15.5,3.2) -- (axis cs:15.5,6.8)  node[right, pos=0.5] {\small Sol. Map};

        \draw [<-, solid, draw=black, line width = 0.6mm] (axis cs:6,9) to [bend left=15] (axis cs:14, 9);

		\draw [->, solid, draw=orange, line width = 0.6mm] (axis cs:8,7.9) -- (axis cs:12,7.9);
		\node[align=center] at (axis cs:10.1, 7.8) {\small Cone\\ \small Decomp.};
	\end{axis}

\end{tikzpicture}%
    \caption{Summary of equivalences between convex (blue) and non-convex (red) neural
        network training problems with gated ReLU (left) and ReLU (right)
        activations.
        The convex programs C-GReLU and C-ReLU are equivalent to the standard
        non-convex training problems NC-GReLU and NC-GReLU and are related to
        each other via the cone decomposition procedure.
    }%
    \label{fig:relations}
    \vspace{-0.3cm}
\end{figure}
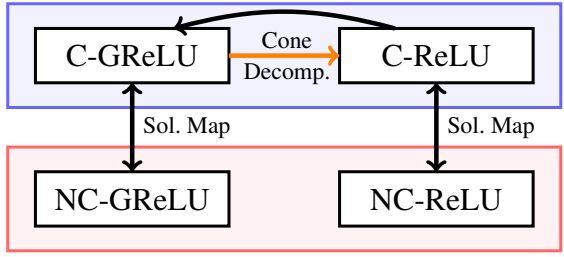


\section{Efficient Global Optimization}\label{sec:}

We now have two options for convex optimization of ReLU models: directly tackling the C-ReLU problem or solving C-GReLU and a cone decomposition program (see \cref{fig:relations}).
This section develops efficient and scalable methods for both approaches.
For simplicity, we assume \(\mathcal{L}\) is squared loss; our results are easily extended to other loss functions.

\subsection{Solving the Gated ReLU Problem}\label{sec:r-fista}

Our goal is a fast and reliable method for the C-GReLU problem even when \( \tilde \calD \) is very large.
To be practical, it should benefit from GPU acceleration, provide convergence certificates, and be ``tuning-free''.
To be theoretically satisfying, it should come with complexity guarantees.

Our starting place is the observation that C-GReLU is exactly the classic \emph{group lasso} with basis expansion,
\[ M(X) = [D_1 X \, D_2 X \, \cdots \, D_{\abs{\tilde \calD}} X ]. \]
A naive approach to huge-scale group lasso is the stochastic subgradient method;
this approach benefits from auto-differentiation engines such as PyTorch~\citep{paszke2019pytorch} and TensorFlow~\citep{abadi2016tensorflow} and is simple to code.
However, subgradient methods require decreasing step-sizes to converge and are extremely slow --- they require \( O(\epsilon^{-2}) \) iterations to to compute an \( \epsilon \)-optimal point.

Instead, we use the composite structure of the objective as a sum of a convex quadratic \( f(u) = \|\sum_{D_i \in \tilde \calD} D_i X u_i - y\|_2^2 \) and the non-smooth penalty \( g(u) = \lambda \sum_{D_i \in \tilde \calD} \norm{u_i}_2 \).
The FISTA algorithm~\citep{beck2009fista} is an accelerated method that treats \( g \) exactly using the iteration,
\begin{align}
	\xkk
	 & = \argmin_{y} Q_{\yk, \etak}(y) + g(y)               \label{eq:prox-gradient-mm} \\
	\ykk
	 & = \xkk + \frac{t_k - 1}{t_{k + 1}} \rbr{\xkk - \xk} \nonumber,
\end{align}
where \( t_{k + 1} = (1 + \sqrt{1 + 4 t_k^2}) / 2 \, \) and,
\begin{equation*}
	\begin{aligned}
		Q_{\xk, \etak}(y) & = f(\xk)\! + \!\abr{\grad(\xk), y \!- \!\xk}\!
		+\! \frac{1}{2 \etak}\norm{y \!- \!\xk}_2^2,
	\end{aligned}
\end{equation*}
majorizes \( f \) as long as \( \etak \leq \lambda_{\text{max}}(M^\top M)^{-1} \).
Using the convergence guarantee for FISTA when \( f, g \) are convex and $f$ is Lipschitz smooth~\citep{duchi2009forwardbackwardsplitting, beck2009fista, nesterov2007proximalgradient} gives the complexity of \emph{global} optimization of the NC-GReLU problem.

\begin{restatable}{theorem}{gatedComplexity}\label{thm:gated-complexity}
	Let \( \rbr{W_{1}^*, w_2^*} \) be the minimum-norm global minimizer of the NC-GReLU problem with gates \( \calG \).
	Then, we can compute an \( \epsilon \)-optimal point \( \rbr{W_{1\epsilon}, w_{2\epsilon}} \) in iterations
	\begin{equation*}
		\begin{aligned}
			T & \leq \big(2 \epsilon^{-1} \lambda_{\text{max}}\rbr{M^\top M}\sum_{D_i \in \tilde \calD} \norm{W_{1i}^* w_{2i}^*}_2^2\big)^{1/2}.
		\end{aligned}
	\end{equation*}
\end{restatable}

Proof in \cref{app:optimization}.
\cref{thm:gated-complexity} can also be expressed directly in terms of the C-GReLU problem by using a mapping between minimizers of the convex and non-convex formulations.
Data normalization gives \( \lambda_{\text{max}}\rbr{M^\top M} \leq d \cdot |\tilde \calD |  \), which is fully polynomial when \( \text{rank} (X) \) is constant (see Appendix~\ref{app:data-normalization}).
Such a condition holds for convolutional networks with fixed filter sizes~\citep{pilanci2020convexnn}.

\subsubsection{Developing an Efficient Optimizer}\label{sec:efficient-fista}

In theory, it is sufficient to run FISTA with small enough step-size to obtain \cref{thm:gated-complexity}, but this approach works poorly in practice.
Additional enhancements are required for fast and reliable solvers.

\textbf{Line-Search}:
Constant step-sizes converge slowly, so we use a line-search with the test condition proposed by \citet{beck2009fista}:
\begin{equation}\label{unconstrained:eq:line-search-cond}
	f(\xkk(\etak)) \leq Q_{\yk, \etak}(\xkk(\etak)).
\end{equation}
Computing this condition requires evaluating \( f(\xkk(\etak)) \), but does not need additional gradient evaluations like the alternative proposed by~\citet{nesterov2007proximalgradient}.
Simple backtracking along \( \xkk - \xk \)  works poorly and does not converge;
instead, we probe the arc of solutions to~\eqref{eq:prox-gradient-mm} by reducing the step-size.
As evaluating the proximal operator is slower than backtracking, it is important to initialize \( \etak \) effectively.

\textbf{Initializing the Step-size}:
Warm-starting with \( \etak = \eta_{k-1} \) can lead to overly-small steps, particularly later in optimization.
An alternative is \emph{forward-tracking} as \( \eta_k = \alpha \eta_{k-1} \) for \( \alpha > 1 \) \citep{fridovich2019choosing}.
This can partially adapt to local Lipschitz smoothness of \( f \), but may also lead to unnecessary evaluations of the proximal operator.
Instead, we check the tightness of~\eqref{unconstrained:eq:line-search-cond} before forward-tracking~\citep{liu2009lassplore}.
Let \( l_{\yk}(\x) = f(\yk) + \abr{\grad(\yk), \x - \yk} \) and
\begin{equation}\label{unconstrained:eq:upper-bound-gap}
	\omega_k := \frac{\norm{\xk - \y_{k-1}}_2^2}{2 \eta_{k-1} \rbr{f(\xk) - l_{\y_{k-1}}(\xk)}},
\end{equation}
to get \( \eta_{k} = \eta_{k-1} + (1 - \alpha) \eta_{k-1} \mathbbm{1}(\omega_k \geq c)\);
\( c = 1 \) obtains forward-tracking, while \( c \gg 1 \) gives a conservative strategy.

\textbf{Restarts}:
Resetting \( (\yk, t_k) \gets (\xk, 1) \) in the middle of optimization is called \emph{restarting}.
Restarting methods adapt to strong convexity and can attain a fast linear rate of convergence~\citep{nesterov2007proximalgradient, allen2014linear}.
Although C-GReLU is not strongly-convex, restarts can allow FISTA to adapt to local curvature~\citep{giselsson2014restart}.
We restart FISTA when
\( \abr{\xkk - \xk, \xkk - \yk} > 0 \) --- that is, \( \xkk \) is not a descent step with respect to the proximal gradient mapping~\citep{odonoghue2015restarts}.

\textbf{Data Normalization}: The proximal step~\eqref{unconstrained:eq:line-search-cond} is equivalent to composition of a gradient update with the group soft-thresholding operator.
Thresholding is sensitive to rounding errors in computation of the gradient and,
since errors accumulate in the ``memory'' \( y_k \), it is critical to improve
the conditioning of this computation.
Appendix~\ref{app:data-normalization} describes a simple data transformation which works well in practice.

Combining these elements together gives an efficient algorithm for C-GReLU which we call R-FISTA.


\subsection{Tractable Cone Decompositions}\label{sec:cone-decomp}

Training a ReLU network using \cref{alg:cone-decomp} requires solving a large-scale LP or SOCP.
Empirically, the complexity of solving these problems with commercial software is similar to directly solving C-ReLU (see \cref{table:cone-decomp}).
Instead, we propose an \emph{approximate} decomposition procedure which
can be solved efficiently using R-FISTA.

Manipulating the cone decomposition \( v - w = u \), \( v, w \in \calK \),
we obtain the equivalent conditions \( \tilde X w \geq (-\tilde X u)_+ \) and \( v = u + w \).
Given
\( \rho \geq 0 \), and \( b = (-\tilde X u)_+ \),
the regularized one-sided quadratic
\begin{equation}\label{eq:cd-approx}
	\textbf{CD-A}: \, \min_{w} \half
	\|
	(b - \tilde X w)_+
	\|_2^2
	+ \rho \norm{w}_2,
\end{equation}
approximates the exact cone-decomposition as follows:
\begin{restatable}{proposition}{approxDecompBounds}\label{prop:approx-decomp-bounds}
	Suppose \( \tilde w \) is a minimizer of \eqref{eq:cd-approx} and let \( \tilde v = u + \tilde w \).
	If \( X \) is full row-rank, then
	\[
		\|(\tilde X \tilde w)_-\|_2 + \|(\tilde X \tilde v)_-\|_2
		\leq \frac{2\rho}{\sigma_{\text{min}}(\tilde X)}.
	\]
	Furthermore, if \( \rho > 0 \), then the norm bound in \cref{prop:socp-decomp} also holds for the approximate solution \( (\tilde v, \tilde w) \).

	Alternatively, suppose \( X \) is not full row-rank.
	As \( \rho_k \rightarrow 0 \), every convergent subsequence of \( (\tilde v_k, \tilde w_k) \) is a feasible cone decomposition.
	Moreover, at least one such sequence exists.
\end{restatable}
Proof in \cref{app:optimization}, where we also provide \cref{prop:approx-decomp-submatrices},
an alternative characterization in terms of sub-matrices \( \tilde X_{\calJ} \).
\autoref{prop:approx-decomp-bounds} shows it is straightforward to control the quality of the approximation by tuning \( \rho \).
In practice, we find CD-A with \( \rho \approx 10^{-10} \) yields competitive performance and is easily solved using R-FISTA.

\subsection{Solving the ReLU Problem}\label{sec:al-method}

The main difficulty in solving C-ReLU is the constraints.
Interior point methods~\citep{nesterov1994interior} and specialized conic solvers~\citep{odonoghue2016scs} can handle \( \calK_i \), but such methods require second-order information or repeated linear-system solves and scale poorly in both \( n \) and \( d \).
Instead, we develop an augmented Lagrangian (AL) method that uses R-FISTA as a sub-routine.

Recall \cref{thm:cone-elimination} established a sub-sampled problem equivalent to the full C-ReLU problem for which each \( \calK_i \) is non-singular.
These cones have an interior point if and only if they are non-singular (see \cref{lemma:affine-characterization}), implying the sub-sampled problem is strictly feasible and satisfies strong duality.
Letting \( \gamma, \zeta \in \R^{|\tilde \calD| \times n} \) be estimates of the optimal Lagrange multipliers, the augmented Lagrangian for~\eqref{convex-forms:eq:convex-relu-mlp} is
\begin{equation}\label{eq:augmented-lagrangian}
	\begin{aligned}
		\!\!\!\!\calL_\delta & (v,\!w,\!\gamma,\!\zeta)\!:=\!(\delta / 2)\!\!\sum_{D_i \in \tilde \calD}\!\!\big[\|(\gamma_i / \delta\!-\! \tilde X_i v_i)_+\|_2^2 \\
		                     & \hspace{2em} + \|(\zeta_i / \delta - \tilde X_i w_i)_+\|_2^2 \big] + F(v,w),
	\end{aligned}
\end{equation}
where \( F(v, w) \) is the primal objective and \( \tilde X_i = (2 D_i - I) X \).
Eq.~\ref{eq:augmented-lagrangian} is a penalty method and can recover an optimal primal-dual pair from \( \rbr{v_k, w_k} \in \argmin \calL_{\delta_k}(v, w, 0, 0) \) as \( \delta_k \into \infty \) \citep{nocedal1999numerical}.
However, choosing \( \delta_k \) is challenging in practice.

\begin{figure*}[t]
	\centering
	\includegraphics[width=0.9\linewidth]{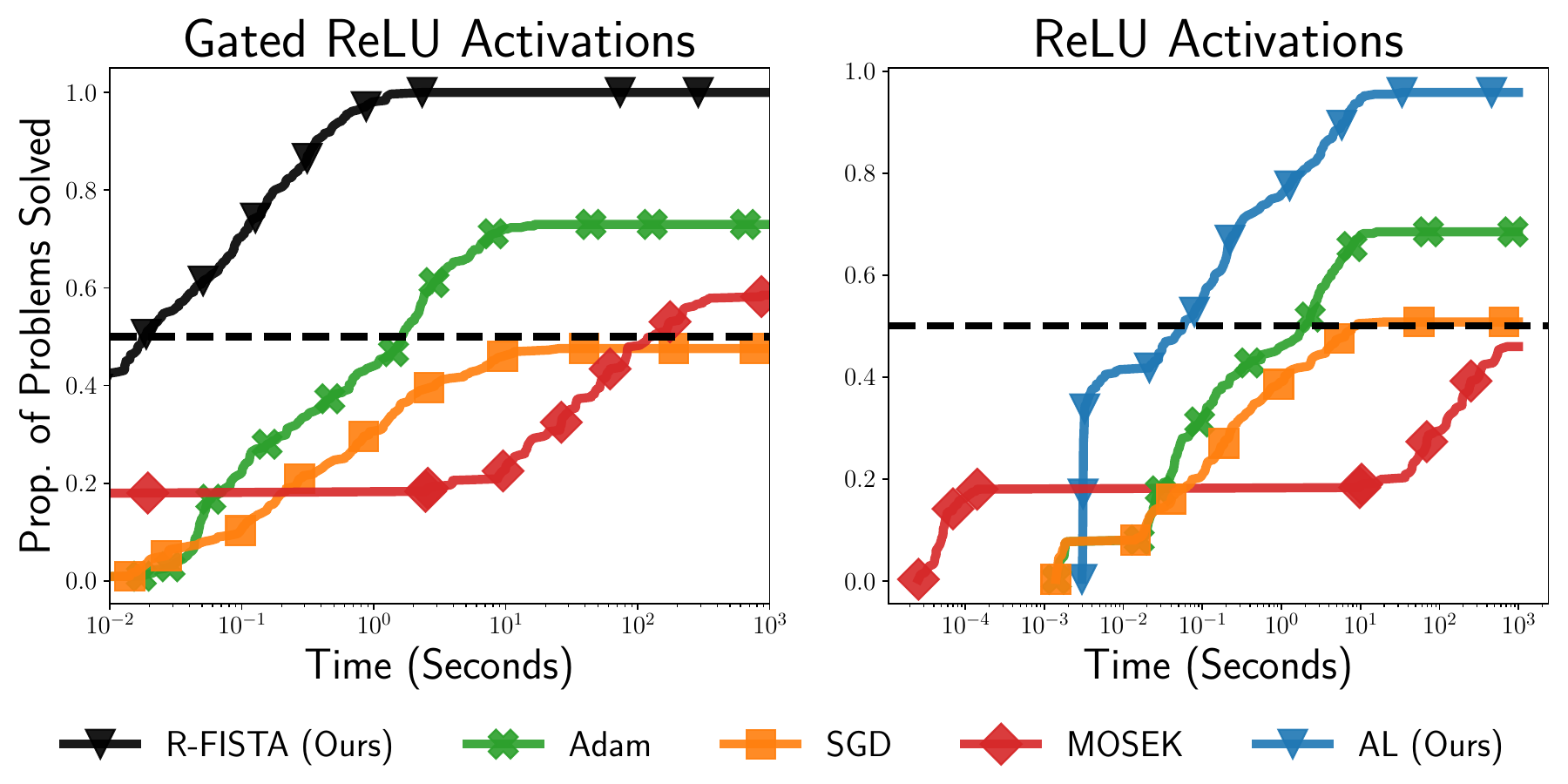}
	\vspace{-0.4cm}
	\caption{Performance profiles comparing (left) R-FISTA and MOSEK for the C-GReLU problem to Adam and SGD for NC-GReLU,
		and (right) the AL method and baselines for C-ReLU/NC-ReLU.
		A problem is solved when \( \rbr{F(x_k) - F(x^*)}/F(x^*) \leq 1 \), where \( F(x^*) \) is the smallest objective value found by any method.
		This rule is method-independent as the convex and non-convex problems share the same optimal objective value.
		See Appendix~\ref{app:performance-profiles} for alternative thresholds.
		Methods are judged by comparing time to a fixed proportion of problems solved (see dashed line at \( 50\% \)).
		R-FISTA and the AL method solve a higher proportion of problems faster than the baselines.
	}%
	\label{fig:performance-profiles}
	\vspace{-0.4cm}
\end{figure*}

Instead, the AL method performs proximal-point iterations on the dual \citep{rockafellar1976monotone, rockafellar1976augmented} via the iterations,
\begin{align}
	\rbr{v_{k+1}, w_{k+1}}
	            & = \argmin_{v, w} \calL_{\delta}(v, w, \gamma_k, \zeta_k), \label{eq:al-subroutine} \\
	\gamma_{k + 1}
	= (\gamma_k & - \delta \tilde X_i v_i)_+, \quad
	\zeta_{k + 1} = (\zeta_k - \delta \tilde X_i w_i)_+. \nonumber
\end{align}
The dual iterates of the AL method converge as \( O(1/\delta \cdot \epsilon) \). See \cref{thm:al-convergence-rate} for a proof going through proximal-point.

\subsubsection{Reliable Constrained Optimization}\label{sec:reliable-co}

AL methods are typically \emph{exterior point} solvers: \( (v_k, w_k) \) will approach the constraint set only as the dual problem is solved.
The complexity of maximizing the dual depends on the penalty strength, with \( \delta \gg 1 \) producing fast convergence.
However, \( \delta \) also affects the Lipschitz smoothness of \( \calL_\delta \) --- large \( \delta \) increases the curvature of the (one-sided) quadratic penalties --- and can make solving~\eqref{eq:al-subroutine} prohibitively expensive for first-order methods.
Thus, we choose \( \delta \) to balance convergence on the primal and dual problems.

\textbf{Choosing the Penalty Strength}:
it is common to set \( \delta \) to aggressively decrease the constraint gap~\citep{conn2013lancelot, murtagh1983minos},
\[ c_{\text{gap}} = \sum_{D_i \in \tilde \calD} \|(\tilde X_i v_i)_-|^2_2 + \|(\tilde X_i w_i)_-\|^2_2. \]
These rules pre-suppose second-order solvers and lead to very poor conditioning of \( \calL_\delta \).
Instead, we propose a simple ``windowing'' heuristic:
when solving Eq.~\ref{eq:al-subroutine} for \( \rbr{\gamma_1, \zeta_1} \), take \( \delta \) to ensure that \( c_{\text{gap}} \in \sbr{r_{\text{l}}, r_{\text{u}}} \).
This condition can be checked and enforced with minimal overhead by using a mild convergence criterion initially and helps avoid extreme behavior.
We found \( [r_{\text{l}}, r_{\text{u}}] = [0.01, 0.1] \) works well.

\textbf{Warm Starts}: The contours of \( \calL_\delta\rbr{\cdot, \cdot, \gamma_{k+1}, \zeta_{k+1}} \) typically change slowly when \( \delta \) is moderate.
In such cases, minimizing the augmented Lagrangian can be greatly sped-up by warm-starting with \( \rbr{v_k, w_k} \).

We obtain an efficient and robust AL method by combining warm-starts, our heuristic for \( \delta \), and the R-FISTA sub-solver.

\section{Experiments}\label{sec:experiments}

We now present experiments validating our optimizers.
We show that training neural networks via convex reformulations is faster and more robust than attempting to solve the non-convex training problem with SGD~\citep{robbins1951sgd} or Adam~\citep{kingma2015adam}.
Moreover, the models learned by convex optimization are consistent and generalize as well as Adam/SGD without their failure modes.

\subsection{Optimization Performance}\label{sec:optimization-performance}

\setlength{\tabcolsep}{3pt} 
\begin{table}[t]
	\centering
	\caption{Approximating the C-ReLU problem with cone decompositions.
		We compare the solution to C-GReLU (FISTA) with cone-decomposition by solving the min-norm program (CD-SOCP),
		the approximate cone decomposition (CD-A), and directly solving C-ReLU using the AL method.
		Exactly solving CD-SOCP is costly compared to direct solutions.
		Although CD-A gives only an approximate decomposition, it yields similar
		test performance to CD-SOCP and is two orders of magnitude faster.
	}%
	\label{table:cone-decomp}
	\vspace{0.1in}
	\begin{small}
		\begin{tabular}{lcccccccc}
			                 & \multicolumn{2}{c}{\textbf{R-FISTA}} & \multicolumn{2}{c}{\textbf{CD-SOCP}} & \multicolumn{2}{c}{\textbf{CD-A}} & \multicolumn{2}{c}{\textbf{AL}}                              \\ \cmidrule(lr){2-3}  \cmidrule(lr){4-5}  \cmidrule(lr){6-7} \cmidrule(lr){8-9}
			\textbf{Dataset} & Acc.                                 & Time                                 & Acc.                              & Time                            & Acc. & Time & Acc. & Time  \\ \midrule
			energy           & 86.3                                 & 0.12                                 & 86.3                              & 134.6                           & 86.3 & 1.56 & 83.7 & 5.05  \\
			ecoli            & 71.6                                 & 0.07                                 & 71.6                              & 149.7                           & 70.1 & 0.29 & 70.1 & 3.38  \\
			glass            & 64.3                                 & 0.13                                 & 64.3                              & 68.76                           & 64.3 & 0.57 & 61.9 & 3.0   \\
			pima             & 73.2                                 & 0.36                                 & 73.2                              & 37.68                           & 73.2 & 4.24 & 75.8 & 4.72  \\
			oocytes          & 78.6                                 & 0.98                                 & 79.1                              & 136.3                           & 78.0 & 4.68 & 74.2 & 81.68 \\ \bottomrule
		\end{tabular}
	\end{small}
\end{table}
\setlength{\tabcolsep}{6pt}

\textbf{Synthetic Classification}:
Convex-reformulations offer a stable approach to model training, especially outside of the over-parameterized setting.
To illustrate this, we create a realizable problem with \( X \sim \calN\rbr{0, \Sigma} \) and \( y = \text{sign}(h_{W_1,w_2}(X)) \), where \( h_{W_1, w_2} \) is a two-layer ReLU network with \( m = 100 \) and random Gaussian weights.
We try to recover this model with ten independent runs of SGD and compare against our AL method on the C-ReLU problem.
For C-ReLU, \( \tilde \calD \) is 100 random arrangements augmented with all activations generated while solving the non-convex problem with SGD.\footnote{This guarantees the non-convex model is in the model space of the convex program.}
\cref{fig:synthetic-classification} shows that SGD converges to sub-optimal stationary points four times, while every run of the convex solver yields a model with perfect training accuracy.
See Appendix~\ref{app:synthetic-classification} for additional results.

\textbf{Large-Scale Comparison}:
\cref{fig:performance-profiles} presents two performance profiles~\citep{dolan2002benchmarking} comparing the optimization performance of R-FISTA and our AL method to Adam, SGD, and the interior-point solver MOSEK~\citep{mosek}.
MOSEK solves the convex reformulations, while Adam and SGD solve the original non-convex problems.
The profiles aggregate performance on \( 438 \) problems generated by considering six regularization parameters for \( 73 \) datasets taken from the UCI repository~\citep{dua2019uci}.
We use the default parameters for MOSEK;
for Adam and SGD, we use a batch-size of \( 10\% \) of the data and take the \emph{best} run per-problem over a grid of seven step-sizes and three different random seeds.
See Appendix~\ref{app:performance-profiles} for details.

We make the following observations: (i) R-FISTA solves \( 50 \% \) of problems two orders of magnitude faster than Adam and SGD; (ii) MOSEK scales poorly and frequently runs out of memory despite being allocated 32GB --- \( 3 \times \) more than the other solvers; (iii) although the ReLU problem is significantly harder, the AL solver converges faster and solves \( 25\% \) more problems than the best baseline.

\begin{figure}[t]
	\centering
	\ifdefined\smallPDF
		\includegraphics[width=0.98\linewidth]{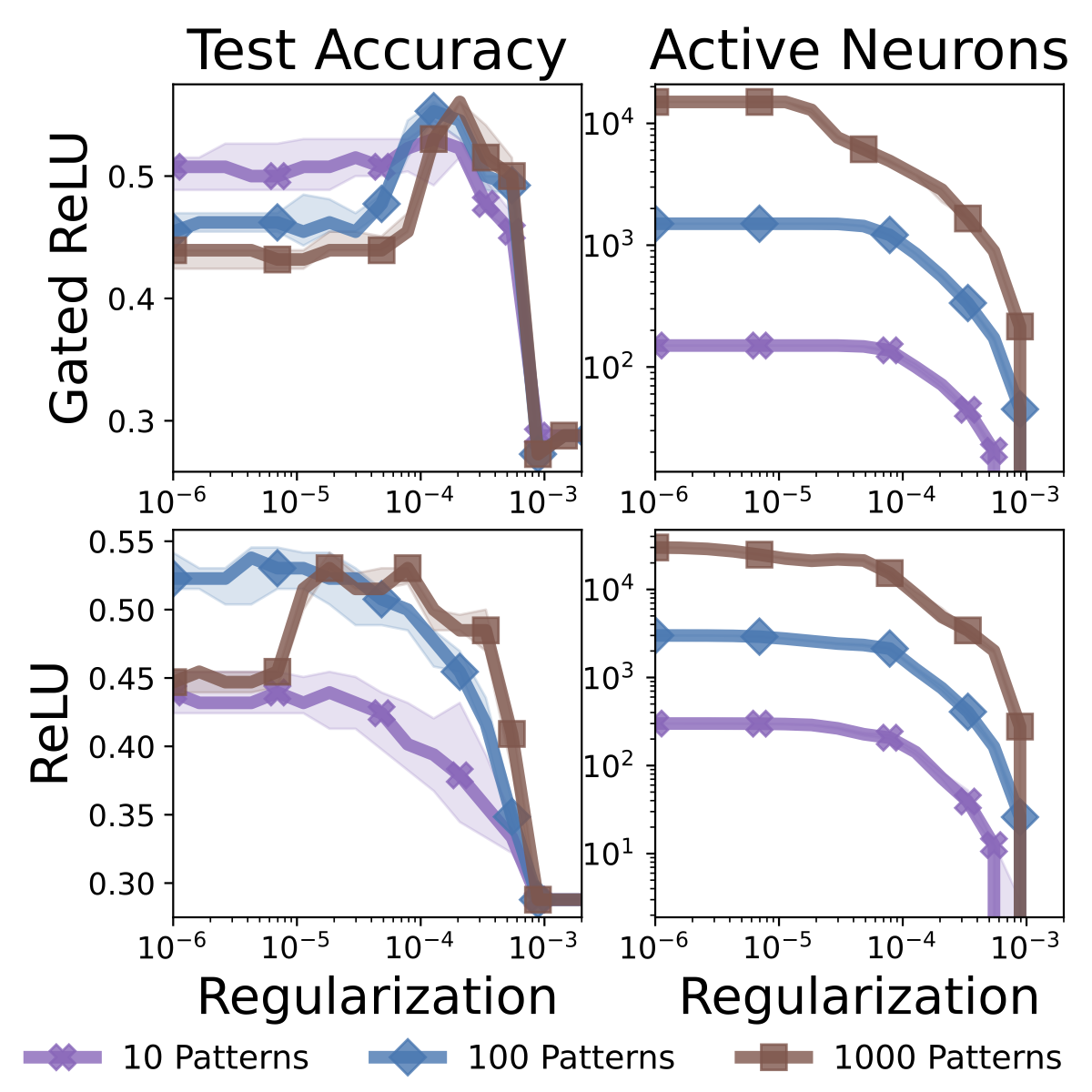}
		\vspace{-0.3cm}
	\else
		\includegraphics[width=0.98\linewidth]{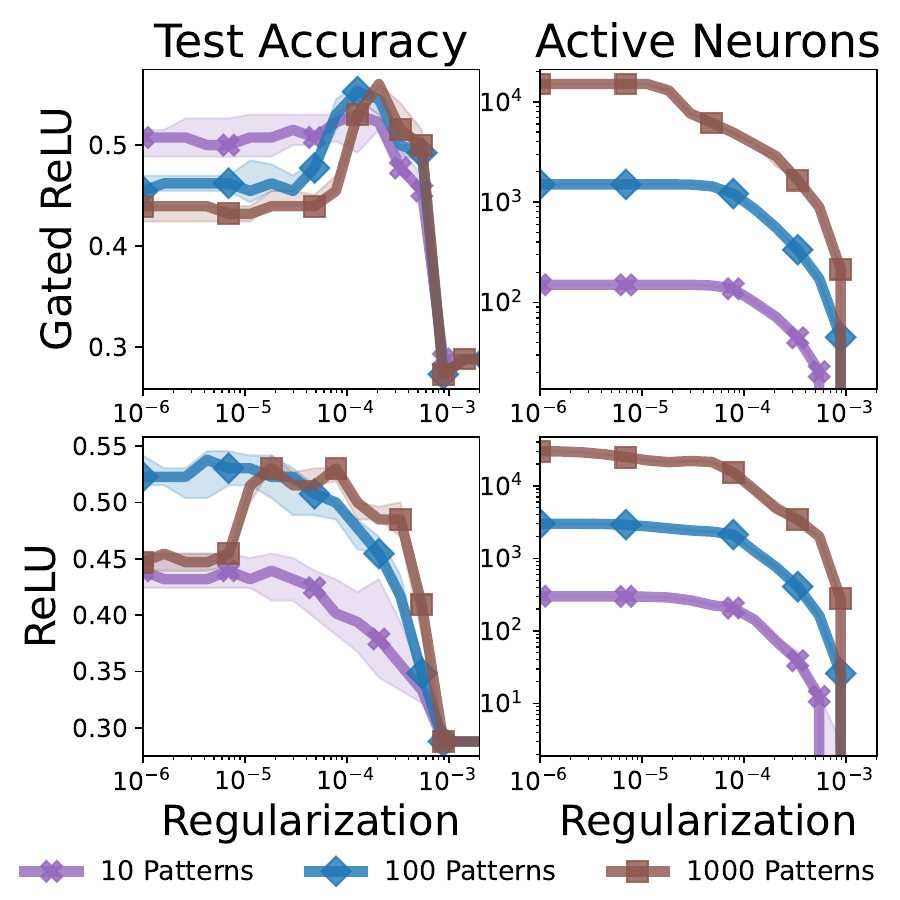}
		\vspace{-0.3cm}
	\fi
	\caption{
		Effect of sampling activation patterns on test accuracy for networks trained using the C-ReLU and C-GReLU problems on the \texttt{primary-tumor} dataset.
		We consider a grid of regularization parameters and plot median (solid line) and first and third quartiles (shaded region) over 10 random samplings of \( \tilde \calD \), where \( |\tilde \calD| \) is limited to 10, 100, or 1000 patterns.
	}%
	\label{fig:reg-plot}%
	\vspace{-0.3cm}
\end{figure}

\textbf{Cone Decompositions}:
We compare optimizing the C-ReLU problem directly using our AL method against \cref{alg:cone-decomp}.
We try two decomposition methods: CD-SOCP, which sets \( R(u,v) = \norm{u}_2 + \norm{v}_2 \) and solves the resulting SOCP, and
CD-A, which approximates the cone decomposition problem by solving Eq.~\eqref{eq:cd-approx}.
We use MOSEK to solve the SOCP.
\cref{table:cone-decomp} gives median test accuracy and time-to-solution for each approach on five UCI datasets.
R-FISTA is an order of magnitude faster than AL and two orders faster CD-SOCP, primarily because SOCPs must be solved on CPU.
CD-A performs comparably to CD-SOCP and is faster than solving the C-ReLU problem
with our AL method.
See \cref{app:cone-decomps} for experimental details and additional results, including model norms.

\subsection{Model Performance}\label{sec:model-performance}

\textbf{Sensitivity and Regularization}:
Figure~\ref{fig:reg-plot} shows the effects of sub-sampling activation patterns on the C-ReLU and C-GReLU problems for the \texttt{primary-tumor} dataset.
Surprisingly, we find that the distribution of test-accuracies is stable across regularization parameters even when the number of patterns is small.
We also observe an inverted-U shaped bias-variance trade-off as the regularization strength is increased, with sparse models showing the best generalization.
This contrasts the double descent phenomena frequently observed with non-convex neural networks~\citep{belkin2019reconciling, loog2020brief, nakkiran2020double}.
See Appendix~\ref{app:activation-pattern-ablations} for results on a further nine UCI datasets.

\setlength{\tabcolsep}{3pt} 
\begin{table}[t]
	\centering
	\caption{Test accuracies for our convex solvers, random forests (RF), SVMs with a linear kernel (Linear)
		and SVMs with an RBF kernel (RBF) for binary classification on 18 UCI datasets.
		C-GReLU and C-ReLU both obtain the best test accuracy on 9 datasets,
		while the most competitive baseline is best on just 4.
	}%
	\label{table:binary-uci-accuracies}
	\vspace{0.1in}
	\begin{small}
		\begin{tabular}{lccccc} \toprule
			\textbf{Dataset} & \textbf{C-GReLU} & \textbf{C-ReLU} & \textbf{RF}   & \textbf{Linear} & \textbf{RBF}  \\ \midrule
			blood            & 79.9             & \textbf{80.5}   & 75.8          & 74.5            & 77.9          \\
			chess-krvkp      & \textbf{99.2}    & 98.6            & 98.9          & 97.2            & 98.4          \\
			conn-bench       & \textbf{90.2}    & 85.4            & 73.2          & 68.3            & 85.4          \\
			cylinder-bands   & 76.5             & \textbf{78.4}   & 77.5          & 71.6            & 71.6          \\
			fertility        & \textbf{80.0}    & \textbf{80.0}   & 75.0          & 75.0            & 75.0          \\
			heart-hung.      & \textbf{86.2}    & \textbf{86.2}   & 84.5          & 84.5            & \textbf{86.2} \\
			hill-valley      & \textbf{76.0}    & 68.6            & 57.9          & 62.0            & 70.2          \\
			ilpd-liver       & 72.4             & \textbf{74.1}   & 66.4          & 71.6            & 71.6          \\
			mammographic     & 77.6             & 78.6            & \textbf{80.7} & \textbf{80.7}   & 80.2          \\
			monks-1          & \textbf{100}     & \textbf{100}    & 95.8          & 79.2            & 83.3          \\
			musk-1           & 94.7             & \textbf{95.8}   & 92.6          & 86.3            & \textbf{95.8} \\
			ozone            & \textbf{97.6}    & \textbf{97.6}   & 97.4          & 97.2            & 97.4          \\
			pima             & 74.5             & 74.5            & \textbf{76.5} & 75.2            & 73.2          \\
			planning         & \textbf{69.4}    & 63.9            & 66.7          & 66.7            & \textbf{69.4} \\
			spambase         & 93.5             & 93.6            & \textbf{94.1} & 92.2            & 93.6          \\
			spectf           & \textbf{87.5}    & 75.0            & 68.8          & 68.8            & 68.8          \\
			statlog-german   & 74.0             & \textbf{77.5}   & 73.5          & 75.0            & 75.5          \\
			tic-tac-toe      & 99.0             & 99.0            & 99.5          & 98.4            & \textbf{100}  \\ \bottomrule
		\end{tabular}
	\end{small}
\end{table}
\setlength{\tabcolsep}{6pt}


\textbf{UCI Classification}:
Table~\ref{table:binary-uci-accuracies} compares the performance of C-ReLU and C-GReLU
with random forests \citep{breiman2001randomforests} and SVMs \citep{boser1992svm} for binary classification on 18 UCI datsasets.
For all methods, we report test accuracy for the best hyperparameters as selected
by cross-validation.
Taken together, C-ReLU and C-GReLU perform best on 14 problems,
showing two-layer neural networks offer an effective, easy-to-train
alternative to common baselines.
Results for additional datasets are given in Appendix~\ref{app:uci-accuracies}

\textbf{Non-Convex Solvers}:
We compare the generalization of C-ReLU and C-GReLU with that of the non-convex problems on 20 UCI datasets.
For each dataset/problem, we select the regularization strength using five-fold cross validation.
For NC-ReLU and NC-GReLU, we use Adam and SGD and tune the step-sizes by cross-validation.
See Appendix~\ref{app:non-convex-solvers} for details.
Table~\ref{table:non-convex-solvers} summarizes the test accuracy results.
We find that our convex programs generalize as well as the non-convex baselines for a fraction of the training time.

\begin{table}[t]
	\centering
	\caption{Median test accuracies from five restarts on a subset of the UCI datasets.
		Results are presented as \texttt{Gated ReLU} / \texttt{ReLU}.
		Overall, we find the convex reformulations have comparable generalization to the non-convex networks.
		Note the catastrophic failure of SGD on \texttt{ecoli}.
		See Appendix~\ref{app:non-convex-solvers} for quartiles.}%
	\label{table:non-convex-solvers}
	\vspace{0.1in}
	\begin{small}
		\begin{tabular}{lccc} \toprule
			\textbf{Dataset} & \textbf{Convex} & \textbf{Adam} & \textbf{SGD} \\ \midrule
			magic            & 86.9  / 85.9    & 82.9 / 86.9   & 82.1 / 86.4  \\
			statlog-heart    & 79.6  / 83.3    & 85.2 / 83.3   & 83.3 / 79.6  \\
			mushroom         & 100 / 100       & 97.6 / 100    & 96.9 / 99.9  \\
			vertebral-col.   & 87.1  / 90.3    & 90.3 / 90.3   & 90.3 / 88.7  \\
			cardiotocogr.    & 90.1  / 89.9    & 85.6 / 36.5   & 85.2 / 88.9  \\
			abalone          & 63.8  / 66.2    & 58.7 / 65.3   & 58.1 / 66.1  \\
			annealing        & 90.6  / 90.6    & 86.2 / 93.7   & 86.2 / 88.7  \\
			car              & 89.9  / 87.8    & 83.8 / 94.8   & 83.2 / 90.1  \\
			bank             & 89.8  / 89.8    & 89.9 / 90.8   & 89.8 / 90.5  \\
			breast-cancer    & 68.4  / 68.4    & 68.4 / 64.9   & 70.2 / 68.4  \\
			page-blocks      & 96.8  / 94.0    & 92.1 / 97.1   & 92.4 / 96.9  \\
			contrac          & 45.9  / 55.1    & 53.1 / 54.4   & 53.4 / 53.7  \\
			congressional    & 63.2  / 63.2    & 64.4 / 62.1   & 66.7 / 67.8  \\
			spambase         & 93.4  / 93.3    & 91.6 / 93.5   & 91.2 / 93.2  \\
			synthetic        & 97.5  / 98.3    & 98.3 / 96.7   & 97.5 / 96.7  \\
			musk-1           & 93.7  / 93.7    & 93.7 / 96.8   & 94.7 / 95.8  \\
			ringnorm         & 69.8  / 77.0    & 77.0 / 77.3   & 77.2 / 77.4  \\
			ecoli            & 82.1  / 80.6    & 79.1 / 82.1   & 4.5  / 80.6  \\
			monks-2          & 69.7  / 69.7    & 66.7 / 69.7   & 60.6 / 72.7  \\
			hill-valley      & 62.0  / 65.3    & 57.0 / 62.8   & 58.7 / 55.4  \\ \bottomrule
		\end{tabular}
		\vspace{-0.7cm}
	\end{small}
\end{table}

\textbf{Image Classification:}
We study the generalization performance of the Gated ReLU model for image classification on the MNIST and CIFAR-10 datasets \citep{lecun1998gradient, krizhevsky2009learning}.
We compare R-FISTA for C-GReLU to solving the NC-GReLU problem with SGD, Adam, and Adagrad~\citep{duchi2011adagrad}.
We choose the regularization strength and step sizes for each dataset-method pair using a train/validation split (see Appendix~\ref{app:image-datasets}).
Table~\ref{table:image-datasets} shows that R-FISTA scales well to these large-scale experiments, with generalization comparable with the non-convex solvers.
This reflects our theory, which shows that these methods are fundamentally solving the same problem.


\begin{table}
	\centering
	\caption{Test accuracy of R-FISTA for the C-GReLU problem compared to SGMs for NC-GReLU on two image classification tasks.
		Models trained using the convex program have comparable test accuracy to the non-convex formulation on MNIST and are slightly better on CIFAR-10.
	}%
	\label{table:image-datasets}%
	\vspace{0.1in}
	\begin{small}
		\begin{tabular}{lcccc} \toprule
			\textbf{Dataset} & \textbf{Convex} & \textbf{Adam} & \textbf{SGD} & \textbf{Adagrad} \\ \midrule
			MNIST            & 97.6            & 98.0          & 97.2         & 97.5             \\
			CIFAR-10         & 56.4            & 50.1          & 54.3         & 54.2             \\\bottomrule
		\end{tabular}
	\end{small}
\end{table}


\subsection{Additional Experiments}\label{sec:additional-experiments}

We defer additional experiments to the supplementary material due to space constraints.
In Appendix~\ref{app:acceleration-exps}, we study the effects of acceleration, restarts, and line-search on the performance of the R-FISTA method and conclude that all three components are key to the efficiency of the optimization procedure.
Appendix~\ref{app:step-size-initialization} presents an ablation study for the step-size initialization procedure in R-FISTA, which is shown to be robust to the choice of \( c \).
Similarly, Appendix~\ref{app:delta-heuristic-exps} examines the windowing heuristic for the penalty strength in our AL method and shows the strategy is comparable to the best fixed \( \delta \) found by grid search.



\section{Conclusion}\label{sec:conclusion}
We propose optimization algorithms for convex reformulations of two-layer neural networks with ReLU activations.
By studying the problem constraints, we split the space of ReLU activations into singular patterns, which may be safely ignored, and non-singular patterns.
As a result, we show that ReLU networks can be trained by decomposing the solution to an unconstrained Gated ReLU training problem onto a difference of polyhedral cones.
Experimentally, we test our algorithms on more than \( 70 \) different datasets, demonstrating that convex optimization is faster and more reliable than popular training methods like Adam and SGD.

Many directions are left to future work.
Efficiently solving the cone decomposition problem is key to improving on our augmented Lagrangian method, but existing conic solvers rely on CPU computation.
We believe developing methods which can natively leverage GPU acceleration is necessary.
Finally, we hope to extend convex optimization to deeper networks by layer-wise training, which has been shown to perform well on ImageNet~\cite{belilovsky2019greedy}.

\ifdefined\isaccepted
	\section*{Acknowledgements}

	This work was partially supported by an Army Research Office Early Career Award, and the National Science Foundation under grants ECCS-2037304, DMS-2134248.
	Aaron Mishkin was supported by the NSF Graduate Research Fellowship Program, Grant No. DGE-1656518 and the NSERC PGS D program, Grant No. PGSD3-547242-2020.
	Arda Sahiner was supported by the National Institutes of Health, Grant No. R01EB009690, U01EB029427.
	Computational resources were provided by the Stanford Research Computing Center.
	We would like to thank Frederik Kunstner and Tolga Ergen for many insightful discussions.
\fi

\newpage

\bibliography{
	bib/monographs.bib,
	bib/gradient_methods.bib,
	bib/lower_bounds.bib,
	bib/assumptions.bib,
	bib/interpolation.bib,
	bib/step_sizes.bib,
	bib/datasets.bib,
	bib/general_refs.bib}

\bibliographystyle{icml2022}

\clearpage
\onecolumn
\appendix


\section{Convex Reformulations:\ Proofs}\label{app:convex-forms}

\begin{restatable}{lemma}{miFormulation}\label{lemma:mi-reformulation}
	The non-convex problem NC-ReLU (Problem \ref{convex-forms:eq:non-convex-relu-mlp}) is equivalent to the mixed-integer program,
	\begin{equation}\label{convex-forms:eq:mi-reformulation}
		\begin{aligned}
			\min_{W_1, w_2} \; & \calL\Big(\sum_{i = 1}^m \rbr{X W_{1i}}_+ w_{2i}, y\Big) + \lambda \sum_{i = 1}^m \norm{W_{1i}}_2 \\
			                   & \text{\emph{s.t.}} \; w_2 \in \cbr{-1, 1}^m.
		\end{aligned}
	\end{equation}
\end{restatable}

\begin{proof}
	The proof proceeds in two steps: first we transform the objective into an equivalent problem which is invariant to certain scale re-parameterizations of the network parameters.
	Then, we use these scale re-parameterizations reduce the two optimization problems to each-other.

	Let \( p^* \) be the optimal value of the non-convex optimization problem and \( d^* \) the optimal value of the mixed-integer program.
	The ReLU activation function,
	\begin{align*}
		\rbr{a}_+ = \max\cbr{a, 0},
	\end{align*}
	is positively homogeneous, meaning \( \rbr{a * \beta}_+ = \beta \rbr{a}_+ \) for any scalar \( \beta \geq 0 \).
	Defining \( W_{1i}' = \beta_i W_{1i}, \, w_{2i}' = w_{2i} / \beta_i \), \( i \in [m] \), we have
	\begin{align*}
		h_{W_1, w_2}(X)
		 & = \sum_{i=1}^m \rbr{X W_{1i} }_+ w_{2i}
		= \sum_{i=1}^m \rbr{X W_{1i} \frac{\beta_i}{\beta_i}}_+ w_{2i} \\
		 & = \sum_{i=1}^m \rbr{X W_{1i}'}_+ w_{2i}'
		= h(W'_{1}, w_2'),
	\end{align*}
	implying that the loss \( \calL(\sum_{i=1}^m \rbr{X W_{1i} }_+ w_{2i}, y) \) is invariant to ``scale-shift'' re-parameterizations of this form.
	To extend the invariance to the full objective function, recall Young's inequality,
	\[
		2 \abr{a, b} \leq a^2 + b^2,
	\]
	which yields,
	\begin{align*}
		\calL\rbr{\sum_{i = 1}^m \rbr{X W_{1i} }_+ w_{2i}, y } + \frac{\lambda}{2} \sum_{i = 1}^m \norm{W_{1i}}_2 + \abs{w_{2i}}^2
		 & \geq \calL\rbr{ \sum_{i = 1}^m \rbr{X W_{1i} }_+ w_{2i}, y } + \lambda \sum_{i = 1}^m \norm{W_{1i}}\abs{w_{2i}}.
	\end{align*}
	For any choice of parameters \( W_{1i}, w_{2i} \), equality in this expression is achieved with the rescaling
	\[
		W_{1i}' = W_{1i} * \beta_i, \quad \quad w_{2i}' = w_{2i} / \beta_i,
	\]
	where \( \beta_i = \sqrt{\frac{w_{2i}}{\norm{W_{1i}}_2}} \).
	As this rescaling does not affect \( h_{W_1, w_2} \), it must be that any global minimizer \( \theta^* = \rbr{W_{1}^*, w_{2}^*} \) of Problem~\ref{convex-forms:eq:non-convex-relu-mlp} achieves the lower-bound in Young's inequality and,
	\begin{align}\label{convex-forms:eq:inner-product-reformulation}
		\calL\rbr{\sum_{i = 1}^m \rbr{X W_{1i}^*}_+ w_{2i}^*, y} + \frac{\lambda}{2} \sum_{i = 1}^m \norm{W_{1i}^*}_2^2 + \abs{w_{2i}^*}^2
		=
		\calL\rbr{\sum_{i = 1}^m \rbr{X W_{1i}^*}_+ w_{2i}^*, y} + \lambda \sum_{i = 1}^m \norm{W_{1i}^*}_2\abs{w_{2i}^*}.
	\end{align}
	The right-hand side of this equation is invariant to scale re-parameterizations of the form \( W_{1i}' = \beta W_{1i}^*, w_{2i}' = w_{2i}^* / \beta \) for \( \beta > 0 \).
	Taking \( \beta = \abs{w_{2i}} \), we deduce
	\begin{align*}
		p^* = \calL\rbr{\sum_{i = 1}^m \rbr{X W_{1i}'}_+ w_{2i}', y} + \lambda \sum_{i = 1}^m \norm{W'_{1i}}_2 \geq d^*.
	\end{align*}

	To show the reverse inequality, observe that every global minimum \( \rbr{W^*_1, w^*_2} \) of Problem~\ref{convex-forms:eq:mi-reformulation} is trivially in the domain of the non-convex ReLU training problem.
	Using the mapping
	\[ \rbr{W'_{1i}, w_{2i}'} = \rbr{\frac{W_{1i}^*}{\sqrt{\norm{W_{1i}^*}_2}}, w^*_{2i} \sqrt{\norm{W_{1i}^*}_2}}, \]
	and plugging \( \rbr{W'_1, w'_{2}} \) into Problem~\ref{convex-forms:eq:non-convex-relu-mlp} shows \( d^* \geq p^* \).
	We have shown \( p^* = d^* \) and so the problems are formally equivalent with mappings between the solutions as given above.

\end{proof}

\convexDualityFree*
\begin{proof}
	The proof proceeds by showing the equivalence of C-ReLU and the mixed integer problem given in \cref{convex-forms:eq:mi-reformulation} and the invoking \cref{lemma:mi-reformulation}.
	Let \( p^* \) be the optimal value of the mixed-integer problem in~\eqref{convex-forms:eq:mi-reformulation} and \( d^* \) the optimal value of the convex program in~\eqref{convex-forms:eq:convex-relu-mlp}.
	We first show that \( d^* \geq p^* \).

	Suppose \( \rbr{v^*, w^*} \) is a global minimizer of Problem~\ref{convex-forms:eq:convex-relu-mlp} and let
	\begin{align*}
		\cbr{\rbr{W_{1k}^*, w_{2k}^*}} = \bigcup_{D_i \in \tilde \calD_i} \cbr{\rbr{v^*_i, 1} : v^*_i \neq 0} \cup \cbr{\rbr{w^*_i, -1} : w^*_i \neq 0},
	\end{align*}
	where we set \( W_{1k}^* = 0, \) and \( w_{2k}^* = 0 \) for all \( k \in [m], k > b \).
	It holds by assumption that \( b \leq m \) and thus \( \rbr{W_1^*, w_{2}^*} \) is a valid input for the mixed-integer problem.

	Recalling the constraints \( (2 D_i - I) X v^*_i \geq 0 \), and \( (2 D_i - I) X w^*_i \geq 0 \),
	we see that \( D_i X v_i^* = \rbr{X v_i^*}_+ \), \( D_i X w_i^* = \rbr{X w_i^*}_+ \), and thus
	\begin{align*}
		\rbr{X W_{1k}^*} w_{2j}
		 & = \begin{cases}
			     D_i X v_i^*   & \mbox{if \( W_{1k}^* = v_i^* \) for some \( i \in [b] \)} \\
			     - D_i X w_i^* & \mbox{if \( W_{1k}^* = w_i^* \) for some \( i \in [b] \)} \\
			     0             & \mbox{otherwise}.
		     \end{cases}
	\end{align*}
	Using this fact in the optimization objective for the convex program, we find
	\begin{align*}
		d^*
		 & = \calL\rbr{\sum_{D_i \in \tilde \calD} D_i X (v_i - w_i), y} + \lambda \sum_{D_i \in \calD} \norm{v_i}_2 + \norm{w_i}_2 \\
		 & = \calL\rbr{\sum_{k = 1}^m \rbr{X W_{1k}^*}_+ w_{2k}^* - y} + \lambda \sum_{k = 1}^m \norm{W_{1k}^*}_2                   \\
		 & \geq p^*,
	\end{align*}
	as required.

	To show the reverse inequality, let \( \rbr{W_1^*, w_{2}^*} \) be a solution to~\eqref{convex-forms:eq:mi-reformulation} and consider the set-function
	\[
		T(j) = \cbr{ i \in [m] : \text{diag}\rbr{X W_{1i}^* > 0} = D_j }.
	\]
	Recalling
	\( \cbr{\text{diag}\rbr{X W_{1i}^* > 0 : i \in [m]} } \subseteq \tilde \calD, \)
	by assumption, we define a valid candidate solution as
	\begin{equation*}
		\cbr{\rbr{v_j^*, w_j^*}}_{D_j \in \tilde \calD} = \cbr{ \sum_{i \in T(j)} W_{1i}^* \mathbbm{1}(w_{2i}^* = 1), \sum_{i \in T(j)} W_{1i}^* \mathbbm{1}(w_{2i}^* = -1) }_{D_j \in \tilde \calD}
	\end{equation*}

	We start by showing that the neurons indexed by \( T(j) \) can be merged without changing the objective of the mixed-integer problem.
	In particular, let \( j \in [m] \) be arbitrary and suppose that there exists \( l,k \in T(j) \) such that \( w_{2l}^* = w_{2k}^* = 1 \).
	By definition of \( T \), it holds that \( \rbr{X W_{1l}^*}_+ \) and \( \rbr{X W_{1k}^*}_+ \) have the same activation pattern.
	Accordingly, we have \(  \rbr{X W_{1l}^*}_+ + \rbr{X W_{1k}^*}_+ = \rbr{X (W_{1l}^* + W_{1k}^*)}_+ \) by definition of the ReLU activation.
	Thus, merging these two parameter vectors as \( Z_{lk}^* = W_{1l}^* + W_{1k}^*\) does not change the prediction of the model in the mixed-integer program.

	Now we consider the group \( \ell_1 \) penalty term.
	Triangle inequality implies
	\begin{align*}
		\norm{\tilde Z_{lk}^*}_2 \leq \norm{W_{1l}^*}_2 + \norm{W_{1k}^*}_2,
	\end{align*}
	with equality if and only if \( W_{1l}^* = 0 \), \( W_{1k}^* = 0 \), or the vectors are collinear.
	Suppose that equality does not hold.
	Then the penalty term could be reduced setting \( W_{1l}^* = Z_{lk} \) and \( W_{1k}^* = 0\) while leaving the squared-loss term unchanged.
	But, this contradicts global optimality of \(W_{1}^*, w_{2}^* \).
	Thus, it must be that \( W_{1l}^* = 0 \), \( W_{1k}^* = 0 \), or the vectors are collinear.
	In each case, we have that the merged vector \( \tilde Z_{lk}^* \) also attains the optimal value \( p^* \).
	Clearly a symmetric argument holds in the case \( w_{2k} = w_{2l} = -1 \).

	Arguing by induction if necessary, we deduce that the vectors \( v^*, w^* \) given by the solution mapping also attain \( p^* \).
	Recalling that \( D_j X v_j^* = \rbr{X v_j}_+ \) and \( D_j X w_j^* = \rbr{X w_j}_+ \) by choice of \( T(j) \) and definition of \( D_j \) gives
	\begin{align*}
		p^*
		 & = \calL\rbr{\sum_{i=j}^P D_j X (v^*_j - w^*_j) - y} + \lambda \sum_{j=1}^P \norm{v^*_j}_2 + \norm{w^*_j}_2 \\
		 & \geq d^*,
	\end{align*}
	where \( v_j^*, w_j^* \) are feasible.
	This completes the proof.
\end{proof}

\unconstrainedEquivalence*
\begin{proof}
	The proof proceeds similarly to the proof \cref{lemma:mi-reformulation}.

	Let \( p^* \) be the optimal value of the Problem~\ref{convex-forms:eq:gated-relu-mlp} and let \( d^* \) be the optimal value of the C-GReLU problem~\eqref{convex-forms:eq:unconstrained-relu-mlp}.
	For each \( z_i \in \tilde \calZ \) and any \( v \in \R^d \), we have the following equality by construction:
	\[ \phi_{z_i}(X, v) = \rbr{X z_i > 0} \circ X v = D_i X v. \]
	Accordingly, the non-convex optimization problem~\eqref{convex-forms:eq:gated-relu-mlp} can be written as
	\begin{equation}
		\min_{W_1,w_2} \mathcal{L}\Big(\sum_{D_i \in \tilde \calD } D_i X W_{1i} w_{2i}, y\Big) + \frac{\lambda}{2} \sum_{z_i \in \tilde \calZ} \norm{W_{1i}}_2^2 + w_{2i}^2,
	\end{equation}
	which makes the connection to C-GReLU clear.
	Applying Young's inequality gives,
	\begin{align*}
		\mathcal{L}\Big(\sum_{D_i \in \tilde \calD } D_i X W_{1i} w_{2i}, y\Big) + \frac{\lambda}{2} \sum_{z_i \in \tilde \calZ} \norm{W_{1i}}_2^2 + w_{2i}^2
		 & \geq \mathcal{L}\Big(\sum_{D_i \in \tilde \calD } D_i X W_{1i} w_{2i}, y\Big) + \lambda \sum_{z_i \in \tilde \calZ} \norm{W_{1i}}_2 \abs{w_{2i}}
	\end{align*}
	For any choice of parameters \( W_{1i}, w_{2i} \), equality in this expression is achieved with the rescaling
	\[
		W_{1i}' = W_{1i} * \beta_i, \quad \quad w_{2i}' = w_{2i} / \beta_i,
	\]
	where \( \beta_i = \sqrt{\frac{w_{2i}}{\norm{W_{1i}}_2}} \).
	As this rescaling does not affect \( D_i X W_{1i} \, w_{2i} \) for each \( i \), it must be that any global minimizer \( \theta^* = \cbr{W_{1i}^*, w_{2i}^*} \) of Problem~\ref{convex-forms:eq:gated-relu-mlp} achieves the lower-bound in Young's inequality.
	Defining \( v_i' = W_{1i} * w_{2i} \), we have shown
	\begin{align*}
		p^*
		 & = \mathcal{L}\Big(\sum_{D_i \in \tilde \calD } D_i X W_{1i} w_{2i}, y\Big) + \lambda \sum_{z_i \in \tilde \calZ} \norm{W_{1i}}_2 \abs{w_{2i}} \\
		 & = \mathcal{L}\Big(\sum_{D_i \in \tilde \calD } D_i X v_i', y \big) + \lambda \sum_{z_i \in \tilde \calZ} \norm{v_{i}'}_2
		\geq d^*,
	\end{align*}
	where we have used absolute homogeneity of the norm.

	To obtain the reverse inequality, let \( v^* \) be a global minimizer of C-GReLU and define \( W_{1i}' = \frac{v_i^*}{\sqrt{\norm{v_i^*}_2}} \), \( w_{2i}' = \sqrt{\norm{v_i^*}_2} \) for all \( i \) to obtain

	\begin{align*}
		d^*
		 & = \mathcal{L}\Big(\sum_{D_i \in \tilde \calD} D_i X v_i, y\Big)+ \lambda \sum_{D_i \in \tilde \calD} \norm{v_i}_2                                                 \\
		 & = \mathcal{L}\Big(\sum_{D_i \in \tilde \calD } D_i X W_{1i}' w_{2i}', y\Big) + \frac{\lambda}{2} \sum_{z_i \in \tilde \calZ} \norm{W_{1i}'}^2_2 + \abs{w_{2i}'}^2
		\geq p^*,
	\end{align*}
	which completes the proof.

\end{proof}


\begin{proposition}\label{prop:unconstrained-relu-mlp}
	Problem~\ref{convex-forms:eq:unconstrained-relu-mlp} is equivalent to following unconstrained relaxation of the ReLU training problem's convex reformulation (Problem~\ref{convex-forms:eq:convex-relu-mlp}):
	\begin{equation} \label{convex-forms:eq:unreduced-unconstrained-relu-mlp}
		\begin{aligned}
			\min_{v, w} \;
			 & \half \calL\rbr{\sum_{D_i \in \tilde \calD} D_i X (v_i - w_i), y} + \lambda \sum_{D_i \in \tilde \calD} \norm{v_i}_2 + \norm{w_i}_2 \\
			 & \text{s.t.} \;  \rbr{2 D_i - I_n} X v_i \succeq 0, \rbr{2 D_i - I_n} X w_i \succeq 0,
		\end{aligned}
	\end{equation}
\end{proposition}
\begin{proof}
	Suppose that \( (v^*, w^*) \) is an optimal solution \eqref{convex-forms:eq:unreduced-unconstrained-relu-mlp}.
	Defining \( r^* = v^* - w^* \), it holds by the triangle inequality that
	\begin{equation*}
		\norm{r_i^*}_2 \leq \norm{v_i^*}_2 + \norm{w_i^*}_2,
	\end{equation*}
	for each \( i \in [P] \).
	Since replacing \( v_i^* - w_i^* \) with \( r_i^* \) does not change the model prediction
	\[ \hat y = \sum_{D_i \in \tilde \calD} D_i X(v_i - w_i), \]
	we have shown that \( (r^*, 0) \) defines an equivalent model with the same or smaller objective value.
	Accordingly, any solution to~\eqref{convex-forms:eq:unreduced-unconstrained-relu-mlp} must have
	\( v_i^* = 0 \) or \( w_i^* = 0 \).
	Equivalence of the two problems follows immediately.
\end{proof}

\subsection{Extension to Multi-class Classification}\label{app:multi-class-extension}

Now we extend our sub-sampled C-ReLU and C-GReLU formulations to vector-valued problems, such as occur in multi-class classification.
Our starting place is the following vector-output variant of the NC-ReLU problem
\begin{equation}\label{convex-forms:eq:vector-non-convex-relu-mlp}
	\min_{W_1, W_2} \mathcal{L}\Big(\sum_{i=1}^m (XW_{1i})_+ W_{2i}^\top, Y\Big) + \frac{\lambda}{2}\sum_{i=1}^m \norm{W_{1i}}_2^2 + \norm{W_{2i}}_1^2,
\end{equation}
where now labels \(Y \in \R^{n \times C}\).
We note that the main difference between this formulation and \cref{convex-forms:eq:non-convex-relu-mlp} is that each row of \(X\) now maps to a vector rather than a single scalar, and the use of \(\ell_1\)-squared regularization on the second-layer weights. We now present a similar result to \cref{lemma:mi-reformulation} for this particular problem.
\begin{lemma}
	The non-convex problem \eqref{convex-forms:eq:vector-non-convex-relu-mlp} is equivalent to the following program,
	\begin{equation}\label{convex-forms:eq:vector-mi-reformulation}
		\begin{aligned}
			\min_{\{W_{1}^k, w_{2}^k\}_{k=1}^C}  \; & \mathcal{L}\Big(\sum_{i=1}^m (XW_{1i})_+ W_{2i}^\top, Y\Big) + \lambda \sum_{i = 1}^{m}  \norm{W_{1i}}_2 \\
			                                        & \text{\emph{s.t.}} \;\|W_{2i}\|_1 = 1~\forall i \in [m]
		\end{aligned}
	\end{equation}
\end{lemma}
\begin{proof}
	Follow from the proof of \cref{lemma:mi-reformulation} (Appendix \ref{app:convex-forms}), i.e. apply Young's inequality to achieve
	\begin{equation*}
		p^* = \min_{W_1, W_2} \mathcal{L}\Big(\sum_{i=1}^m (XW_{1i})_+ W_{2i}^\top, Y\Big) + \frac{\lambda}{2}\sum_{i=1}^m \norm{W_{1i}}_2^2 + \norm{W_{2i}}_1^2 = \min_{W_1, W_2} \mathcal{L}\Big(\sum_{i=1}^m (XW_{1i})_+ W_{2i}^\top, Y\Big) + \lambda\sum_{i=1}^m \norm{W_{1i}}_2\norm{W_{2i}}_1
	\end{equation*}
	Then, this is clearly equivalent to
	\begin{equation*}
		\begin{aligned}
			\min_{W_{1}, W_2}  \; & \mathcal{L}\Big(\sum_{i=1}^m (XW_{1i})_+ W_{2i}^\top, Y\Big) + \lambda \sum_{i = 1}^{m}  \norm{W_{1i}}_2 \\
			                      & \text{\emph{s.t.}} \;\|W_{2i}\|_1 = 1~\forall i \in [m]
		\end{aligned}
	\end{equation*}
\end{proof}
We can form the one-vs-all convex reformulation as follows:
\begin{equation}\label{convex-forms:eq:vector-convex-relu-mlp}
	\begin{aligned}
		\!\!\! \min_{\{v^k, w^k\}_{k=1}^C} \, & \mathcal{L}\Big(\sum_{k=1}^C \sum_{D_i \in \tilde \calD}\!\! D_i X (v_{i}^k \!-\! w_{i}^k)e_k^\top, Y\Big)\!+\!\lambda\!\sum_{k=1}^C\sum_{D_i \in \tilde \calD}\!\norm{v_{i}^k}_2\!+\!\norm{w_{i}^k}_2 \\
		                                      & \text{s.t.} \, \rbr{2 D_i - I_n} X v_{i}^k \succeq 0, \, \rbr{2 D_i - I_n} X w_{i}^k \succeq 0,
	\end{aligned}
\end{equation}
where \(e_k\) is the \(k\)th standard basis vector.

Then, we have the following analog of \cref{thm:duality-free} for the vector-output case:
\begin{restatable}{theorem}{convexVectorDualityFree}\label{thm:vector-duality-free}
	Suppose \(\rbr{W_1^* W_2^*}\) and \( \{\rbr{v^{k*}, w^{k*}}\}_{k=1}^C \) are global minimizers of Problems~\ref{convex-forms:eq:vector-mi-reformulation} and Problem~\ref{convex-forms:eq:vector-convex-relu-mlp}, respectively.
	If the number of hidden units satisfies
	\[ m^* \geq b := \sum_{k=1}^C \sum_{D_i \in \tilde \calD} \abs{\cbr{v^{k*}_i : v^{k*}_i \neq 0} \cup \cbr{w^{k*}_i : w^{k*}_i \neq 0} }, \]
	and the optimal activations are in the convex model,
	\[ \cbr{\text{\emph{diag}}\rbr{X W_{1i}^{*} > 0 : i \in [m]} } \subseteq \tilde \calD, \]
	then the two problems have same the optimal value.
\end{restatable}
\begin{proof}
	We follow from the proof of \cref{thm:duality-free}.   Let \(p^*\) be the optimal value of \eqref{convex-forms:eq:vector-mi-reformulation} and \(d^*\) be the optimal value of \eqref{convex-forms:eq:vector-convex-relu-mlp}.

	First, suppose \( \{\rbr{v^{k*}, w^{k*}}\}_{k=1}^C\) is a global minimizer of Problem \ref{convex-forms:eq:vector-convex-relu-mlp}. Then, let
	\[ \rbr{W_{1(i, k)}^*, W_{2(i, k)}^*} = \bigcup_{D_i \in \tilde \calD} \bigcup_{k=1}^C \cbr{(v_i^{k*}, e_k): v_i^{k*} \neq 0} \cup \cbr{(w_i^{k*}, -e_k): w_i^{k*} \neq 0)}  \]
	where we set \(W_{1(i, k)}^* = 0\) and \(W_{2(i, k)}^*\) = 0 for non-assigned neurons. It holds by assumption that \(b \leq m\) and thus this is a valid input for \eqref{convex-forms:eq:vector-mi-reformulation}.
	Further, we have, due to the constraints,
	\begin{equation*}
		(XW_{1(i, k)}^*)_+W_{2(i, k)}^{*\top } = \begin{cases}
			D_iXv_{i}^{k*}e_k^\top & \text{if } W_{1(i, k)}^* = v_i^{k*} \\
			D_iXw_{i}^{k*}e_k^\top & \text{if } W_{1(i, k)}^* = w_i^{k*} \\
			0                      & o.w.
		\end{cases}
	\end{equation*}
	Inserting to the objective for the convex program,
	\begin{align}
		d^* & =\mathcal{L}\Big(\sum_{k=1}^C \sum_{D_i \in \tilde \calD}\!\! D_i X (v_{i}^{k*} \!-\! w_{i}^{k*})e_k^\top, Y\Big)\!+\!\lambda\!\sum_{k=1}^C\sum_{D_i \in \tilde \calD}\!\norm{v_{i}^{k*}}_2\!+\!\norm{w_{i}^{k*}}_2 \\
		    & = \mathcal{L}\Big(\sum_{(i, k)}\!\!(XW_{1(i, k)}^*)_+W_{2(i, k)}^{*\top }, Y\Big)\!+\!\lambda\!\sum_{(i,k) }\!\norm{W_{1(i, k)}^*}_2                                                                                \\  &\geq p^*
	\end{align}
	Now, we seek to find the other direction, i.e. show \(p^* \geq d^*\) and show a mapping. Let \(\rbr{W_{1i}^*, W_{2i}^*}\) be a solution to \eqref{convex-forms:eq:vector-mi-reformulation}. Defining, as in \cref{thm:duality-free},
	\[
		T(j) = \cbr{ i \in [m] : \text{diag}\rbr{X W_{1i}^* > 0} = D_j }.
	\]
	Recalling
	\( \cbr{\text{diag}\rbr{X W_{1i}^* > 0 : i \in [m]} } \subseteq \tilde \calD, \)
	by assumption, we define a valid candidate solution as
	\begin{equation*}
		\cbr{\cbr{\rbr{v_j^{k*}, w_j^{k*}}}_{k \in [C]}}_{D_j \in \tilde \calD} = \cbr{ \cbr{\sum_{i \in T(j)} W_{1i}^*w_{2i}^{k*} \mathbbm{1}(W_{2i}^{k*} \geq 0)}_{k \in [C]}, \cbr{-\sum_{i \in T(j)} W_{1i}^*w_{2i}^{k*} \mathbbm{1}(W_{2i}^{k*} < 0)}_{k \in [C]} }_{D_j \in \tilde \calD}
	\end{equation*}
	Then, by the same co-linearity arguments as in \cref{thm:duality-free}, we have
	\begin{align*}
		p^* & =  \mathcal{L}\Big(\sum_{i=1}^m (XW_{1i}^*)_+ W_{2i}^{*\top}, Y\Big) + \lambda \sum_{i = 1}^{m}  \norm{W_{1i}^*}_2                                                                                   \\
		    & = \mathcal{L}\Big(\sum_{i=1}^m \sum_{k=1}^C (XW_{1i}^*)_+ W_{2i}^{k*}e_k^\top, Y\Big) + \lambda \sum_{i = 1}^{m}  \norm{W_{1i}^*}_2 \sum_{k=1}^C |W_{2i}^{k*}|                                       \\
		    & = \mathcal{L}\Big(\sum_{i=1}^m \sum_{k=1}^C (XW_{1i}^*)_+ W_{2i}^{k*}e_k^\top, Y\Big) + \lambda \sum_{i = 1}^{m} \sum_{k=1}^C \norm{W_{1i}^*}_2  |W_{2i}^{k*}|                                       \\
		    & = \mathcal{L}\Big(\sum_{D_j \in \tilde \calD} \sum_{k=1}^C D_jX(v_{j}^{k*}- w_{j}^{k*})e_k^\top, Y\Big) + \lambda \sum_{D_j \in \tilde \calD} \sum_{k=1}^C \norm{v_{j}^{k*}}_2 + \norm{w_{j}^{k*}}_2 \\
		    & \geq d^*
	\end{align*}
\end{proof}
Thus, the vector-output NC-ReLU training problem~\eqref{convex-forms:eq:vector-non-convex-relu-mlp} is equivalent to the one-vs-all C-ReLU problem~\eqref{convex-forms:eq:vector-convex-relu-mlp} if the conditions of \cref{thm:vector-duality-free} are satisfied.
Further, taking \(\tilde \calD = \calD_X\) and applying \cref{thm:vector-duality-free} yields
\begin{equation*}
	{m}^{*} = \sum_{k=1}^C \sum_{D_i \in \calD_X} \abs{\cbr{v^{k*}_i : v^{k*}_i \neq 0} \cup \cbr{w^{k*}_i : w^{k*}_i \neq 0} }.
\end{equation*}
It follows that global optimization of the vector-output NC-ReLU problem requires \(m \geq m^*\) neurons, where \(m^* \leq C(n+1)\).

The gated ReLU analogs to vector-output ReLU architectures can be formulated in the same fashion.

\newpage


\section{Equivalence of ReLU and Gated ReLU: Proofs}\label{app:relu-grelu-equivalence}

First we give a simple lemma that will be useful when characterizing the span of \( \calK_i - \calK_i \).

\begin{lemma}\label{lemma:affine-characterization}
    The cone \( \calK_i \) has a non-empty interior if and only if \( \calK_i - \calK_i = \R^d \).
\end{lemma}
\begin{proof}
    Let \( x \in \text{aff} (\calK_i) \).
    Then \( x = \sum_{j=1}^m \alpha_j y_j \), \( y_j \in \calK_i \).
    Let \( j \in [m] \).
    If \( \alpha_j \geq 0 \), then \( \alpha_j y_j \in \calK_i \) since \( \calK_i \) is a cone.
    Otherwise, \( \alpha_j y_j \in - \calK_i \).
    Either way, we have \( \alpha_j y_j \in \calK_i - \calK_i \).

    Observe that \( \calK_i - \calK_i \) is a convex cone since \( \calK_i \) is a convex cone.
    Thus, \( \alpha_1 y_1 + \alpha_2 y_2 \in \calK_i - \calK_i \).
    Induction on \( j \in [m] \) now implies \( x \in \calK_i - \calK_i \) and thus \( \text{aff} (\calK_i) \subseteq \calK_i - \calK_i \).
    Now suppose \( x \in \calK_i - \calK_i \) so that \( x = y_1 - y_2 \), where \( y_1, y_2 \in \calK_i \).
    It is trivial to deduce \( x \in \text{aff} (\calK_i) \);
    we conclude that \( \text{aff} (\calK_i) = \calK_i - \calK_i \).

    Since \( 0 \in \text{aff} (\calK_i) \), this set is a linear subspace of \( \R^d \).
    If \( \calK_i \) has an interior point, then \( \text{aff} (\calK_i) = \R^d \) and we must have \( \text{aff} (\calK_i) = \calK_i - \calK_i = \R^d \).
    On the other hand, if \( \calK_i \) does not have an interior point, then \( \text{aff} (\calK_i) \subset \R^d \) and \( \text{aff} (\calK_i) = \calK_i - \calK_i \subset \R^d \) must hold;
    we have shown the reverse implication by the contrapositive.
\end{proof}

Now we show that \( \calK_i \) has an interior point when \( X \) is full row-rank.
The proof proceeds by studying a relative interior point of \( \calK_i \).

\nonSingularCones*
\begin{proof}
    Let \( \bar w \in \text{relint} (\calK_i) \), which exists since the relative interior of a non-empty convex set is non-empty \citep{bertsekas2009convex}.
    Assume that the inequality,
    \[
        (2 D_i - I) X \bar w \succeq 0,
    \]
    is tight for at least one index \( j \in [n] \);
    let \( X_{\calI} \) be the submatrix of \( X \) formed by the rows of \( X \) for which the inequality
    is tight.
    Define \( \tilde D =  (2 D_i - I) \).
    Since \( X \) is full row-rank, the rows of \( X_{\calI} \) are linearly independent.
    Let \( x_k \) be an arbitrary row of \( X_{\calI} \) (noting \( x_k \neq 0 \) by linear independence) and define \( z_k \) to be the component of \( x_k \) which is orthogonal to the remaining rows of \( X_{\calI} \).
    Clearly such a vector exists since the rows of \( X_{\calI} \) are linearly independent.
    Define \( w' = \bar w + [\tilde D_i]_{kk} z_k \) to obtain
    \[
        [\tilde D_i]_{kk} x_k^\top w' = [\tilde D_i]_{kk} x_k^\top \bar w + \norm{z_k}_2^2 = \norm{z_k}_2^2 > 0,
    \]
    and, for \( j \neq k \),
    \[
        [\tilde D_i]_{jj} x_j^\top w' = [\tilde D_i]_{jj} x_{j}^\top \bar w + [\tilde D_i]_{jj} [D_i]_{kk} x_{j}^\top z_k \geq 0,
    \]
    since \( x_j \) and \( z_k \) are orthogonal.
    This contradicts \( \bar w \in \text{relint} (\calK_i) \) and we deduce that \( (2 D_i - I) X \bar w \succ 0 \).
    \cref{lemma:affine-characterization} now implies that \( \calK_i - \calK_i = \R^d \).

    Let \( u^* \) be an optimal solution to the C-GReLU problem with \( \tilde \calD = \calD_X \).
    Since the Minkowski difference \( \calK_i - \calK_i \) spans \( \R^d \) for every \( D_i \in \calD_X \), we can find \( v_i, w_i \)  such that \( u_i^* = v_i - u_i \).
    Moreover, we can always reparameterize the optimal solution to the C-ReLU problem as \( u_i = v_i^* - w_i^* \).
    A simple reduction argument now shows the two problems are equivalent.
    Applying theorems~\ref{thm:duality-free} and~\ref{thm:unconstrained-equivalence} extends the equivalence to NC-ReLU and NC-GReLU.
\end{proof}

The main difficulty extending \cref{prop:non-singular-cones} to general, full-rank \( X \) is showing that none of the cone-constraints are tight at \( \bar w \).
Unfortunately, the following shows that these difficulties cannot be resolved.

\begin{proposition}\label{prop:col-rank-counter-example}
    There exists a full-rank data matrix \( X \) and activation pattern \( D_i \in \calD_X \) such that \( \calK_i \) is contained in a linear subspace of \( \R^d \).
\end{proposition}
\begin{proof}
    Let \( d = 3 \), \( n = 4 \) and take
    \begin{equation*}
        X =
        \begin{bmatrix}
            1  & 0  & 0  \\
            0  & 1  & 0  \\
            -1 & -1 & 0  \\
            0  & 0  & 1.
        \end{bmatrix}
    \end{equation*}
    It is easy to see that \( X \) is full-rank, although it does not have full row-rank since \( x_1, x_2, x_3 \) are collinear.
    The cone \( \calK_i = \cbr{ w : X w \succeq 0 } \), which corresponds to positive activations for each example, has the following alternative representation:
    \[
        \calK_i = \cbr{\alpha * e_3 : \alpha \geq 0 }.
    \]
    Clearly \( \calK_i \) is contained in a subspace of dimension one.
    Thus, we cannot hope for \( \calK_i \) to have full affine dimension in this more general setting.
\end{proof}

\subsection{Singular Cones are Contained in Non-Singular Cones}\label{app:singular-cones}

Considering the counter-example in \cref{prop:col-rank-counter-example}, we find the ``bad'' \( \calK_i \) is contained within the subspace \( \calS \) spanned by \( e_3 \).
By construction, every \( w \in \calK_i \) is orthogonal to \( x_1, x_2, \) and \( x_3 \), meaning these examples don't contribute to the constraints on \( \calK_i \) once it is restricted to \( \calS \).
Intuitively, changing the activation associated with \( x_1, x_2 \), or \( x_3 \) can only lead to cones which contain \( \calK_i \).
For example, consider \( \calK_i' = \cbr{w : \abr{x_3, w} \leq 0}, X_{-3} w \succeq 0 \), which is equal to the non-negative orthant, \( \R^3_+ \).
We immediately observe \( \calK_i \subset \calK_i' \) and we may replace the degenerate cone with the alternative, full-dimensional \( \calK_2 \).
The rest of this section formalizes these observations.

\begin{definition}\label{def:minimal-constraints}
    Let \( \tilde X \in \R^{m \times d} \) and consider a cone
    \( \calK = \cbr{ w: \tilde X w \succeq 0} \)
    such that \( \text{aff} (\calK) = \calS \subset \R^d \).
    We call an index set \( \calI \subseteq [m] \) minimal for \( S \) if
    \[ \calK_{\calI} = \cbr{w : \tilde X_{\calI} w \succeq 0} \subseteq \calS \]
    and, for any \( j \in \calI \),
    \[
        \calK_{\calI \setminus j} \not \subseteq \calS.
    \]
    That is, removing any half-space constraint indexed by \( \calI \) ensures \( \calK_{\calI} \) is not contained in \( \calS \).
\end{definition}

Note that there may be many minimal index sets for a singular cone and these sets may have varying cardinalities.
However, each minimal index set shares a key property: every row \( \tilde x_i \) indexed by such \( \calI \) must be orthogonal to \( \calS \).

\begin{lemma}\label{lemma:orthogonality}
    Let \( \tilde X \in \R^{m \times d} \) such that the cone
    \( \calK = \cbr{ w: \tilde X w \succeq 0} \)
    is singular.
    Let \( \calS = \text{aff} (\calK) \) be the smallest containing subspace and \( \calI \) a minimal index set for \( \calS \).
    Then, \( \abr{\tilde x_i, s} = 0 \) for all \( i \in \calI \) and \( s \in \calS \).
\end{lemma}
\begin{proof}
    Suppose \( \abr{\tilde x_i, w} \neq 0 \) for some \( i \in \calI \) and \( w \in \calK \).
    Since \( w \in \calK \), \( \tilde X w \succeq 0 \) and it must be that \( \abr{\tilde x_i, w} > 0 \).
    Let \( z \in \calS^{\perp} \) be arbitrary and define \( w' = z + \alpha w \), \( \alpha > 0 \).
    By taking \( \alpha \) to be sufficiently large, we obtain
    \[
        \abr{\tilde x_i, w'} = \abr{\tilde x_i, z}  + \alpha \abr{\tilde x_i, w} > 0,
    \]
    Since \( w' \not \in \calS \supseteq \calK \), we must have
    \[ \tilde X_{\calI \setminus i} w' \not \succeq 0 \implies \tilde X_{\calI \setminus i} z \not \succeq 0, \]
    where we have used \( X w \succeq 0 \).
    Moreover, this holds for all \( z \in \calS^{\perp} \), which implies that \( \calK_{\calI \setminus i} \subseteq \calS \) and \( \calI \) cannot be minimal for \( \calS \).
    We conclude \( \abr{\tilde x_i, w} = 0 \) for all \( i \in \calI \) and \( w \in \calK \) by contradiction.

    Since \( \calS \) is the affine hull of \( \calK \), we have for every \( s \in \calS \) and \( i \in \calI \) the following:
    \[
        \abr{\tilde x_i, s} = \abr{\tilde x_i, \sum_{j=1}^k \alpha_j y_j} = \sum_{j=1}^k \alpha_j \abr{ \tilde x_i, y_j} = 0,
    \]
    since \( y_j \in \calK \).
\end{proof}

Similarly, if any constraint is tight at a relative interior point, then that constraint must be orthogonal to the cone.

\begin{lemma}\label{lemma:tight-constraints}
    Let \( \tilde X \in \R^{m \times d} \), \( \calK = \cbr{ w : \tilde X w \succeq 0} \), and \( \bar w \) be a relative interior point of \( \calK \).
    If \( \abr{\tilde x_j, \bar w} = 0 \) for any \( j \in [m] \), then \( \tilde x_j \) is orthogonal to \( \calK \).
\end{lemma}
\begin{proof}
    Suppose \( \abr{\tilde x_j, \bar w} = 0 \) for some \( j \in [m] \).
    If there exists \( w \in \calK \) such that \( \abr{\tilde x_j, \bar w} > 0 \), then \( \abr{\tilde x_j, \bar w + w} > 0 \) and \( \bar w + w \in \calK \), which contradicts the assumption \( \bar w \) is a relative interior point.
    Since every \( w \in \calK \) satisfies \( \abr{\tilde x_j, w} \geq 0 \), we conclude \( \abr{\tilde x_j, w} = 0 \) for all such \( w \).
\end{proof}

\cref{lemma:orthogonality} is key to our analysis because it implies that the half-space constraints which force \( \calK \) to lie in a subspace don't ``cut into'' that subspace.
In particular, it means that we can choose to enforce membership in \( \calH_{x_i} \) or \( \calH_{- x_i} \) without changing the inclusion.
We show now that there exists a choice of signed half-spaces for which the intersection is non-singular.

\begin{lemma}\label{lemma:sign-assignment}
    Let \( x_1, \ldots x_m \) be a collection of vectors in \( \R^d \) and \( X \in \R^{m \times d} \) the matrix formed by stacking these vectors.
    Then there exists a diagonal matrix \( \tilde D \), where \( \tilde D_{jj} \in \cbr{-1, 1} \), such that
    \[
        \text{aff} (\cbr{ w : \tilde D X w \succeq 0}) = \R^d.
    \]
\end{lemma}
\begin{proof}
    We proceed by induction.
    Let \( \tilde \calD_{11} = 1 \) and \( \calK_1 = \calH_{x_1} := \cbr{w : \abr{x_1, w} \geq 0} \).
    Clearly \( \text{aff} (\calK_1) = \R^d \) since it is a half-space.

    Now, let \( t < m \) and assume that \( \text{aff} (\calK_t) = \R^d \).
    Consider
    \[
        A_{t+1} := \calK_t \cap \calH_{x_{t+1}}.
    \]
    If \( \calS := \text{aff} (A_{t+1}) = \R^d \), then the inductive hypothesis holds at \( \calK_{t+1} = A_{t+1} \) and we can choose \( \tilde D_{t+1, t+1} = 1 \).
    Otherwise, \( x_{t+1} \) must be part of a minimal index set \( \calI \subseteq [t+1] \) such that \( \cbr{w : \tilde D_{\calI} X_{\calI} w \succeq 0 } \subseteq \calS \).
    \cref{lemma:orthogonality} now implies that \( x_{t+1} \) is orthogonal to \( \calS \).
    Let \( w \in \calK_{t} \cap \calS^{\perp} \) (which is non-empty by the inductive hypothesis) and observe that
    \[
        \abr{w, x_{t+1}} < 0,
    \]
    must hold, otherwise \( w \in A_{t+1} \).
    We deduce \( \abr{w, x_{t+1}} \leq 0 \) for every \( w \in \calK_{t} \) and thus
    \[
        \calK_t \cap \calH_{- x_{t+1}} = \calK_t,
    \]
    which is full-dimensional by the inductive hypothesis.
    Taking \( \calK_{t+1} = \calK_t \) and \( \calD_{t+1, t+1} = -1 \) completes the case.

    The desired result follows by induction.

\end{proof}

We now use \cref{lemma:sign-assignment} to show that every singular cone is contained in a non-singular cone.

\coneContainment*
\begin{proof}
    For simplicity, we drop the index \( i \) and work with \( \calK = \cbr{w : (2 D - I) X w \succeq 0 } \).
    To ease the notation, we also write \( \tilde D = (2 D - I) \) and \( \tilde X = \tilde D X \).
    Let \( \calS = \text{aff} (\calK) \) be the smallest subspace containing \( \calK \),
    \( \calO \) be the set of all indices \( j \) such that \( x_j \) is orthogonal to \( \calS \),
    and \( \calU = [n] \setminus \calO \).
    \cref{lemma:sign-assignment} implies that there exists an alternative activation pattern \( \tilde D' \) such that,
    \( \tilde D_{\calU}' = \tilde D_{\calU} \), and
    if \( \calK' = \cbr{w : \tilde D' X w \succeq 0 } \), then
    \[
        \calK'_{\calO} := \cbr{w : \tilde D_{\calO}' X_\calO w \succeq 0} \quad \text{satisfies} \quad \text{aff} \rbr{\calK_\calO'} = \R^d.
    \]
    Since the vectors indexed by \( \calO \) are orthogonal to \( \calS \), they are also orthogonal to every \( w \in \calK \), implying \( \calK \subset \calK' \).
    In other words, the change of activation signs preserves inclusion of \( \calK \).

    Let us show that \( \calK' \) contains an interior point.
    Let \( \bar w \) be a relative interior point of \( \calK \) and suppose that \( \abr{\tilde x_j, \bar w} = 0 \) for some \( j \in \calU \).
    \cref{lemma:tight-constraints} implies \( \tilde x_j \) is orthogonal to \( \calK \).
    But, then \( j \in \calO \), which is a contradiction.
    We conclude \( \tilde X_{\calU }\bar w \succ 0\).

    Let \( \bar y \) be an interior point of \( \calK'_{\calO} \), which exists because \( \text{aff} (\calK'_{\calO}) = \R^d \).
    The point \( \bar z = \bar y + \alpha \bar w \), \( \alpha > 0 \) satisfies
    \[
        \tilde D_{\calO}' X_{\calO} \bar z = \tilde D_{\calO}' X_{\calO} \bar y \succ 0,
    \]
    since \( \bar w \) is orthogonal to the rows of \( X_{\calO} \).
    Similarly, by taking \( \alpha \) to be sufficiently large, we have
    \[
        \tilde D_{\calU}' X_{\calU} \bar z \succ 0.
    \]
    We have shown that \( \bar z \) is an interior point of \( \calK' \)
    and \cref{lemma:affine-characterization} now implies \( \calK' - \calK' = \R^d \).
\end{proof}

\coneElimination*
\begin{proof}
    Assume \( \lambda = 0 \).
    The equivalence of the C-ReLU and C-GReLU problems with sub-sampled
    patterns \( \tilde \calD \setminus \calS(\tilde\calD) \) follows
    immediately from cone decompositions.
    Let \( u^* \) be a solution to the C-GReLU problem with
    \( \tilde \calD \setminus \calS(\tilde\calD) \)
    and \( p^* \) the optimal value.
    Let \( d^* \) be the optimal value of the C-ReLU problem with the same
    patterns.
    For every \( D_i \in \tilde \calD \setminus \calS(\tilde\calD) \),
    \( \calK_i - \calK_i = \R^d \) by construction.
    Thus, there exists \( v_i, w_i \in \calK_i \) such that
    \( v_i - w_i = u_i^* \).
    It is now straightforward to conclude that,
    \[
        p^* =
        \calL\rbr{\sum_{D_i \in \tilde \calD \setminus \calS(\tilde\calD)}
            D_i X u_i^*, y}
        =
        \calL\rbr{\sum_{D_i \in \tilde \calD \setminus \calS(\tilde\calD)}
            D_i X (v_i - w_i), y}
        \geq d^*,
    \]
    since \( v_i, w_i \) are feasible for the C-ReLU problem.
    Letting \( v^*, w^* \) be a solution to the C-ReLU problem,
    we obtain the opposite inequality immediately be noting that
    \( u_i = v^*_i - w^*_i \) is feasible for the C-GReLU problem.
    Thus, the two problems are equivalent when singular cones are omitted.
    Now let us show that singular cones do not affect the optimal value
    of the C-ReLU problem.

    Assume \( \lambda \geq 0 \) and let \( (v^*, w^*) \) be a solution to the
    C-ReLU problem with patterns \( \tilde\calD \) and \( p^* \) the optimal
    value.
    For every \( D_i \in \calD_{X} \), we either have \( \calK_i - \calK_i = \R^d \), or not.
    If the former condition holds, take \( r_i^* = v_i^* \) and \( q_i^* = w_i^* \).
    Otherwise, invoke \cref{prop:cone-containment} to obtain \( D_j \in \calD_{X} \)
    such that \( \calK_i \subset \calK_j \) and \( \calK_j - \calK_j = \R^d \).
    By the construction in the proof of \cref{prop:cone-containment},
    if \( \sbr{D_j}_{kk} \neq \sbr{D_i}_{kk} \),
    then \( x_k \) is orthogonal to \( v_i^*, w_i^* \in \calK_i \) and we deduce
    \[
        D_j X (v_i^* - w_i^*)  = D_i X (v_i^* - w_i^*).
    \]
    We may therefore merge the two neurons as \( r^*_j = v^*_j + v^*_i \),
    \( q_j^* = w_j^* + w_i^* \) and update
    \( \tilde \calD' = \tilde \calD \setminus D_i \) without changing the
    (optimal) prediction of the C-ReLU model.
    In this way, we obtain a sub-sampled C-ReLU problem with activation patterns
    \( \tilde \calD \setminus \calS(\tilde \calD) \),
    for which
    \begin{align*}
        p^* & = \calL\rbr{\sum_{D_i \in \tilde \calD} D_i X (v_i^* - w_i^*), y}
        + \lambda \sum_{D_i \in \tilde \calD} \norm{v_i^*}_2 + \norm{v_i^*}_2
        \\
            & \geq \calL\rbr{\sum_{D_j \in \tilde \calD \setminus \calS(\tilde \calD)} D_j X (r_j^* - q_j^*), y}
        + \lambda \sum_{D_j \in \tilde \calD \setminus \calS(\tilde \calD)} \norm{r_j^*}_2 + \norm{q_j^*}_2                              \\
            & \geq \min_{r_j, q_j \in \calK_j} \calL\rbr{\sum_{D_j \in \tilde \calD \setminus \calS(\tilde \calD)} D_j X (r_j - q_j), y}
        + \lambda \sum_{D_j \in \tilde \calD \setminus \calS(\tilde \calD)} \norm{r_j}_2 + \norm{q_j}_2
        := d^*,
    \end{align*}
    where we used triangle inequality to imply \( \norm{r_j^*}_2 \leq
    \norm{v_j^*}_2 + \norm{v_i^*}_2 \) also the fact that \( r_i^*, q_i^* \)
    are feasible because \( \calK_j \) is a convex cone and closed under
    addition.
    It is now straightforward to conclude \( p^* = d^* \) since sub-sampling
    the C-ReLU problem can only increase the optimal objective value.
\end{proof}

\subsection{Approximating ReLU by Cone Decomposition: Proofs}\label{app:cone-decompositions}
\closedFormDecomp*
\begin{proof}
    Let's show that the decomposition is valid.
    Setting \( v = u + w \),
    we obtain \( v - w = u \) by construction.
    For notational ease, let \( \calJ = [n] \setminus \calI \).
    It holds that
    \[
        \tilde X_{\calI} w = - \tilde X_{\calI} \tilde X_{\calI}^\dagger \tilde X_{\calI} u = - \tilde X_{\calI} u > 0.
    \]
    Moreover, since \( X \) is full row-rank and \( \calJ \cap \calI = \emptyset \), we have
    \[
        \tilde X \tilde X^\dagger = \tilde X \tilde X^\top \rbr{\tilde X \tilde X^\top}^{-1} = I,
    \]
    which implies that \( \tilde X_{\calJ} \tilde X_{\calI}^\dagger = 0 \).
    We deduce
    \[
        \tilde X_{\calJ} w = - \tilde X_{\calJ} \tilde X_{\calI}^\dagger \tilde X_{\calI} u = 0.
    \]
    Moreover, we also have
    \[
        \tilde X_{\calI} v = \tilde X_{\calI} u - \tilde X_{\calI} u = 0,
    \]
    and
    \[
        \tilde X_{\calJ} v = \tilde X_{\calJ} u + 0 \geq 0,
    \]
    by definition of \( \calJ \).
    We conclude \( w, v \in \calK \).

    To show the approximation result, start from
    \begin{align*}
        \norm{w}_2
         & = \norm{\tilde X_{\calI}^\dagger \tilde X_{\calI} u}_2             \\
         & \leq \norm{\tilde X_{\calI}^\dagger \tilde X_{\calI}}_2 \norm{u}_2 \\
         & = \norm{u}_2.
    \end{align*}
    Triangle inequality now gives,
    \begin{align*}
        \norm{v}_2
         & = \norm{u - \tilde X_{\calI}^\dagger \tilde X_{\calI} u}_2               \\
         & \leq \norm{(I - \tilde X_{\calI}^\dagger \tilde X_{\calI})}_2 \norm{u}_2 \\
         & = \norm{u}_2,
    \end{align*}
    and summing these two inequalities gives the result.
\end{proof}

\worstCaseNorm*
\begin{proof}
    Consider the data matrix
    \begin{equation}
        X =
        \begin{bmatrix}
            1   & \alpha \\
            - 1 & \alpha \\
        \end{bmatrix}
    \end{equation}
    The cone corresponding to positive activations for both examples is \( \calK = \cbr{x \in \R^d : x_2 \geq - x_1 / \alpha, x_2 \geq x_1 / \alpha  } \).
    Consider decomposing the vector \( u = [2, 0] \) onto \( \calK - \calK \).
    Clearly \( u \not \in \calK \); by inspection, we see that the minimum norm decomposition is given by \( v = [1, 1/\alpha] \), and \( w = [-1, 1/\alpha] \).
    Taking \( \alpha \rightarrow 0 \), we find \( \norm{v}_2 = \norm{w}_2 \rightarrow \infty \).
\end{proof}

\socpDecomp*
\begin{proof}
    First, we re-parameterize the problem: \( u = v - w \) implies \( u + w = v \), giving the equivalent program
    \begin{equation}
        \begin{aligned}
            \min_{w \in \calK} \norm{w}_2 + \norm{u + w}_2^2.
        \end{aligned}
    \end{equation}
    In order to character the solution, we re-write the constraints into a single system of linear inequalities as follows:
    \begin{align*}
        \calF
              & = \cbr{w : \tilde X w \succeq 0, \tilde X w \succeq - \tilde X u} \\
              & = \cbr{w : \tilde X w \succeq 0, \tilde X w \succeq b},
        \intertext{where we have introduced \( b = - \tilde X u \). It is possible to combine these inequalities by taking the element-wise maximum as follows:}
        \calF & = \cbr{w : \tilde X w \succeq \max\rbr{0, b}}                     \\
              & = \cbr{w : \tilde X w \succeq \rbr{b}_+}.
    \end{align*}
    Let \( \bar w \in \calF \) be a optimal point for the reparameterized program.
    Relaxing the objective using triangle inequality gives,
    \begin{align*}
        \norm{\bar w}_2 + \norm{u + \bar w}_2
         & = \min_{w \in \calF} \norm{w}_2 + \norm{u + w}_2  \\
         & \leq \min_{w \in \calF} 2\norm{w}_2 + \norm{u}_2.
    \end{align*}
    Let \( w' \) be a solution to the relaxation.
    The KKT conditions imply there exists a submatrix \( \tilde X_{\calJ} \) for which the inequality constraints are tight:
    \[
        \tilde X_{\calJ} w' = \rbr{b_{\calJ}}_+.
    \]
    The set of vectors satisfying this equality is \( \cbr{\sbr{\tilde X_{\calJ}}^\dagger \rbr{b_{\calJ}}_+ + z : z \in \text{null} (X_{\calJ})} \).
    Choosing \( z \neq 0 \) can only increase the value of the objective, from which we deduce \( w' = \sbr{\tilde X_{\calJ}}^\dagger \rbr{b_{\calJ}}_+ \).
    We obtain
    \begin{align*}
        \norm{\bar w}_2 + \norm{\bar v}_2
         & \leq 2\norm{w'}_2 + \norm{u}_2                                                                                  \\
         & = 2 \norm{\sbr{\tilde X_{\calJ}}^\dagger \rbr{b_{\calJ}}_+} + \norm{u}_2                                        \\
         & \leq \frac{2}{\sigma_{\text{min}}(\tilde X_{\calJ})} \norm{\rbr{b_{\calJ}}_+}_2 + \norm{u}_2                    \\
         & \leq \frac{2}{\sigma_{\text{min}}(\tilde X_{\calJ})} \norm{b_{\calJ}}_2 + \norm{u}_2                            \\
         & =  \frac{2}{\sigma_{\text{min}}(\tilde X_{\calJ})} \norm{\tilde X_{\calJ} u}_2 + \norm{u}_2                     \\
         & \leq \rbr{1 + 2\frac{\sigma_{\text{max}}(\tilde X_{\calJ})}{\sigma_{\text{min}}(\tilde X_{\calJ})}} \norm{u}_2.
    \end{align*}
\end{proof}

\approxResult*
\begin{proof}
    The proof is straightforward given our existing results.
    Let \( u^* \) be the solution to the full (potentially regularized) C-GReLU problem.
    For every \( D_i \in \calD_{X} \), we either have \( \calK_i - \calK_i = \R^d \), or not.
    If the former condition holds, take \( r_i^* = u_i^* \).
    Otherwise, invoke \cref{prop:cone-containment} to obtain \( D_j \in \calD_{X} \) such that \( \calK_i \subset \calK_j \) and \( \calK_j - \calK_j = \R^d \).
    By the construction in the proof of \cref{prop:cone-containment}, if \( \sbr{D_j}_{kk} \neq \sbr{D_i}_{kk} \), then \( x_k \) is orthogonal to \( \calK_i \) and we deduce
    \[
        D_j X u^*_i = D_i X u^*_i.
    \]
    We may therefore merge the two neurons as \( r^*_j = u^*_j + u^*_i \) and update \( \tilde \calD' = \tilde \calD \setminus D_i \) without changing the loss component of the C-GReLU program.
    Furthermore, since \( \norm{r^*_j} \leq \norm{u^*_j}_2 + \norm{u^*_i}_2 \), merging these neurons can only decrease the regularization term.
    \footnote{In fact, we know that one of \( u^*_j, u^*_i \) is zero or they are collinear}
    In this way, we obtain a sub-sampled C-GReLU problem with activation patterns \( \tilde \calD \) and optimal value \( d^* \).

    Let \( \rbr{v^*, u^*} \) be the optimal solution to full C-ReLU problem.
    Applying \cref{prop:socp-decomp} for each \( D_i \in \tilde \calD \) gives decompositions \( u_i^* = v_i' - u_i' \) such that
    \begin{align*}
        d^*
         & = \mathcal{L}\Big(  \sum_{D_i \in \tilde \calD} D_i X u_i^*, y\Big) + \lambda \sum_{D_i \in \tilde \calD} \norm{u_i^*}_2                                                                                                                                               \\
         & \leq \mathcal{L}\Big(  \sum_{D_i \in \tilde \calD} D_i X v^*_i - w_i^*, y\Big) + \lambda \sum_{D_i \in \tilde \calD} \norm{v_i^* - w_i^*}_2                                                                                                                            \\
         & \leq \mathcal{L}\Big(  \sum_{D_i \in \tilde \calD} D_i X v^*_i - w_i^*, y\Big) + \lambda \sum_{D_i \in \tilde \calD} \norm{v_i^*}_2 + \norm{w_i^*}_2                                                                                                                   \\
         & = p^*                                                                                                                                                                                                                                                                  \\
         & \leq \mathcal{L}\Big(  \sum_{D_i \in \tilde \calD} D_i X v'_i - w_i', y\Big) + \lambda \sum_{D_i \in \tilde \calD} \norm{v_i'}_2 + \norm{w_i'}_2                                                                                                                       \\
         & = \mathcal{L}\Big(  \sum_{D_i \in \tilde \calD} D_i X u^*_i, y\Big) + \lambda \sum_{D_i \in \tilde \calD} \norm{v_i'}_2 + \norm{w_i'}_2                                                                                                                                \\
         & \leq \mathcal{L}\Big(  \sum_{D_i \in \tilde \calD} D_i X u^*_i, y\Big) + \lambda \sum_{D_i \in \tilde \calD} \norm{u_i^*}_2 + \lambda \sum_{D_i \in \tilde \calD} 2 \frac{\sigma_{\text{max}}(\tilde X_{\calJ})}{\sigma_{\text{min}}(\tilde X_{\calJ})} \norm{u_i^*}_2 \\
         & = p^* + \lambda \sum_{D_i \in \tilde \calD} 2 \frac{\sigma_{\text{max}}(\tilde X_{\calJ})}{\sigma_{\text{min}}(\tilde X_{\calJ})} \norm{u_i^*}_2.
    \end{align*}
    We have abused notation here and omitted the dependence on \( i \) in \( \tilde X_{\calJ} = \rbr{2 D_i -I} X_{\calJ_i} \).
    However, observe that \( 2D_i - I \) is orthonormal so that \( \sigma_{\text{max}}(\tilde X_{\calI}) = \sigma_{\text{max}}(X_{\calI}) \) and \( \sigma_{\text{min}}(\tilde X_{\calI}) = \sigma_{\text{min}}(X_{\calI}) \) for all \( \calI \).
    Maximizing over \( \calI \subseteq [n] \) now gives a fixed subset \( \calJ \) for which the claimed bound holds.
\end{proof}

\newpage

\section{Efficient Global Optimization:\ Proofs}\label{app:optimization}

\gatedComplexity*
\begin{proof}
	Applying the solution mapping from the proof of \cref{thm:unconstrained-equivalence} (see \cref{app:convex-forms}) we find that taking
	\[
		v_i' = W_{1i} * w_{2i},
	\]
	for each \( D_i \in \calD \) yields a global minimizer of C-ReLU.
	Now we apply the iteration complexity of FISTA~\citep[Theorem 4.4]{beck2009fista} to obtain an \( \epsilon \)-optimal solution in
	\begin{align*}
		T & \leq \frac{\rbr{2 \lambda_{\text{max}}\rbr{M^\top M}\sum_{D_i \in \tilde \calD} \norm{v_i'}_2^2}^{1/2}}{\epsilon^{1/2}}            \\
		  & = \frac{\rbr{2 \lambda_{\text{max}}\rbr{M^\top M}\sum_{D_i \in \tilde \calD} \norm{W_{1i}^* w_{2i}^*}_2^2}^{1/2}}{\epsilon^{1/2}},
	\end{align*}
	iterations.
	Note that we have used the fact that \( \lambda_{\text{max}}(M^\top M) \) is the Lipschitz smoothness constant of the squared-error loss.

\end{proof}

\approxDecompBounds*
\begin{proof}
	First-order optimality conditions for \( \tilde w \) imply
	\begin{align*}
		-\tilde X^\top \rbr{b - \tilde X \tilde w}_+
		 & \in \rho \cdot \partial \norm{\tilde w}_2.
	\end{align*}
	Noting that every vector in \( \partial \norm{\tilde w}_2 \) has norm at most
	\( 1 \), we deduce
	\begin{align*}
		\norm{\tilde X^\top \rbr{b - \tilde X \tilde w}_+}_2
		 & \leq \rho                                        \\
		\implies
		\norm{\rbr{b - \tilde X \tilde w}_+}_2
		 & \leq \frac{\rho}{\sigma_{\text{min}}(\tilde X)}  \\
		\iff \norm{\rbr{\max\cbr{- \tilde X u, 0} - \tilde X \tilde w}_+}_2
		 & \leq \frac{\rho}{\sigma_{\text{min}}(\tilde X)}.
		\intertext{since \( \tilde X \) is full row-rank and by definition of \( b \).
			Using the fact that only positive elements contribute to the norm, we obtain the following two inequalities:}
		\norm{\rbr{\tilde X \tilde w}_-}_2
		 & \leq \frac{\rho}{\sigma_{\text{min}}(\tilde X)}  \\
		\norm{\rbr{\tilde X (u + \tilde w)}_-}_2
		 & \leq \frac{\rho}{\sigma_{\text{min}}(\tilde X)},
	\end{align*}
	Recalling \( \tilde v = u + \tilde w \) and summing gives the first result.

	If \( \rho > 0 \), then it is easy to observe
	(i.e. by arguing via contradiction) that \( \tilde w \) must
	have smaller norm than any \( w' \) in a feasible decomposition \( (v', w') \).
	Thus, it must also have smaller norm than \( \bar w \) from the SOCP cone
	decomposition:
	\[ \norm{\tilde w}_2 \leq \norm{\bar w}_2. \]
	Since the proof of \cref{prop:socp-decomp} relies on controlling
	only the norm of \( \bar w \), the conclusion of that theorem also applies
	to \( (\tilde v, \tilde w) \).

	Finally, suppose \( X \) is not full row-rank. For \( \rho > 0 \),
	\cref{eq:cd-approx} is equivalent to solving
	\begin{equation}\label{eq:cd-approx-constr}
		\textbf{CD-A}: \, \min_{w} g(w, \rho) : =
		\half
		\|
		(b - \tilde X w)_+
		\|_2^2
		+ \rho P(w) \quad \text{s.t.} \norm{w}_2 \leq \norm{\bar w}_2,
	\end{equation}
	where \( \bar w \) is the norm of the minimum-norm (with respect to \( w \)
	only) solution to the cone decomposition problem.
	This is a minimization problem with compact constraint set;
	since \( g \) is continuous in both \( w \) and \( \rho \),
	we may apply Berge's maximum theorem~\citep{ausubel1993generalized}
	to obtain find
	\[
		g^*(\rho) = \min_{w}
		\cbr{ g(w, \rho) : \norm{w}_2 \leq \norm{\bar w}_2},
	\]
	is continuous.
	Since the cone decomposition is realizable at \( \rho = 0 \),
	\( g^*(0) = 0 \) and any sequence \( \rho_k \) converging to \( 0 \)
	satisfies
	\( \lim_k g^*(\lambda_k) = 0 \).

	Let \( \tilde w_k \) be the sequence of minimizers associated with \( \rho_k \).
	Since \( \tilde w_k \) is bounded, it has at least one convergent subsequence.
	Let \( \tilde w_0 \) be the associated limit point.
	Since \( g \) is continuous in \( w \) and \( \rho \), we find
	\[ g(w_0, 0) = \lim_k g(\tilde w_k, \rho_k) = \lim_k g^*(\rho_k) = 0, \]
	which shows that \( \tilde w_0 \) is a feasible decomposition.
	This completes the proof.
\end{proof}

\begin{restatable}{proposition}{approxDecompSubmatrices}\label{prop:approx-decomp-submatrices}
	Suppose \( \tilde w \) is a minimizer of \eqref{eq:cd-approx} and let \( \tilde v = u + \tilde w \).
	There exists \( \calJ \subseteq [n] \) such that
	\[
		\|(\tilde X \tilde w)_-\|_2 + \|(\tilde X \tilde v)_-\|_2
		\leq \frac{2\rho}{\sigma_{\text{min}}(\tilde X_{\calJ})},
	\]
	where \( \sigma_{\text{min}}(\tilde X_{calJ}) \) is the minimum
	(possibly zero) singular value of the sub-matrix \( \tilde X_{\calJ} \).
\end{restatable}
\begin{proof}
	First-order optimality conditions for \( \tilde w \) imply
	\begin{align*}
		-\tilde X^\top \rbr{b - \tilde X \tilde w}_+
		 & \in \rho \cdot \partial \norm{\tilde w}_2.
	\end{align*}
	Let \( \calJ = \cbr{i \in [n] : b_i - \abr{\tilde x_i, \tilde w} > 0 } \)
	and define \( B \) be a diagonal matrix such that \( B_{jj} = 1 \) if
	\( j \in \calJ \) and \( 0 \) otherwise.
	In this notation, the optimality conditions can be written as
	\begin{align*}
		-\tilde X^\top B\rbr{b - \tilde X \tilde w}
		 & \in \rho \cdot \partial \norm{\tilde w}_2  \\
		\iff
		-(B \tilde X^\top)\rbr{b - \tilde X \tilde w}
		 & \in \rho \cdot \partial \norm{\tilde w}_2.
	\end{align*}
	Noting that every vector in \( \partial \norm{\tilde w}_2 \) has norm at most
	\( 1 \), we deduce
	\begin{align*}
		\norm{(B \tilde X)^\top \rbr{b - \tilde X \tilde w}}_2
		 & \leq \rho                                                \\
		\iff
		\norm{\tilde X_{\calJ}^\top \rbr{b - \tilde X \tilde w}_{\calJ}}_2
		 & \leq \rho                                                \\
		\implies
		\norm{\rbr{b - \tilde X \tilde w}_{\calJ}}_2
		 & \leq \frac{\rho}{\sigma_{\text{min}}(\tilde X_{\calJ})}.
	\end{align*}
	Recalling the definition of \( \calJ \) and proceeding as in the proof of
	\cref{prop:approx-decomp-bounds} gives the desired result.
\end{proof}

\begin{remark}
	The bound given in \cref{prop:approx-decomp-submatrices} may be
	vacuous when \( \tilde X_{\calJ} \) is not full row-rank since it
	concerns the minimum singular value, rather than the minimum \emph{non-zero}
	singular value.
	In this respect, it is unlike the other results given so which have relied
	on the minimum non-zero singular value.
	However, it is worth reporting since this bound may be non-vacuous even when
	\( \tilde X \) is not full row-rank and only the asymptotic result in
	\cref{prop:approx-decomp-bounds} applies.
\end{remark}

\begin{restatable}{theorem}{alConvergenceRate}\label{thm:al-convergence-rate}
	Let \( \rbr{\gamma^*, \zeta^*} \) be the minimum-norm maximizer of the Lagrange dual of Problem~\ref{convex-forms:eq:convex-relu-mlp}.
	Assume \( \delta > 0 \) is fixed and at each iteration \cref{eq:al-subroutine} is carried out exactly.
	Then, the AL method computes an \( \epsilon \)-optimal estimate \( \rbr{\gamma_\epsilon, \zeta_\epsilon} \) in
	\begin{equation*}
		T \leq \frac{\norm{\gamma^*}_2^2 + \norm{\zeta^*}_2^2}{\delta \epsilon},
	\end{equation*}
\end{restatable}
\begin{proof}
	Let \( d \) be the Lagrange dual function associated the Problem~\ref{convex-forms:eq:convex-relu-mlp}.
	We will show that the desired iteration complexity follows from standard results in the optimization literature.

	Firstly, it is well-known that if the primal objective is a proper, closed, convex function, then one iteration of the AL method with penalty strength \( \delta > 0 \) is equivalent to the following proximal-point step on the dual problem:
	\begin{align*}
		\rbr{\gamma_{k+1}, \zeta_{k+1}} & = \argmax_{\gamma \geq 0, \zeta \geq 0} \cbr{d(\gamma, \zeta) - \frac{1}{2 \delta} \sbr{\norm{\gamma - \gamma_k}^2 \norm{\zeta - \zeta_k}^2} }.
	\end{align*}
	See \citet[Section 5.4.6]{bertsekas1997nonlinear} for a proof of this fact.

	Invoking \citet[Theorem 2.1]{guler1991convergence} implies that the AL method attains the following convergence rate for the dual parameters:
	\begin{equation}
		d(\gamma^*, \zeta^*) - d(\gamma_k, \zeta_k) \leq \frac{\norm{\gamma^* - \gamma_0}_2^2 + \norm{\zeta^* - \gamma_0}_2^2}{\delta k}.
	\end{equation}
	Choosing \( \gamma_0 = \zeta_0 = 0 \) and re-arranging this equation gives the desired iteration complexity.
\end{proof}

\subsection{Data Normalization}\label{app:data-normalization}

Recall that the proximal gradient update has the form,
\begin{align*}
	\xkk & = \argmin_{x} \cbr{f(\xk) + \abr{\grad(\xk), x - \xk} + \frac{1}{2 \etak}\norm{x - \xk}_2^2 + g(x)}                         \\
	     & = \argmin_{x} \cbr{\frac{1}{2\etak} \norm{x - \rbr{\xk - \etak \grad(\xk)}}_2^2 + g(x)}.
	\intertext{Taking \( g(x) \) to be the group \( \ell_1 \) penalty, we have}
	\xkk & = \argmin_{x} \cbr{\frac{1}{2\etak} \norm{x - \rbr{\xk - \etak \grad(\xk)}}_2^2 + \lambda \sum_{g \in \calG} \norm{x_g}_2},
\end{align*}
where \( \calG \) is the set of group indices.
Letting \( \x^{+} = \xk - \etak \grad(\xk) \), the update takes the form (see, e.g. \citet[Section 2.3]{sra2012optimization}),
\begin{align*}
	\sbr{\xkk}_g & = \rbr{1 - \frac{\lambda}{\norm{\x^{+}_g}_2}}_+ \x^{+}_g,
\end{align*}
which establishes our claim that the proximal step~\eqref{unconstrained:eq:line-search-cond} is a thresholding operator.

Thresholding operators are sensitive to rounding and other forms of numerical error.
Indeed, it is not hard to see that using a perturbed gradient \( \hat \grad(\xk) = \grad(\xk) + \epsilon \) can lead to groups dropping out of the model (or staying in the model) when they should remain non-zero.
Thus, it is important to reduce numerical error as much as possible by improving the condition of other operations, like computing \( \grad(\xk) \).
We can use data normalization to partially achieve this goal.

In the remainder of this section, we restrict ourselves to the C-GReLU problem with squared loss,
\[
	\min_{v} \; \half \norm{\sum_{D_i \in \tilde \calD} D_i X v_i - y}_2^2 + \lambda \sum_{D_i \in \tilde \calD} \norm{v_i}_2.
\]
The Hessian of the smooth component of this problem is \( \nabla^2 f(v) = M^\top M \), where \( M \) is the ``expanded'' data matrix \( M = \sbr{D_1 X \; D_2 X \; \ldots D_{|\tilde \calD|} X} \).
Let \( h \in \R^{d}, h_i = \norm{X_{\cdot,i}}_2 \) and \( H = \text{diag}(h) \).
That is, \( H \) is a diagonal matrix with the column-norms of \( X \) along the diagonal.
Finally, define the column-normalized version of \( X \) to be \( N = H^{-1} X \).

It is not hard to see that the diagonal elements of \( N^\top N \) are \( 1 \) by construction.
Applying a trace bound, we have
\begin{equation}
	\lambda_{\text{max}}(N^\top N) \leq \text{trace}\rbr{N^\top N} = d.
\end{equation}
Now, consider the normalized version of the expanded data matrix, \( \tilde N = \sbr{D_1 N \; D_2 N \ldots} \).
Recalling each \( D_i \) is a diagonal matrix whose elements are either \( 0 \) or \( 1 \), we have

\[ N^\top D_i^\top D_i N = N^\top D_i N^\top \preceq N^\top N, \]
for each \( D_i \in \tilde \calD \) and the diagonal elements of this matrix are bounded by \( 1 \).
We conclude that
\[
	\lambda_{\text{max}}(\tilde N^\top \tilde N) \leq \text{trace}\rbr{\tilde N^\top \tilde N} \leq d \cdot |\tilde \calD|.
\]
This establishes the claim in Section~\ref{sec:efficient-fista} that data normalization can be used to upper-bound the maximum eigenvalues of the Hessian.

Moving on to computation of the gradient, note that the Hessian \( \tilde N^\top \tilde N \) will be low-rank as long as \( |\tilde \calD| * d > n \).
In fact, this is nearly always the case since we typically choose \( \tilde \calD \) to be as large as possible.
Thus, although the condition number of \( \nabla^2 f(v) \) is not well-defined, it is possible to reduce the maximum expansion entailed by the Hessian via column normalization.
Observing \( \nabla f(v) = \nabla^2 f(v) v - \tilde N^\top y \), we may expect conditioning of the gradient computation to improve.
Finally, since the C-GReLU is a linear model, transforming the weights as \( v_i' = H^{-1} v_i \) after optimization can be used to project the model back into the original data space, ensuring that data normalization has no effects outside of optimization.

\newpage

\section{Additional Experiments and Experimental Details}\label{app:additional-experiments}

Now we provide additional experimental results which were omitted from the main paper due to space constraints.
We also give all necessary details to replicate our experiments.


\subsection{Synthetic Classification}\label{app:synthetic-classification}

\begin{figure}[t]
	\centering
    \ifdefined\smallPDF
	\includegraphics[width=0.98\linewidth]{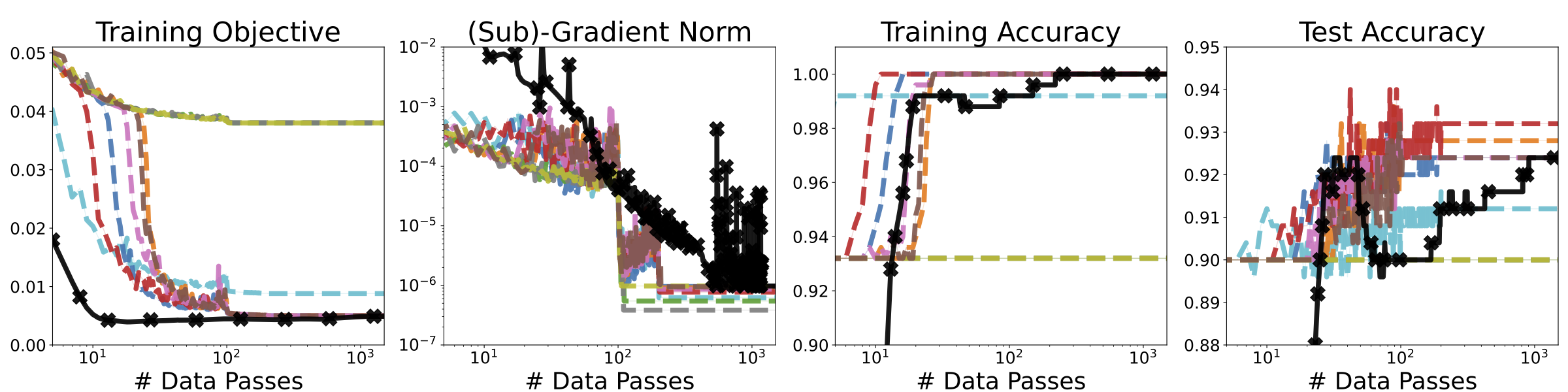}
    \else
	\includegraphics[width=0.98\linewidth]{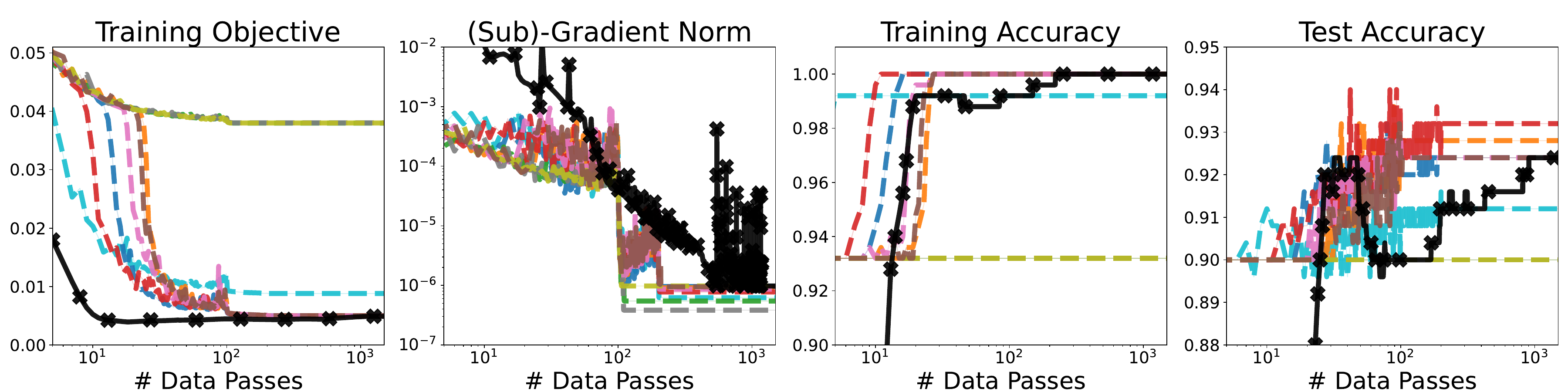}
    \fi
	\caption{Expanded version of \cref{fig:synthetic-classification} showing the square of the training gradient norm and test accuracy in addition to training objective and training accuracy, as reported in the main paper.
		Every run of SGD clearly converges to an approximate stationary point before terminating.
		The ``bumps'' in gradient norm for the AL method are caused by the dual updates, which increase the norm of the minim-norm subgradient of the augmented Lagrangian (\( \calL_\delta \)) with respect to the primal parameters.
	}%
	\label{fig:synthetic-classification-expanded}
\end{figure}

In this section, we provide additional details and results for the synthetic classification problem shown in \cref{fig:synthetic-classification}.

\textbf{Experimental Details}:
As mentioned in the main text, we generate the dataset by sampling \( X \sim \calN\rbr{0, \Sigma} \) and then taking \( y = \text{sign}(h_{W_1,w_2}(X)) \), where \( h_{W_1, w_2} \) is a two-layer ReLU network with \( m = 100 \) and random Gaussian weights.
We create \( 250 \) training examples and \( 250 \) test examples with \( d = 50 \) in this fashion.
The covariance matrix \( \Sigma \) is generated by sampling a random orthonormal matrix of eigenvectors and \( d - 2 \) eigenvalues from the interval \( [1, 10] \).
We then append \( 10 \) and \( 1 \) to this list and form \( \Sigma \) from the diagonalization;
this guarantees that the condition number of \( \Sigma \) is exactly \( 10 \).
Before optimization, we unitize the columns of the feature matrix (see Appendix~\ref{app:data-normalization}) to be consistent without our other experiments.

For the non-convex optimization problem, we use the standard PyTorch initialization and a step-size of \( 10 \).
This step-size gave the fastest convergence out of a grid of \( \{20, 10, 5, 1, 0.5, 0.1, 0.01\} \).
The mini-batch size is \( 25 \) examples (\( 10 \% \) of the dataset) and the maximum number of epochs is \( 1000 \).
We consider SGD to have converged when the gradient norm (as computed by PyTorch) is less than \( 10^{-3} \).
We change the global seed at each of the ten different runs, which ensures that both the initialization and mini-batch order/composition are different.
For our AL method, we randomly sample \( 100 \) ``diversity'' arrangements, which we augmented will activation patterns generated by SGD while optimizing the non-convex model.
Unlike SGD, the randomness in 10 runs of the AL method is due only to (i) sampling of the diversity set and (ii) the sign patterns from the SGD run.
Note that we are careful to use exactly the same runs of SGD as described above when using the active set method to compute \( \tilde \calD \).
We use the standard parameters as given in Appendix~\ref{app:default-parameters} for the remainder of the AL method's settings.

\textbf{Additional Results}:
\cref{fig:synthetic-classification-expanded} shows the convergence behavior of our AL and SGD.\@
As in the main paper, we omit all runs of the AL method but one because they are nearly identical.
One run of SGD diverges, while nine runs converge to stationary points as measured by the convergence criterion.
Of these, four converge to local minima with sub-optimal objective values; these models also do not have \( 100 \% \) accuracy on the training set despite the problem being realizable.
These sub-optimal local minimal also give worse test accuracy than the model found by the AL method.


\subsection{Large-Scale Comparison}\label{app:performance-profiles}

This section gives concrete experimental details for the large-scale comparison of optimization performance presented in \cref{fig:performance-profiles}.
We also present the \emph{same} experimental results with different thresholds for success.

\subsubsection{Experimental Details}

We first provide details required to reproduce~\cref{fig:performance-profiles}.

\textbf{Data}: We generate the performance profile using 73 datasets taken from the UCI machine learning repository and six individual regularization parameters for each dataset.
See Appendix~\ref{app:uci-datasets} for details.
This created a set of 438 optimization problems on which we tested the optimization algorithms.
Post-optimization, we omit all problems for which a degenerate solution (ie. all weights are zero) is optimal.

\textbf{Models}: We generated \( \tilde \calD \) for the C-ReLU and C-GReLU problems by sampling \( 5000 \) and \( 2500 \) generating vectors from \( z_i \sim \calN(0, \bfI) \), respectively, and then computing \( D_i = \text{diag}(X z_i > 0) \).
The zero matrix was removed and duplicate patterns were filtered out.
The convex formulations were extended to multi-class classification problems as described in Appendix~\ref{app:multi-class-extension}.
To ensure the convex and non-convex problems have the same optimal values, we use the vector-out formulation of the NC-ReLU problem given in \cref{convex-forms:eq:vector-non-convex-relu-mlp}.

We use the exact same activation patterns for NC-GReLU and C-ReLU.
We approximately match the NC-GReLU and C-ReLU model spaces by choosing
\[ m = \sum_{D_i \in \tilde \calD} \abs{\cbr{v^*_i : v^*_i \neq 0} \cup \cbr{w^*_i : w^*_i \neq 0} }. \]
Recall from \cref{thm:duality-free} that this choice ensures the model space for C-ReLU is a strict subset of that for NC-ReLU.
Thus, our results can only favor the non-convex formulations.

\textbf{Optimizers}: For models with gated ReLU activations, we compare R-FISTA with default parameters (see Appendix~\ref{app:default-parameters}) to Adam, SGD, and MOSEK.\@

We use a mini-batch size of \( \frac{n}{10} \) for Adam and SGD and perform a grid-search over the following set of step-sizes:
\[
	\eta \in \cbr{10, 5, 1, 0.5, 0.1, 0.01, 0.001}.
\]
For each optimization problem, we choose the step-size which gives the smallest final training objective.
We also use a decay schedule that halves the step-size every 100 epochs; experimentally, this ``step'' schedule worked much better than classical schedules of the form \( \eta_k = \eta_0 * t^{-r} \), \( r > 0.5 \).
To control for stochasticity, Adam and SGD are run with three independent random seeds and only the best execution is reported.
MOSEK is run with the default configuration using CVXPY as an interface~\citep{diamond2016cvxpy, agrawal2018rewriting};
we only use MOSEK on the convex reformulation.

The same experimental procedure is used for Adam and SGD on models with ReLU activations,
We use our AL method with standard parameters (Appendix~\ref{app:default-parameters}) to solve the convex reformulation and MOSEK is again used with standard parameters to solve the convex reformulation.

\textbf{Hardware and Timing}:
R-FISTA, our AL method, Adam, and SGD are run on GPU compute nodes with one GeForce RTX 2080Ti graphics card, two AMD 7502P CPUs, and 8 GB of RAM.
Note that the GPUs themselves have 11 GB of GPU RAM.\@
MOSEK cannot be run on GPUs, so instead these experiments are executed on CPU nodes with 32 GB of RAM and 4 AMD EPYC 7502 CPUs, each of which have 32 cores.
In practice, we observed extremely small variance when timing identical runs.
As such, we do not average times over multiple runs.

\textbf{Determining Successes}:
We use the (sub)-optimality gap \( F(x_k) - F(x^*) \) to determine if optimization is successful.
In particular, the relative optimality gap can be checked as
\[
	\Delta_k := \frac{F(x_k) - F(x^*)}{F(x^*)} \leq r_{\text{gap}},
\]
for some threshold \( r_{\text{gap}} \).
\cref{fig:performance-profiles} reports results for \( r_{\text{gap}} = 1 \).
For fairness, we provide figures generated from the same experimental data with different choices of \( r_{\text{gap}}\) in the next sub-section.
Finally, runs which exceed their available memory and crash are considered failures, as are problems which take more than 15 minutes.
In practice, this is only applicable for MOSEK, which scales poorly in both memory and time.

\subsubsection{Additional Results}

\begin{figure}
	\centering
	\includegraphics[width=0.96\linewidth]{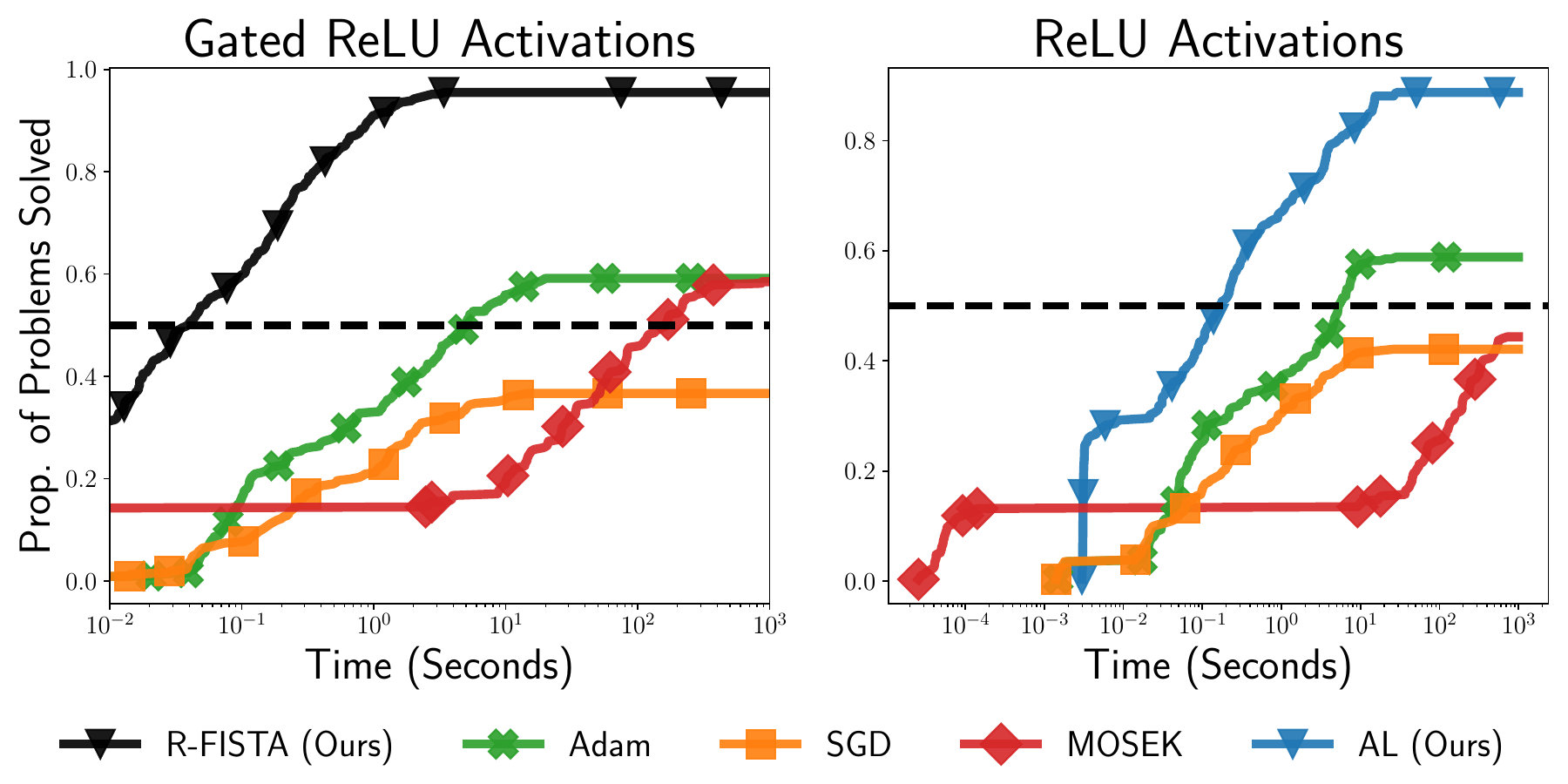}
	\caption{
		Alternative version of \cref{fig:performance-profiles} with success threshold set to be \( r_{\text{gap}} = 0.5 \).
		Recall that a problem is considered solved when \( \rbr{F(x_k) - F(x^*)}/F(x^*) \leq r_{\text{gap}} \), where \( F(x^*) \) is the smallest objective value found by any method.
		We find that tightening the success threshold (as compared to \cref{fig:performance-profiles}) only improves the performance of our convex solvers relative to Adam and SGD.
	}%
	\label{fig:pp-main-05}
\end{figure}

\begin{figure}
	\centering
	\includegraphics[width=0.96\linewidth]{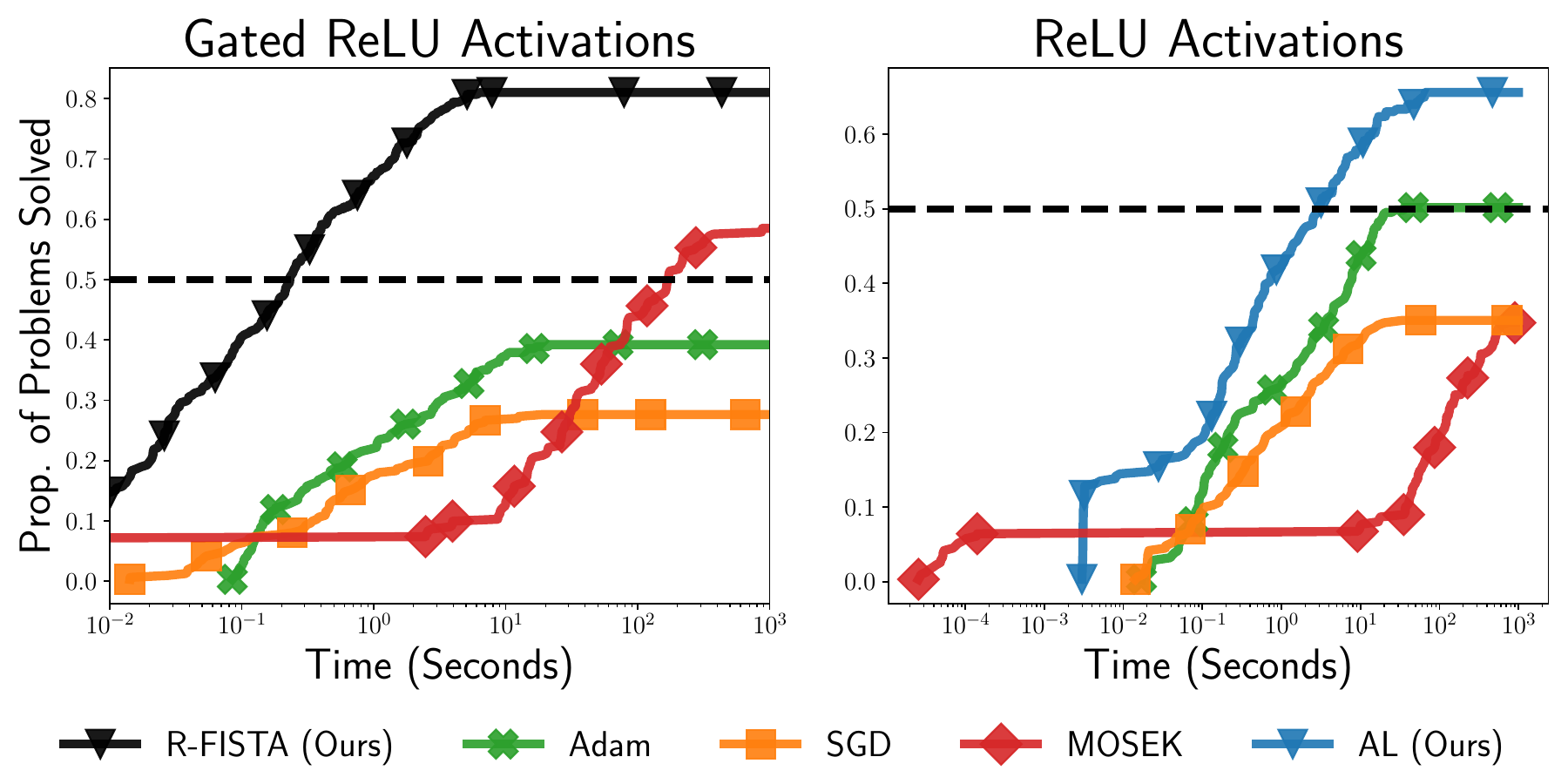}
	\caption{
		Alternative version of \cref{fig:performance-profiles} with success threshold set to be \( r_{\text{gap}} = 0.1 \).
		Performance of all optimization methods decreases at this threshold, with the notable exception of MOSEK for the C-GReLU problem.
		However, the relative ordering of R-FISTA, AL, Adam, and SGD remains unchanged.
		In other words, our optimizers applied to the convex reformulations still out-perform the non-convex baselines.
	}%
	\label{fig:pp-main-01}
\end{figure}

\textbf{Alternative Success Thresholds}:
Choosing the threshold for the relative optimality gap \( \Delta_k \) is subjective and can potentially favor some methods over others.
In this section, we show that alternative values of \( r_{\text{gap}} \) preserve the ordering of methods from \cref{fig:performance-profiles}.
In particular, tightening the threshold shows that our convex solvers are not only faster than the non-convex baselines, but also solve the optimization problems to greater accuracy.

Figures~\ref{fig:pp-main-05} and~\ref{fig:pp-main-01} present the same experimental results as in Figure~\ref{fig:performance-profiles} with \( r_{\text{gap}} = 0.5 \) and \( r_{\text{gap}} = 0.1 \), respectively.
They should be compared against the threshold value of \( r_{\text{gap}} = 1 \) used in the main paper.
These figures show that the relative performance of each optimization method remains unchanged as the threshold is decreased, with the notable exception of MOSEK.
This is because MOSEK uses a highly accurate, but slow, interior point method.


\subsection{Cone Decompositions}\label{app:cone-decomps}

\setlength{\tabcolsep}{4pt} 

\begin{table*}[t]
	\centering
	\caption{Approximating the C-ReLU problem with cone decompositions.
		We compare the solution to C-GReLU (FISTA) with cone-decomposition by solving the min-norm program (CD-SOCP),
		the approximate cone decomposition (CD-A), and directly solving C-ReLU using the AL method.
		We report the norm of each model to demonstrate the ``blow-up'' effect of the cone decomposition.
	}
	\label{table:cone-decomp-full}
	\vspace{0.1in}

	\begin{small}
		\begin{tabular}{lcccccccccccc}
			               & \multicolumn{3}{c}{\textbf{R-FISTA}} & \multicolumn{3}{c}{\textbf{CD-SOCP}} & \multicolumn{3}{c}{\textbf{CD-A}} & \multicolumn{3}{c}{\textbf{AL}}                                                                                                                                        \\ \cmidrule(lr){2-4}  \cmidrule(lr){5-7}  \cmidrule(lr){8-10} \cmidrule(lr){11-13}
			Dataset        & Acc.                                 & Time                                 & Norm                              & Acc.                            & Time   & Norm                          & Acc. & Time & Norm                           & Acc. & Time  & Norm                          \\ \midrule
			breast-cancer  & 68.4                                 & 0.84                                 & 1.06{\tiny \(\times 10^{2}\)}     & 68.4                            & 12.14  & 3.73{\tiny \(\times 10^{3}\)} & 68.4 & 3.18 & 4.76{\tiny \(\times 10^{3}\)}  & 66.7 & 6.17  & 5.84{\tiny \(\times 10^{1}\)} \\
			congressional  & 64.4                                 & 0.89                                 & 4.20{\tiny \(\times 10^{1}\)}     & 64.4                            & 80.69  & 1.90{\tiny \(\times 10^{3}\)} & 97.3 & 5.26 & 3.5{\tiny \(\times 10^{3}\)}   & 69.0 & 18.32 & 3.48{\tiny \(\times 10^{1}\)} \\
			sonar          & 87.8                                 & 0.45                                 & 2.25{\tiny \(\times 10^{1}\)}     & 90.2                            & 23.13  & 1.28{\tiny \(\times 10^{2}\)} & 64.4 & 1.09 & 1.83{\tiny \(\times 10^{2}\)}  & 87.8 & 4.34  & 2.33{\tiny \(\times 10^{1}\)} \\
			credit         & 82.6                                 & 0.44                                 & 7.32{\tiny \(\times 10^{1}\)}     & 83.3                            & 63.46  & 3.97{\tiny \(\times 10^{3}\)} & 84.1 & 5.1  & 4.76{\tiny \(\times 10^{3 }\)} & 84.1 & 6.46  & 5.11{\tiny \(\times 10^{1}\)} \\
			cylinder       & 75.5                                 & 1.18                                 & 9.64{\tiny \(\times 10^{1}\)}     & 77.5                            & 79.78  & 2.09{\tiny \(\times 10^{3}\)} & 75.5 & 2.68 & 2.61{\tiny \(\times 10^{3}\)}  & 75.5 & 15.18 & 1.11{\tiny \(\times 10^{2}\)} \\
			ecoli          & 71.6                                 & 0.07                                 & 1.73{\tiny \(\times 10^{1}\)}     & 71.6                            & 149.7  & 5.82{\tiny \(\times 10^{2}\)} & 70.1 & 0.36 & 1.13{\tiny \(\times 10^{2}\)}  & 70.1 & 3.38  & 1.68{\tiny \(\times 10^{1}\)} \\
			energy-y1      & 86.3                                 & 0.12                                 & 2.90{\tiny \(\times 10^{1}\)}     & 86.3                            & 134.55 & 2.34{\tiny \(\times 10^{3}\)} & 86.3 & 2.01 & 1.34{\tiny \(\times 10^{3}\)}  & 83.7 & 5.05  & 2.89{\tiny \(\times 10^{1}\)} \\
			glass          & 64.3                                 & 0.13                                 & 2.00{\tiny \(\times 10^{1}\)}     & 64.3                            & 68.76  & 4.05{\tiny \(\times 10^{2}\)} & 64.3 & 0.69 & 2.31{\tiny \(\times 10^{2}\)}  & 61.9 & 3.0   & 1.72{\tiny \(\times 10^{1}\)} \\
			cleveland      & 51.7                                 & 0.13                                 & 1.97{\tiny \(\times 10^{1}\)}     & 51.7                            & 109.4  & 2.52{\tiny \(\times 10^{2}\)} & 51.7 & 0.6  & 2.18{\tiny \(\times 10^{2 }\)} & 50.0 & 1.82  & 1.77{\tiny \(\times 10^{1}\)} \\
			hungarian      & 86.2                                 & 0.88                                 & 8.01{\tiny \(\times 10^{1}\)}     & 86.2                            & 17.15  & 2.75{\tiny \(\times 10^{3}\)} & 86.2 & 4.09 & 3.59{\tiny \(\times 10^{3}\)}  & 84.5 & 7.01  & 5.09{\tiny \(\times 10^{1}\)} \\
			heart-va       & 35.0                                 & 0.16                                 & 2.44{\tiny \(\times 10^{1}\)}     & 37.5                            & 67.11  & 4.79{\tiny \(\times 10^{2}\)} & 37.5 & 0.9  & 4.73{\tiny \(\times 10^{2 }\)} & 37.5 & 1.86  & 1.58{\tiny \(\times 10^{1}\)} \\
			hepatitis      & 80.6                                 & 1.06                                 & 2.96{\tiny \(\times 10^{1}\)}     & 80.6                            & 10.57  & 2.63{\tiny \(\times 10^{2}\)} & 80.6 & 1.97 & 3.4{\tiny \(\times 10^{2}\)}   & 74.2 & 8.56  & 5.53{\tiny \(\times 10^{1}\)} \\
			horse-colic    & 85.0                                 & 1.0                                  & 5.79{\tiny \(\times 10^{1}\)}     & 86.7                            & 29.42  & 9.45{\tiny \(\times 10^{2}\)} & 86.7 & 4.54 & 1.36{\tiny \(\times 10^{3}\)}  & 90.0 & 9.75  & 8.92{\tiny \(\times 10^{1}\)} \\
			ionosphere     & 90.0                                 & 1.15                                 & 4.74{\tiny \(\times 10^{1}\)}     & 90.0                            & 44.12  & 1.51{\tiny \(\times 10^{3}\)} & 90.0 & 5.76 & 2.05{\tiny \(\times 10^{3}\)}  & 90.0 & 15.91 & 6.42{\tiny \(\times 10^{1}\)} \\
			mammograph     & 78.1                                 & 0.24                                 & 4.13{\tiny \(\times 10^{1}\)}     & 78.1                            & 33.17  & 8.11{\tiny \(\times 10^{3}\)} & 78.1 & 5.18 & 3.02{\tiny \(\times 10^{3}\)}  & 79.2 & 5.14  & 3.70{\tiny \(\times 10^{1}\)} \\
			monks-2        & 60.6                                 & 1.95                                 & 1.02{\tiny \(\times 10^{2}\)}     & 60.6                            & 5.81   & 7.36{\tiny \(\times 10^{3}\)} & 57.6 & 6.61 & 9.43{\tiny \(\times 10^{3}\)}  & 45.5 & 18.31 & 8.13{\tiny \(\times 10^{1}\)} \\
			monks-3        & 87.5                                 & 1.6                                  & 5.59{\tiny \(\times 10^{1}\)}     & 87.5                            & 3.89   & 3.44{\tiny \(\times 10^{3}\)} & 87.5 & 6.49 & 4.49{\tiny \(\times 10^{3}\)}  & 95.8 & 31.84 & 8.60{\tiny \(\times 10^{1}\)} \\
			oocytes        & 78.6                                 & 0.98                                 & 1.14{\tiny \(\times 10^{2}\)}     & 79.1                            & 136.3  & 1.19{\tiny \(\times 10^{4}\)} & 78.0 & 5.67 & 1.3{\tiny \(\times 10^{4}\)}   & 74.2 & 81.68 & 8.54{\tiny \(\times 10^{1}\)} \\
			parkinsons     & 92.3                                 & 1.9                                  & 4.70{\tiny \(\times 10^{1}\)}     & 92.3                            & 16.03  & 1.06{\tiny \(\times 10^{3}\)} & 92.3 & 4.84 & 1.45{\tiny \(\times 10^{3}\)}  & 89.7 & 19.04 & 6.69{\tiny \(\times 10^{1}\)} \\
			pima           & 73.2                                 & 0.36                                 & 7.88{\tiny \(\times 10^{1}\)}     & 73.2                            & 37.68  & 8.09{\tiny \(\times 10^{3}\)} & 73.2 & 5.07 & 7.45{\tiny \(\times 10^{3}\)}  & 75.8 & 4.72  & 4.03{\tiny \(\times 10^{1}\)} \\
			planning       & 63.9                                 & 1.33                                 & 8.94{\tiny \(\times 10^{1}\)}     & 63.9                            & 10.54  & 2.16{\tiny \(\times 10^{3}\)} & 63.9 & 3.07 & 3.05{\tiny \(\times 10^{3}\)}  & 58.3 & 9.74  & 1.09{\tiny \(\times 10^{2}\)} \\
			seeds          & 95.2                                 & 0.25                                 & 1.91{\tiny \(\times 10^{1}\)}     & 95.2                            & 26.71  & 4.83{\tiny \(\times 10^{2}\)} & 95.2 & 1.99 & 4.17{\tiny \(\times 10^{2}\)}  & 95.2 & 5.16  & 1.82{\tiny \(\times 10^{1}\)} \\
			australian     & 64.5                                 & 0.69                                 & 1.58{\tiny \(\times 10^{2}\)}     & 63.8                            & 55.59  & 1.06{\tiny \(\times 10^{4}\)} & 64.5 & 5.46 & 1.21{\tiny \(\times 10^{4}\)}  & 65.2 & 4.8   & 2.74{\tiny \(\times 10^{1}\)} \\
			statlog-heart  & 81.5                                 & 0.81                                 & 6.76{\tiny \(\times 10^{1}\)}     & 81.5                            & 15.57  & 1.48{\tiny \(\times 10^{3}\)} & 81.5 & 2.5  & 2.02{\tiny \(\times 10^{3 }\)} & 85.2 & 7.36  & 6.11{\tiny \(\times 10^{1}\)} \\
			teaching       & 40.0                                 & 0.24                                 & 3.71{\tiny \(\times 10^{1}\)}     & 40.0                            & 12.18  & 1.10{\tiny \(\times 10^{3}\)} & 40.0 & 1.75 & 1.14{\tiny \(\times 10^{3}\)}  & 33.3 & 2.05  & 2.21{\tiny \(\times 10^{1}\)} \\
			tic-tac-toe    & 97.9                                 & 0.45                                 & 1.70{\tiny \(\times 10^{2}\)}     & 97.9                            & 60.36  & 6.86{\tiny \(\times 10^{3}\)} & 97.9 & 4.38 & 7.6{\tiny \(\times 10^{3}\)}   & 93.7 & 14.25 & 1.57{\tiny \(\times 10^{2}\)} \\
			vertebral-col. & 88.7                                 & 0.68                                 & 7.41{\tiny \(\times 10^{1}\)}     & 88.7                            & 9.64   & 7.76{\tiny \(\times 10^{3}\)} & 88.7 & 5.54 & 6.73{\tiny \(\times 10^{3}\)}  & 90.3 & 8.23  & 4.74{\tiny \(\times 10^{1}\)} \\
			wine           & 100                                  & 0.25                                 & 1.87{\tiny \(\times 10^{1}\)}     & 100                             & 32.05  & 1.97{\tiny \(\times 10^{2}\)} & 100  & 0.95 & 2.42{\tiny \(\times 10^{2}\)}  & 100  & 3.67  & 1.78{\tiny \(\times 10^{1}\)} \\ \bottomrule
		\end{tabular}
	\end{small}
\end{table*}

\setlength{\tabcolsep}{6pt}

Now we provide details and additional results for the cone-decomposition experiments given in \cref{table:cone-decomp}.

\textbf{Experimental Details}:
We selected 23 datasets from the UCI repository and fixed the regularization parameter at \( \lambda = 0.01 \).
Note that this parameter is not necessary optimal for each dataset;
the purpose of these experiments is to study the effects of using cone-decompositions to approximate the C-ReLU solution with a G-ReLU solution, rather than to obtain optimal test accuracies.
We randomly sampled 1000 activation patterns for the C-ReLU and C-GReLU models, removing duplicates and the zero pattern as necessary.
Note that we report the median results from five individual runs with re-sampled activation patterns to control for variance in the procedure.

For multi-class datasets, the convex formulations were extended as described in Appendix~\ref{app:multi-class-extension}.
We used the standard parameters for R-FISTA and the AL method as given in Appendix~\ref{app:default-parameters},
while the min-norm decomposition programs (CD-SOCP) was solved with MOSEK using the default parameters.
For CD-A, we set \( \lambda = 10^{-10} \) and \cref{eq:cd-approx} with R-FISTA
using the default parameters.
We terminate the optimization procedure when the min-norm subgradient has squared-norm
less than or equal to \( 10^{-10} \).
R-FISTA and the AL method were run on GPU compute nodes with one GeForce RTX 2080Ti graphics card, and four AMD 7502P CPUs, and 32 GB of RAM.
The cone decompositions were solved on identical nodes with four AMD 7502P CPUs with 32 GB of RAM.

\textbf{Additional Results}:
\cref{table:cone-decomp-full} provides the full set of results on all 23 datasets.
It also includes the final group norms of the models, calculated as
\[
	\sum_{D_i \in \tilde \calD} \norm{u_i^*}_2,
\]
for the C-GReLU model and
\[
	\sum_{D_i \in \tilde \calD} \norm{v_i^*}_2 + \norm{w_i^*}_2,
\]
for the C-ReLU models.
This allows us to quantify the ``blow-up'' in the model norm from decomposing \( u_i^* \) onto \( \calK_i - \calK_i \).
In practice, we find that CD-SOCP leads to very large increases in the model norm compared to the FISTA/AL solutions, while CD-A has a less severe effect.
However, the increased norms do not appear to affect the test accuracy of the final models.
Indeed, CD-SOCP and CD-A perform as well as the solution to the C-GReLU problem given by R-FISTA and are comparable to the AL method's solution.
The major downside of exact the cone-decomposition method is the huge increase in time necessary to solve for the decomposition.
This is largely because MOSEK is restricted to running on CPU.


\subsection{The Role of Acceleration and other Algorithmic Components}\label{app:acceleration-exps}

\begin{figure}
	\centering
	\includegraphics[width=0.95\linewidth]{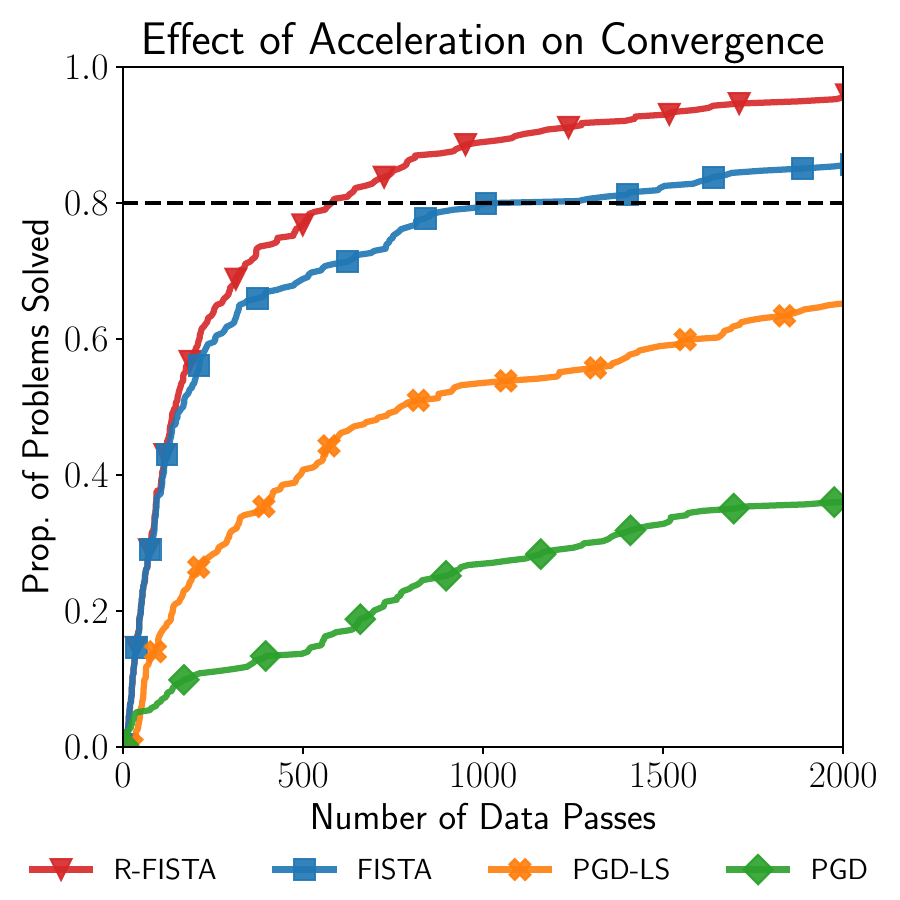}
	\caption{Performance profile comparing R-FISTA, FISTA without restarts, proximal gradient descent with line-search (PGD-LS) and proximal gradient descent with a fixed step-size (PGD) for solving C-GReLU on 73 datasets from the UCI repository.
		For PGD, we report results for the best step-size chosen by grid-search individually for each problem.
		R-FISTA solves a higher proportion of problems in fewer passes through the dataset.
	}%
	\label{fig:acceleration-pp}
\end{figure}

This section studies the effects of different algorithmic components on the optimization performance of R-FISTA for the C-GReLU problem.
By systematically removing restarts, acceleration, and line-search, we illustrate the importance of these enhancements to the speed and robustness of the optimization procedure.

\cref{fig:acceleration-pp} shows a performance profile comparing R-FISTA, the FISTA algorithm without restarts (FISTA), proximal gradient descent with the line-search described in Section~\ref{sec:efficient-fista} (PGD-LS), and proximal gradient descent (PGD) with a fixed step-size.
We use the same problem set as for \cref{fig:performance-profiles}: \( 438 \) individual training problems generated by considering six regularization parameters for \( 73 \) datasets taken from the UCI dataset repository.
See \cref{app:uci-datasets} for more details.
Note that we do not include problems for which the regularization parameter is overly large and a degenerate model (ie.\ all zeros) is optimal.
A problem is considered solved the minimum norm subgradient has norm less than or equal to \( 10^{-3} \); in practice, we check an identical condition on the gradient norm squared.
The C-GReLU model is formed by sampling \( 5000 \) activation patterns.

The x-axis shows the number of passes through that dataset that each method performs.
This quantity is equivalent to the iteration counter for PGD; for the remaining methods it also includes the number of function evaluations due to back-tracking on the line-search condition.
For R-FISTA, FISTA, and PGD-LS, we use the step-size initialization strategy described in the main paper (see Appendix~\ref{app:step-size-initialization} for experiments studying this rule) with the standard parameters given in Appendix~\ref{app:default-parameters}.
For each problem, we use the best fixed step-size for PGD out of the grid \( \cbr{10, 1, 0.1, 0.01} \).

We make the following observations:
(i) R-FISTA requires about three-fourths as many data passes as FISTA to solve \( 80\% \) of problems, which suggests restarts allow greater adaptivity to problem structure;
(ii) acceleration is critical to solving problems quickly and PGD-LS performs poorly compared to both R-FISTA and FISTA;
(iii) PGD is very slow despite using about \( 4 \times \) more compute than the other methods.

\begin{figure}
	\centering
	\ifdefined\smallPDF
		\includegraphics[width=0.95\linewidth]{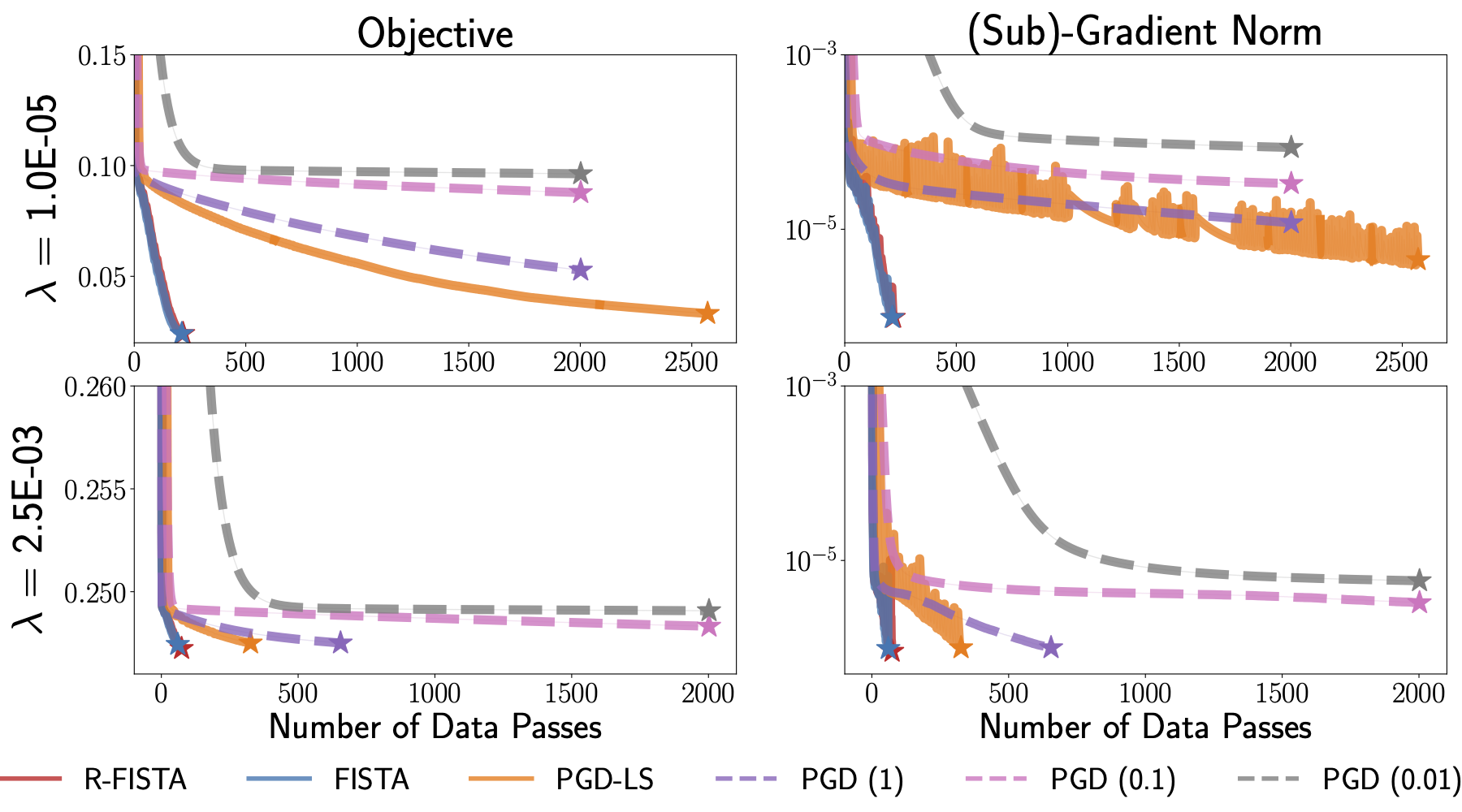}%
	\else
		\includegraphics[width=0.95\linewidth]{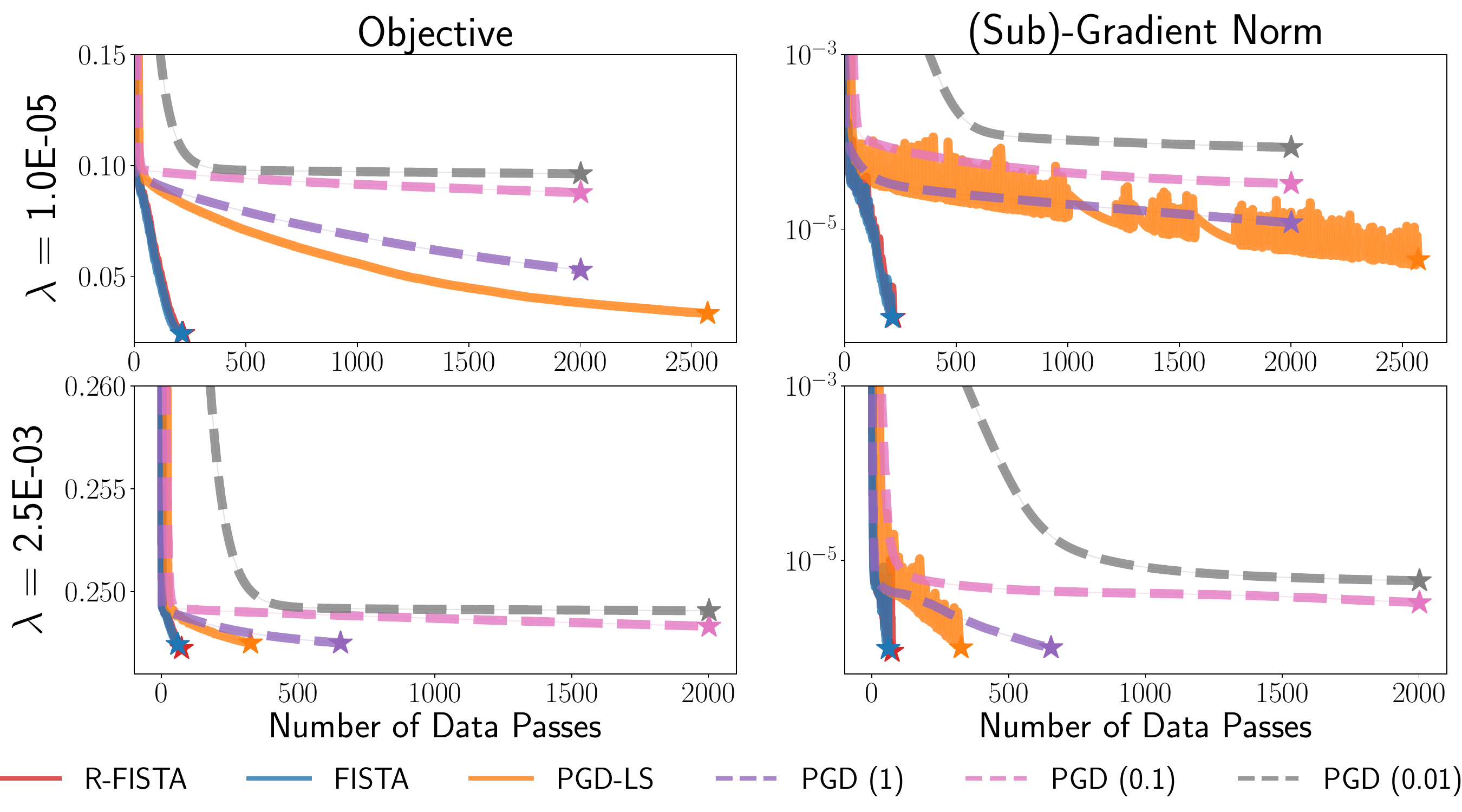}%
	\fi
	\caption{Convergence comparison for R-FISTA, FISTA without restarts (FISTA), proximal gradient descent with line-search (PGD-LS) and proximal gradient descent (PGD) with several fixed step-sizes (reported in parenthesis) on the \texttt{twonorm} dataset.
		PGD stalls while the accelerated methods converge very quickly to an approximate stationary point.
	}%
	\label{fig:acceleration-twonorm}
\end{figure}

\begin{figure}
	\centering
	\ifdefined\smallPDF
		\includegraphics[width=0.95\linewidth]{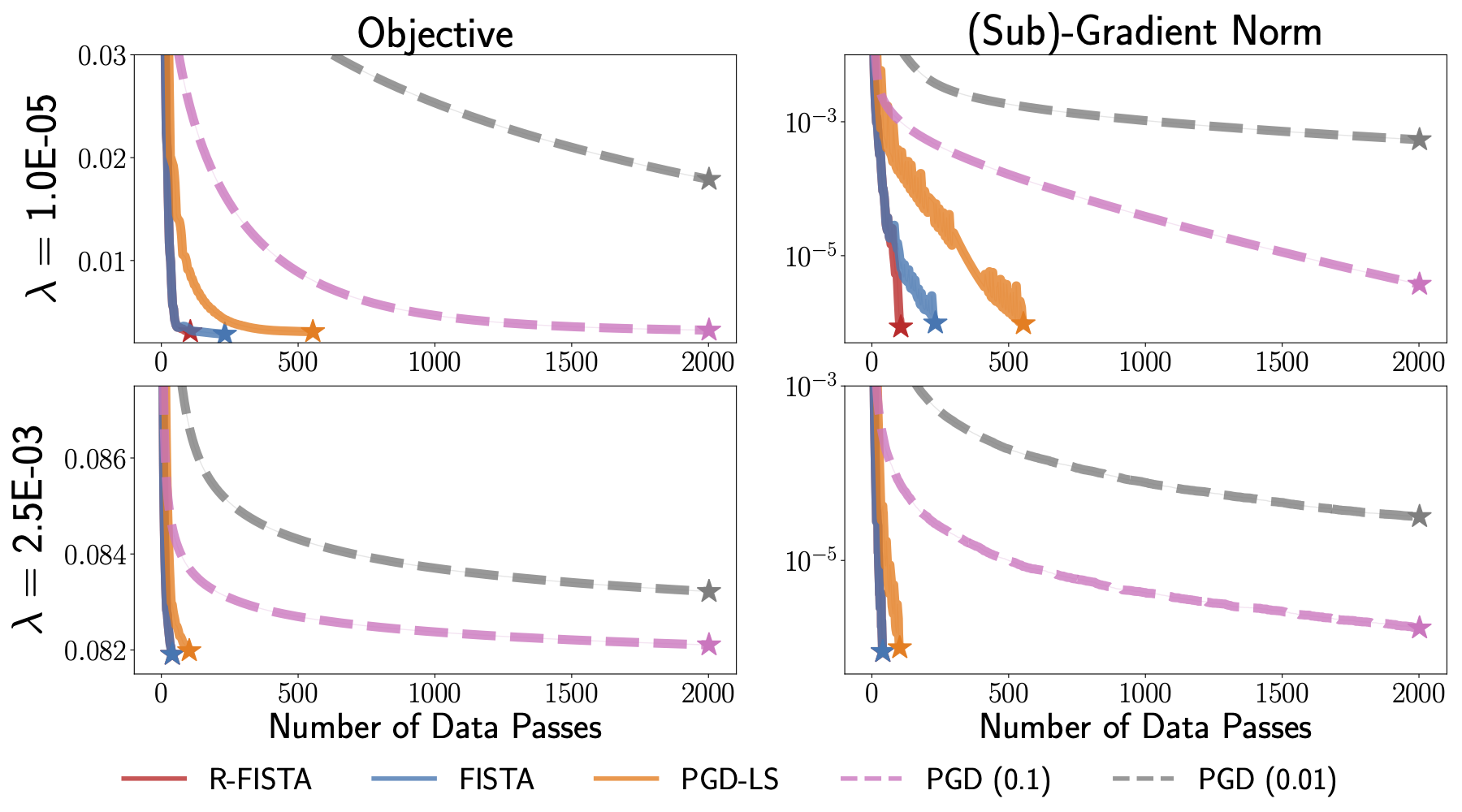}%
	\else
		\includegraphics[width=0.95\linewidth]{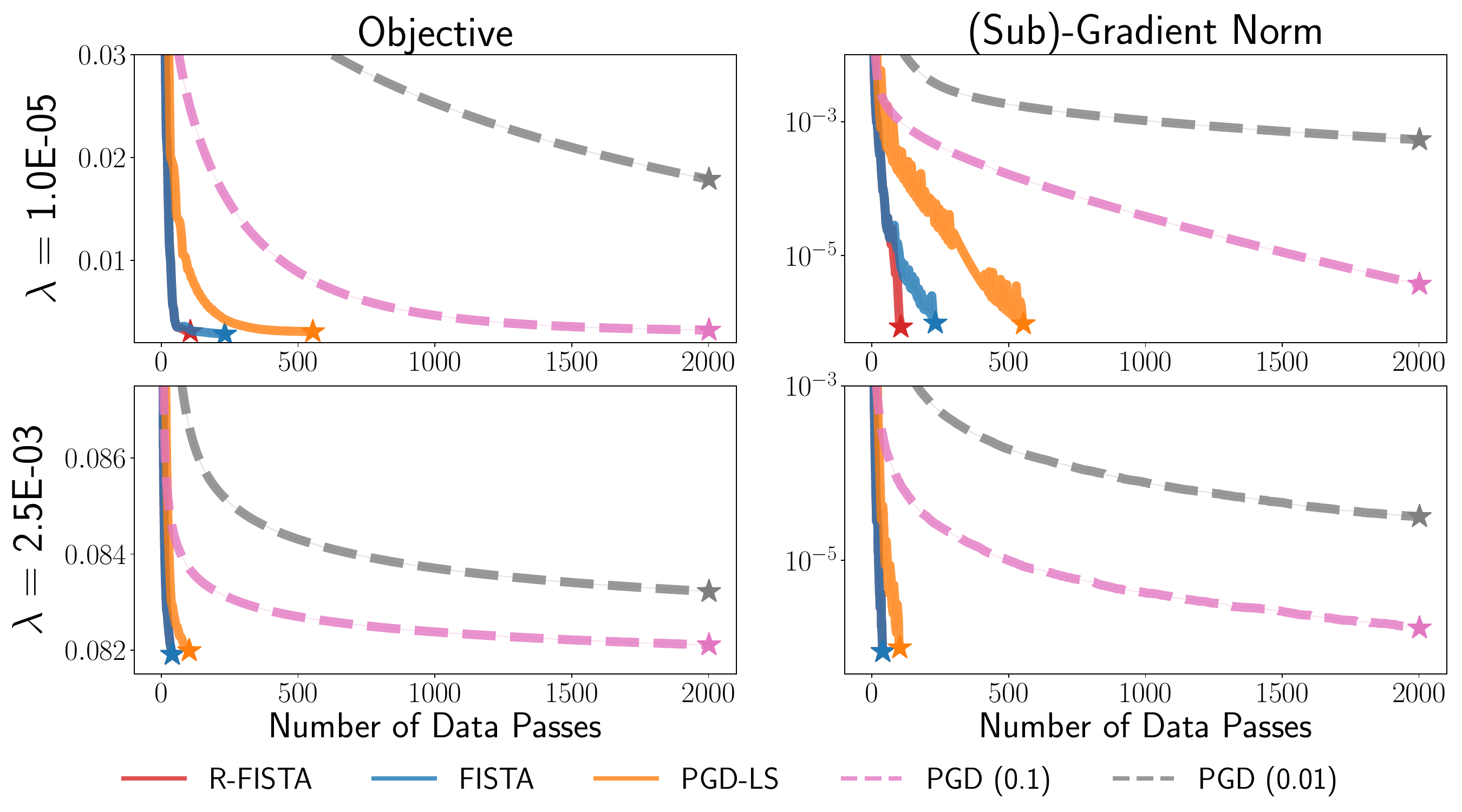}%
	\fi
	\caption{Convergence comparison for R-FISTA, FISTA without restarts (FISTA), proximal gradient descent with line-search (PGD-LS) and proximal gradient descent (PGD) with several fixed step-sizes (reported in parenthesis) on the \texttt{heart-cleveland} dataset.
		The performance of R-FISTA and FISTA is identical when \( \lambda = 2.5 \times 10^{-3} \).
		In contrast, restarting allows R-FISTA to converge in around half as many iterations as FISTA for the smoother problem with \( \lambda = 1 \times 10^{-5} \).
	}%
	\label{fig:acceleration-heart-cleveland}
\end{figure}

We also report convergence behavior on two randomly selected datasets to illustrate the fine-grained performance of each method.
Figures~\ref{fig:acceleration-twonorm} and~\ref{fig:acceleration-heart-cleveland} show the convergence of R-FISTA, FISTA, PGD-LS, and PGD with respect to objective value and subgradient norm (squared) for the \texttt{twonorm} and \texttt{heart-cleveland} datasets.
Results for are shown for the smallest regularization parameter considered and the largest for which the model was not degenerate.
We omit step-sizes for which PGD diverged.

\subsection{Step-size Update Rules}\label{app:step-size-initialization}

\begin{figure}[t]
	\centering
	\includegraphics[width=0.8\linewidth]{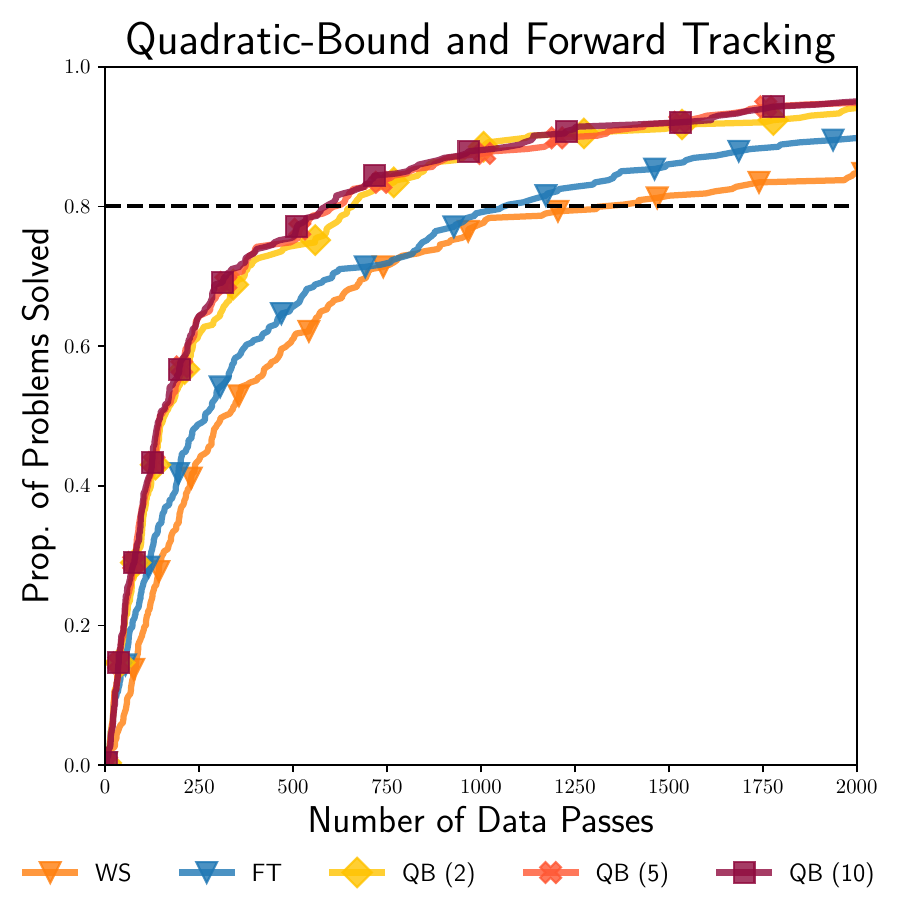}
	\caption{Performance profile comparing R-FISTA, FISTA with different step-size initialization rules.
		We compare checking the quadratic bound for tightness (QB) with a variety of choices for the threshold parameter \( c \) against warm starting as \( \etak = \eta_{k-1} \) (WS), and forward tracking (FT).
		As before, we generate the profile by solving the C-GReLU problem on 73 datasets from the UCI repository.
		QB is robust to the choice of \( c \) and outperforms both WS and FT.
		WS and FT have similar performance despite very different behavior.
	}%
	\label{fig:lassplore-pp}
\end{figure}

Now we perform an ablation study on the step-size initialization rule proposed by \citet{liu2009lassplore} and discussed in Section~\ref{sec:efficient-fista}.
Throughout this section, we refer to this initialization strategy as quadratic-bound (QB).
We compare QB against warm starting as \( \etak = \eta_{k-1} \) (WS), and forward tracking (FT).
As in the previous section, we use a performance profile to summarize results for solving the C-GReLU problem on the same \( 438 \) problems as in \cref{fig:performance-profiles}.
We use the same backtracking parameter \( \beta = 0.8 \) for QB, WS, and FT, while we use a forward-tracking parameter of \( \alpha = 1.25 \) for QB and FT.
Note that these are the standard parameters discussed in Appendix~\ref{app:default-parameters}.
We use the standard settings for all other parameters of R-FISTA.
We sample \( 5000 \) random activation patterns just as in the previous section.

Empirically, we find (see Figure~\ref{fig:lassplore-pp}) that the QB initialization strategy is surprisingly resilient to the choice of threshold parameter, \( c \).
Indeed, QB with any \( c \in \cbr{10, 5, 2} \) is more efficient than FT or WS.
Surprisingly, FT and WS have similar performance despite their substantially different convergence behavior (see Figures~\ref{fig:lassplore-glass} and~\ref{fig:lassplore-flags}).
This is primarily because we measure progress in total data passes, which includes the unnecessary backtracking performed by R-FISTA with the FT update.

\begin{figure}
	\centering
	\includegraphics[width=0.95\linewidth]{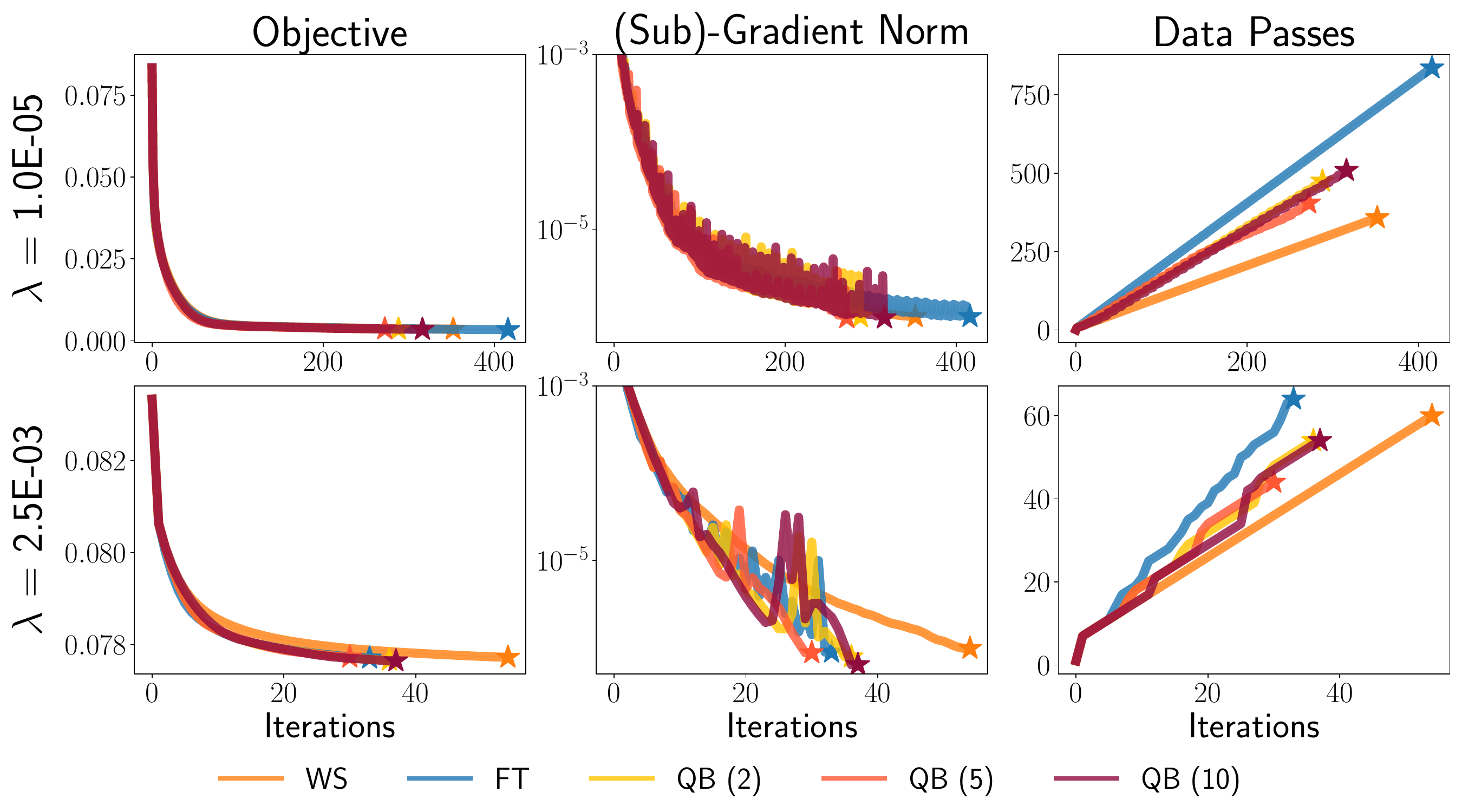}
	\caption{Convergence comparison for R-FISTA with different step-size initialization rules on the \texttt{glass} dataset.
		We compare warm-starting (WS) and forward-tracking (FT) against the initialization proposed by \citet{liu2009lassplore} (QB) for several fixed thresholds (reported in parentheses).
		QB has similar convergence performance to FT without requiring as many passes through the training set and is resilient to the choice of threshold.
	}

	\label{fig:lassplore-glass}
\end{figure}

\begin{figure}
	\centering
	\includegraphics[width=0.95\linewidth]{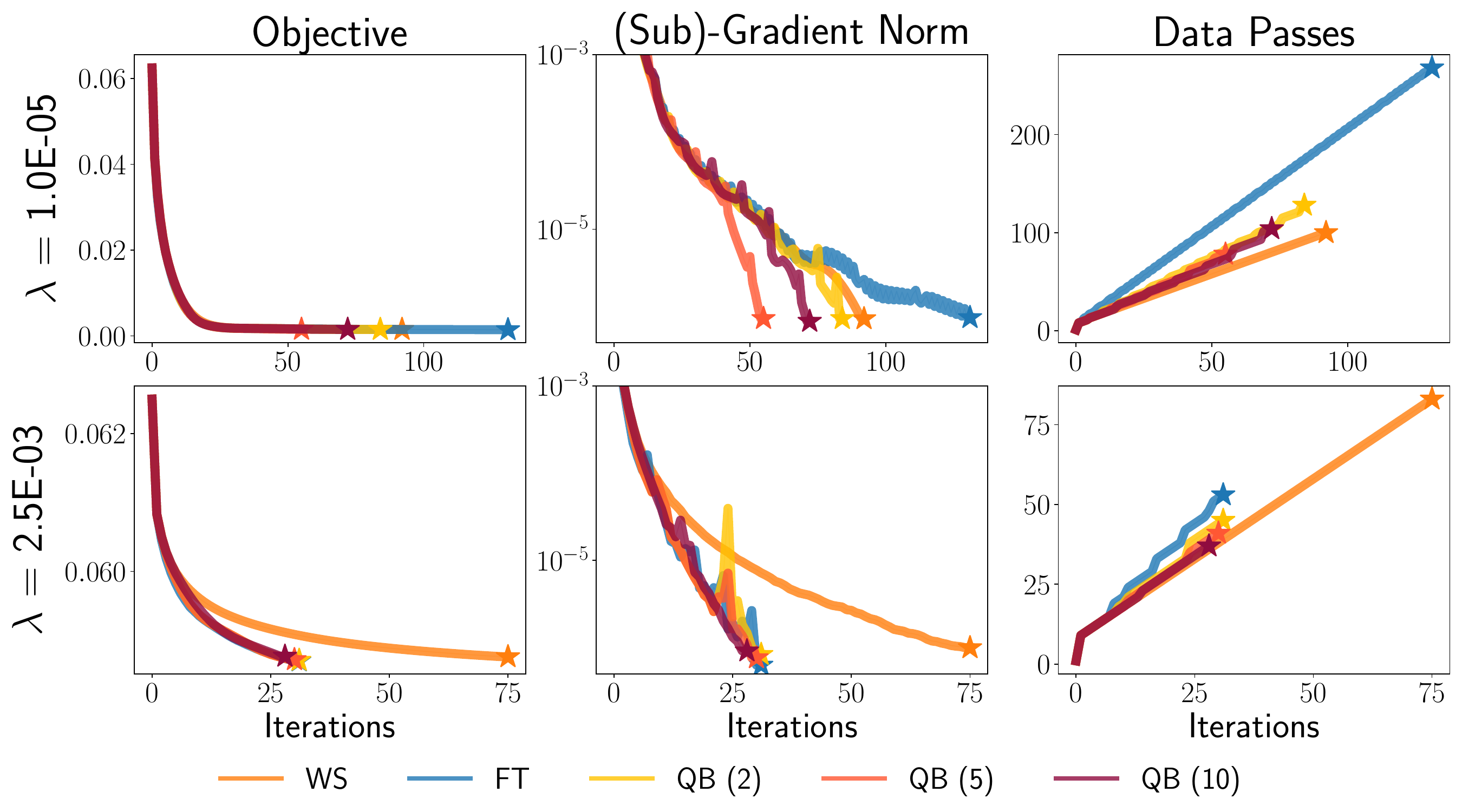}
	\caption{Convergence comparison for R-FISTA with different step-size initialization rules on the \texttt{flags} dataset.
		See \cref{fig:lassplore-glass} for additional details.}%
	\label{fig:lassplore-flags}
\end{figure}


\subsection{The Windowing Heuristic}\label{app:delta-heuristic-exps}

\begin{figure}
	\centering
	\includegraphics[width=0.95\linewidth]{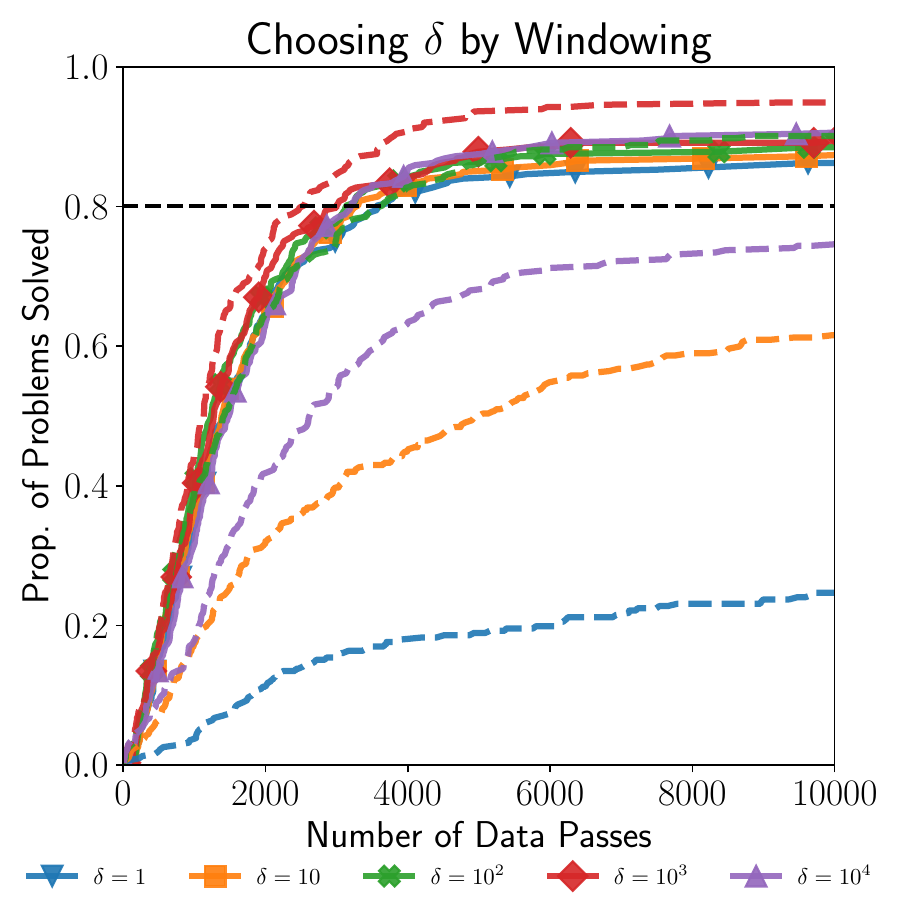}
	\caption{Performance profile comparing our AL method with (solid lines with markers) and without (dashed lines) the windowing heuristic for setting the penalty strength.
		We consider a wide range of initial \( \delta \) values and generate the profile by solving the C-ReLU problem on 73 datasets from the UCI repository with \( 6 \) different regularization parameters for each dataset.
		The windowing heuristic performs nearly as well as the best fixed \( \delta \) and without a noticeable computational overhead.
		In contrast, extreme values of \( \delta \) can cause the ``fixed'' approach to fail on approximately \( 40\% \) of problems.
	}%
	\label{fig:delta-pp}
\end{figure}

Recall that the key hyper-parameter for our AL method is the penalty strength, denoted \( \delta \).
Here we verify the effectiveness of the windowing heuristic for selecting \( \delta \) as proposed in Section~\ref{sec:reliable-co}.
Experimentally, the rule performs nearly as well as the best fixed value of \( \delta \) across a wide range of datasets and avoids the catastrophic failures which can occur when \( \delta \) is miss-specified.

We initialize our AL method with \( \delta_0 \in \cbr{1, 10, 10^2, 10^3, 10^4} \) and compare tuning \( \delta \) using the windowing heuristic against keeping \( \delta \) fixed throughout optimization.
All other parameters are identical and constant for the two approaches (see Appendix~\ref{app:default-parameters} for specifics).
To evaluate speed and robustness, we use another performance profile on the 438 problems generated from the UCI datasets as detailed in Appendix~\ref{app:uci-datasets}.
In this case, a problem is considered ``solved'' when the minimum-norm subgradient of the augmented Lagrangian is smaller than \( 10^{-3} \) and the norm of the constraint gaps is also less than \( 10^{-3} \).
This isn't equivalent to terminating when the Lagrangian function is approximately stationary, but we found the rule to work well in practice.
We use \( 500 \) randomly sampled activation patterns for the C-ReLU model.

\cref{fig:delta-pp} plots the result, with dotted lines for the AL method with fixed \( \delta \) and solid lines with markers for methods using the windowing heuristic.
Empirically, the windowing heuristic is nearly effective as the best fixed \( \delta \) and avoids the complete failure of AL methods with fixed, poorly specified penalty parameters (e.g. \( \delta = 1  \) or \( \delta = 10^4 \)).
Moreover, this is achieved at almost no overhead in terms of total data passes required for convergence.
Finally, we observe that fixing \( \delta = 10^3 \) works very well across all problems;
this is likely because the problems are carefully normalized before optimization to ensure they are on the same scale.
Specifically, the columns of the data matrix for each problem are unitized (Appendix~\ref{app:data-normalization}), and the augmented Lagrangian \( \calL_\delta \) is normalized by \( n * k \), where \( k \) is the number of classes.

\begin{figure}
	\centering
	\ifdefined\smallPDF
		\includegraphics[width=0.95\linewidth]{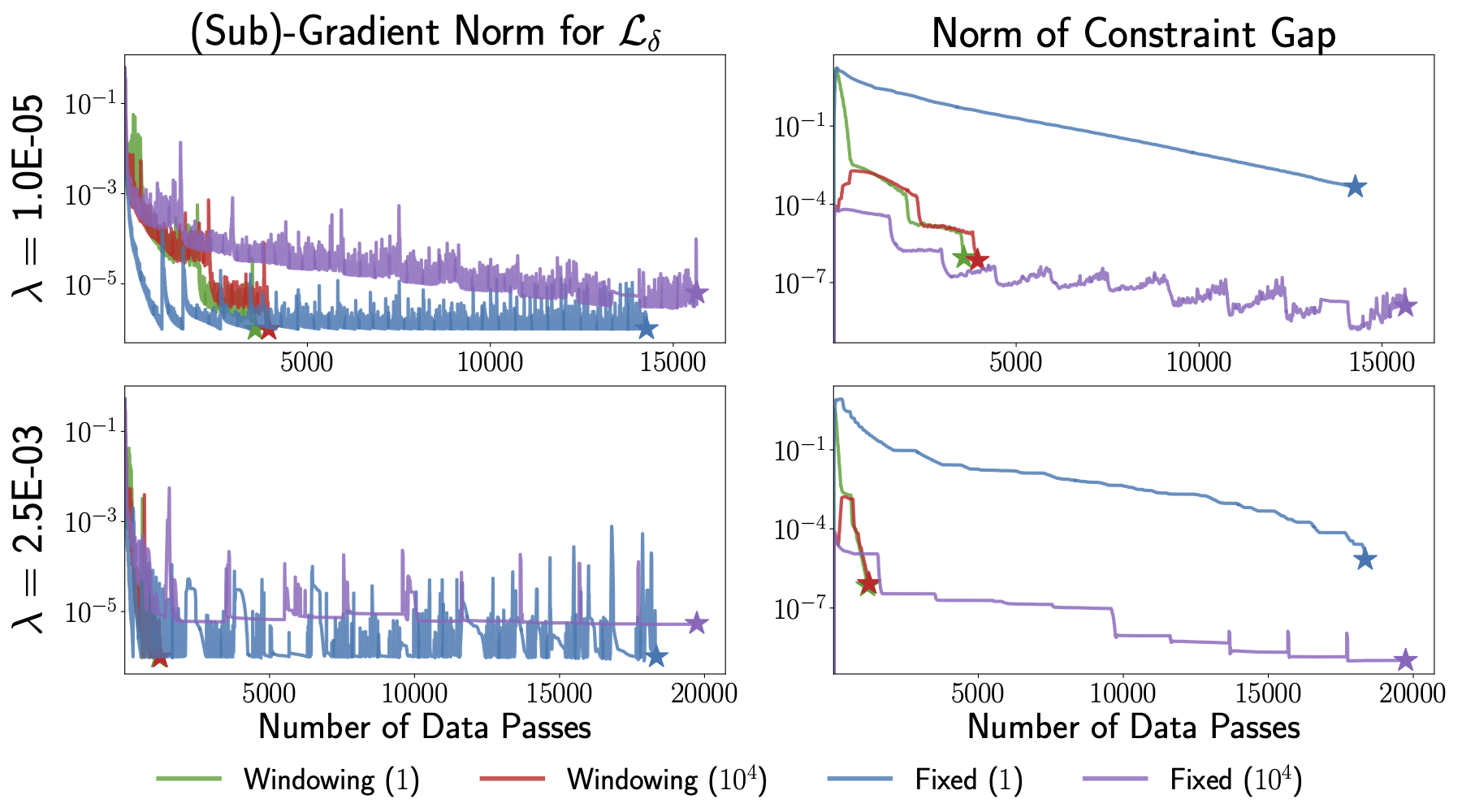}
	\else
		\includegraphics[width=0.95\linewidth]{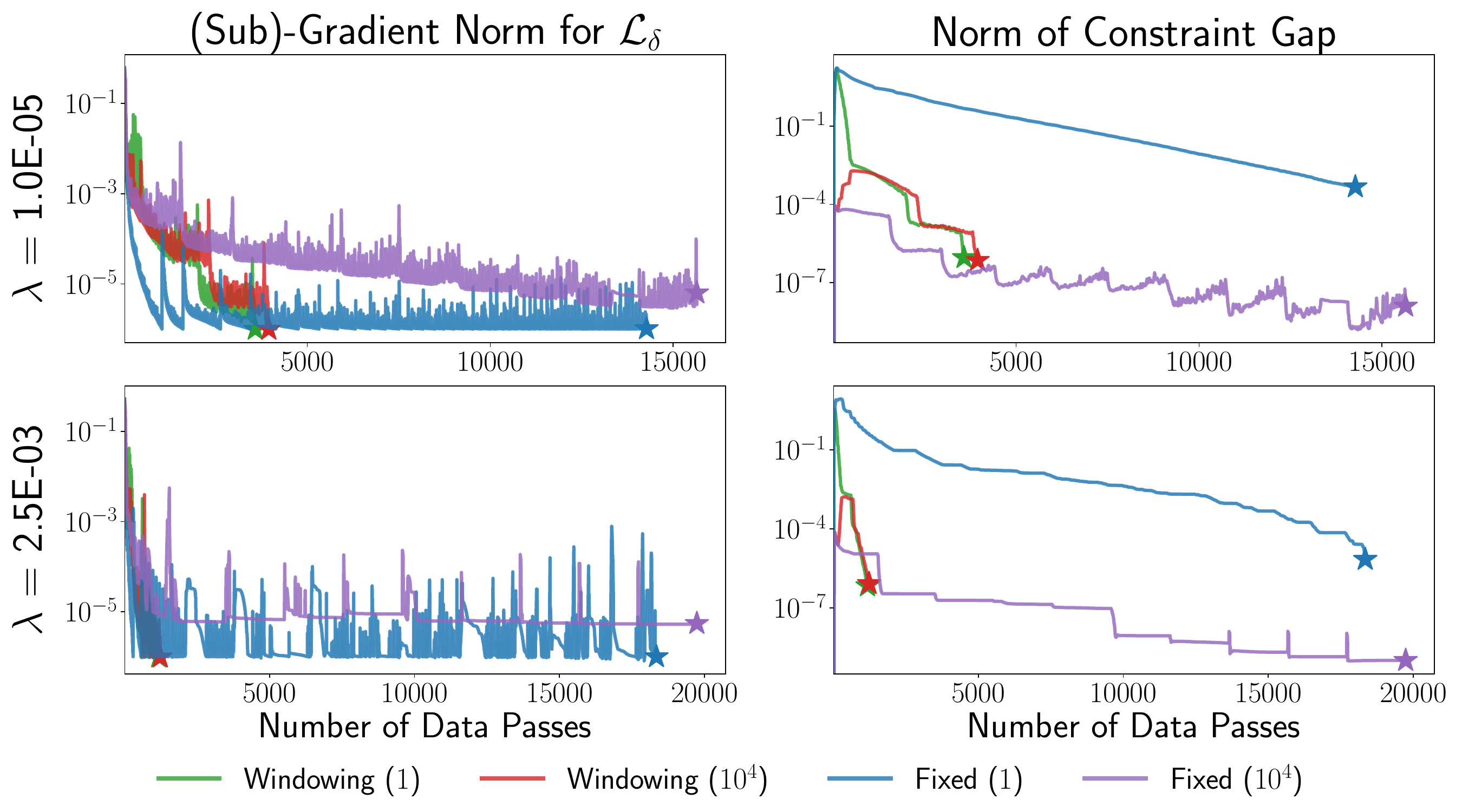}
	\fi
	\caption{Convergence comparison for our AL method with with and without the windowing heuristic on the \texttt{monks-2} dataset.
		We show two extreme values of \( \delta \) (shown in parentheses) to illustrate failure models the AL method without our heuristic.
		Roughly, each ``bump'' in the subgradient norm (squared) of \( \calL_\delta \) corresponds to one prox-point iteration on the dual parameters (i.e.\ an AL update).
		Observe that when \( \delta \) is small, R-FISTA solves sub-problem~\eqref{eq:al-subroutine} quickly, but the AL method cannot find a feasible solution without an extreme number of dual updates.
		In contrast, R-FISTA never solves~\eqref{eq:al-subroutine} to the first-order tolerance (\( 10^{-6} \))  when \( \delta \) is very large.
		In this case, dual updates are triggered only by a limit on the number of R-FISTA iterations for solving the sub-problem. See Appendix~\ref{app:default-parameters}.
		The windowing heuristic corrects both failure modes.
	}%
	\label{fig:delta-heuristic-monks}
\end{figure}

\begin{figure}
	\centering
	\ifdefined\smallPDF
		\includegraphics[width=0.95\linewidth]{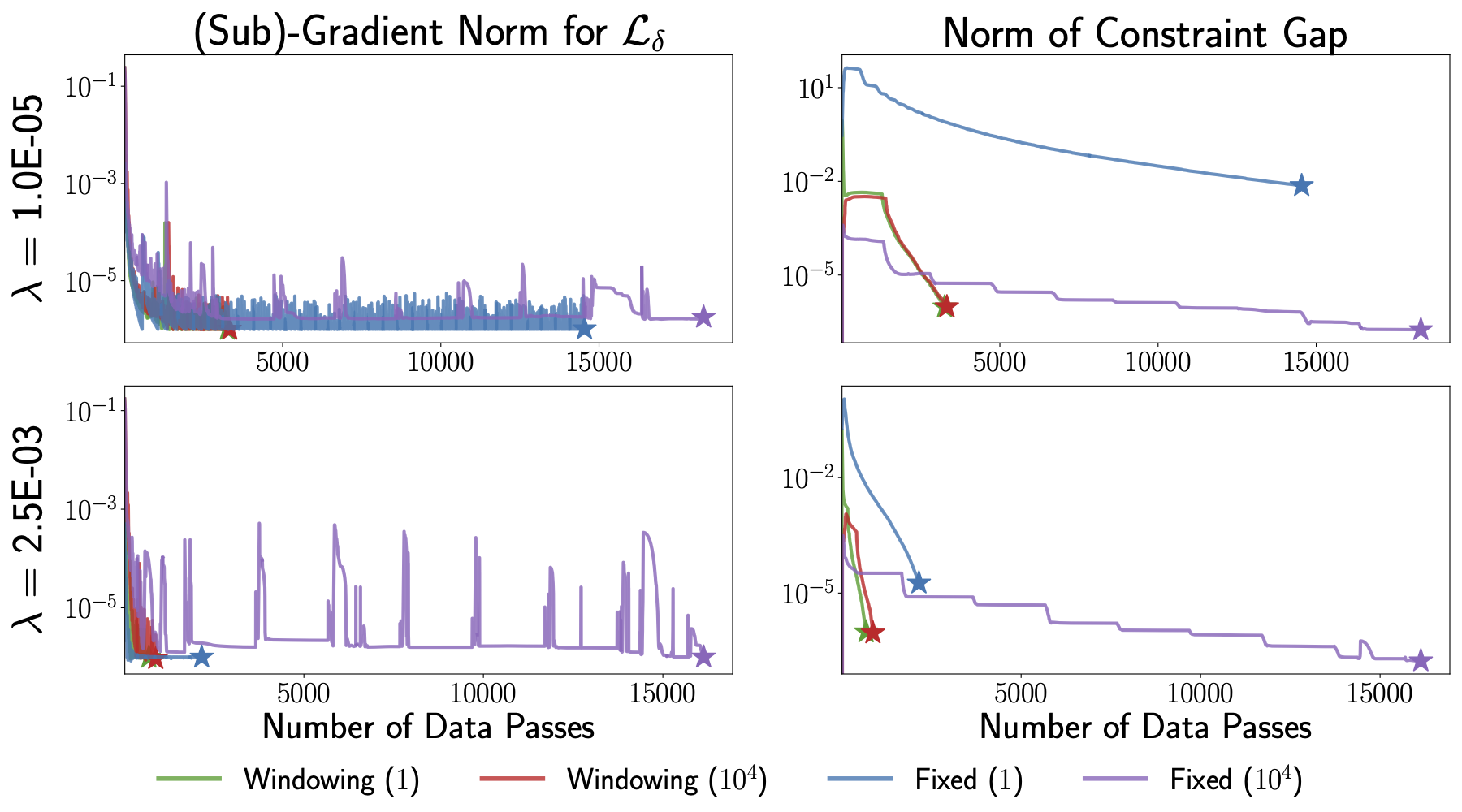}
	\else
		\includegraphics[width=0.95\linewidth]{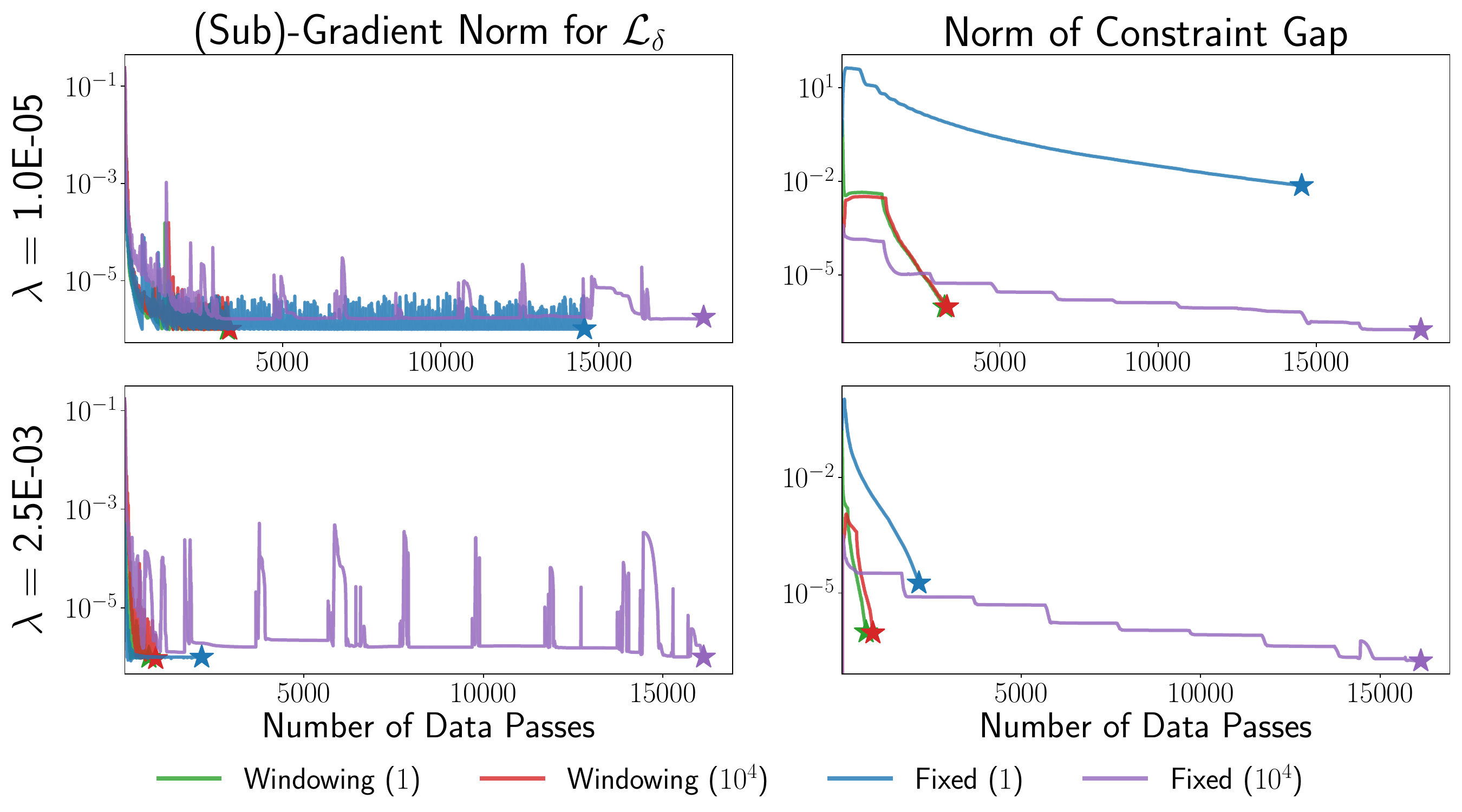}
	\fi
	\caption{Convergence comparison for our AL method with with and without the windowing heuristic on the \texttt{ilpd-indian-liver} dataset.
		Penalty parameters \( \delta \) are reported in parenthesis. }%
	\label{fig:delta-heuristic-ilpd-indian-liver}
\end{figure}

We also provide convergence plots on two randomly selected datasets to better illustrate the failures modes of the AL method with miss-specified penalty strength.
Figures~\ref{fig:delta-heuristic-monks} and~\ref{fig:delta-heuristic-ilpd-indian-liver} and show detailed results for the \texttt{monks-2} and \texttt{ilpd-indian-liver} datasets.
When \( \delta \) is too small, the AL method easily solves subproblem~\eqref{eq:al-subroutine}, but struggles to make progress on the constraint gaps.
Intuitively, the step-size for the dual proximal-point algorithm is too small and a very large number of iterations is required to make progress on the dual problem.
Conversely, the augmented Lagrangian \( \calL_\delta \) is poorly conditioned when \( \delta \) is overly large and R-FISTA struggles to solve the primal sub-problem to the necessary tolerance.
The windowing heuristic corrects for both pathologies by ensuring the initial constraint gap is in a ``normal'' regime that balances penalizing constraint violations and conditioning of the subproblem.
This behavior is particularly noticeable for \texttt{monks-2}, where the windowing heuristic adjusts \( \delta \) to shrink the constraint gap (\( \delta = 1 \)) or relax the optimization problem (\( \delta = 10^4 \)).


\section{Sensitivity and Regularization}\label{app:activation-pattern-ablations}

This section presents additional ablations studying the sensitivity of the C-ReLU and C-GReLU problems to the selection of the sub-sampled activation patterns, \( \tilde \calD \), and the regularization strength, \( \lambda \).

\textbf{Experimental Details}:
We randomly select 10 datasets from our set of 73 filtered UCI datasets (see \cref{app:uci-datasets}).
For each dataset, we considered thirty individual regularization parameters on log-scale grid over the interval \( [1 \times 10^{-6}, 1] \).
To form the convex formulations, we computed \( \tilde \calD \) by sampling 10, 100, or 1000 generating vectors from \( \calN(0, \bfI) \).
We repeated the sampling procedure with \( 10 \) different random seeds, giving a final total of 60 (30 C-ReLU and 30 C-GReLU) optimization problems for each dataset.
These problems were then solved using R-FISTA and our AL method with the default parameters (see \cref{app:default-parameters}).

\textbf{Additional Results}:
Figures~\ref{fig:hp-gated} and~\ref{fig:hp-relu} present results for the C-GReLU and C-ReLU problems, respectively.
Similar to \cref{fig:reg-plot}, a U-shaped bias-variance trade-off is visible as the regularization strength is increased.
This trend is especially noticeable for the \texttt{monks-3} and \texttt{statlog-heart} datasets.
Variance introduced by sampling \( \tilde \calD \) is only significant for \texttt{heart-va}.

\begin{figure}[htpb]
	\centering
	\includegraphics[width=0.85\linewidth]{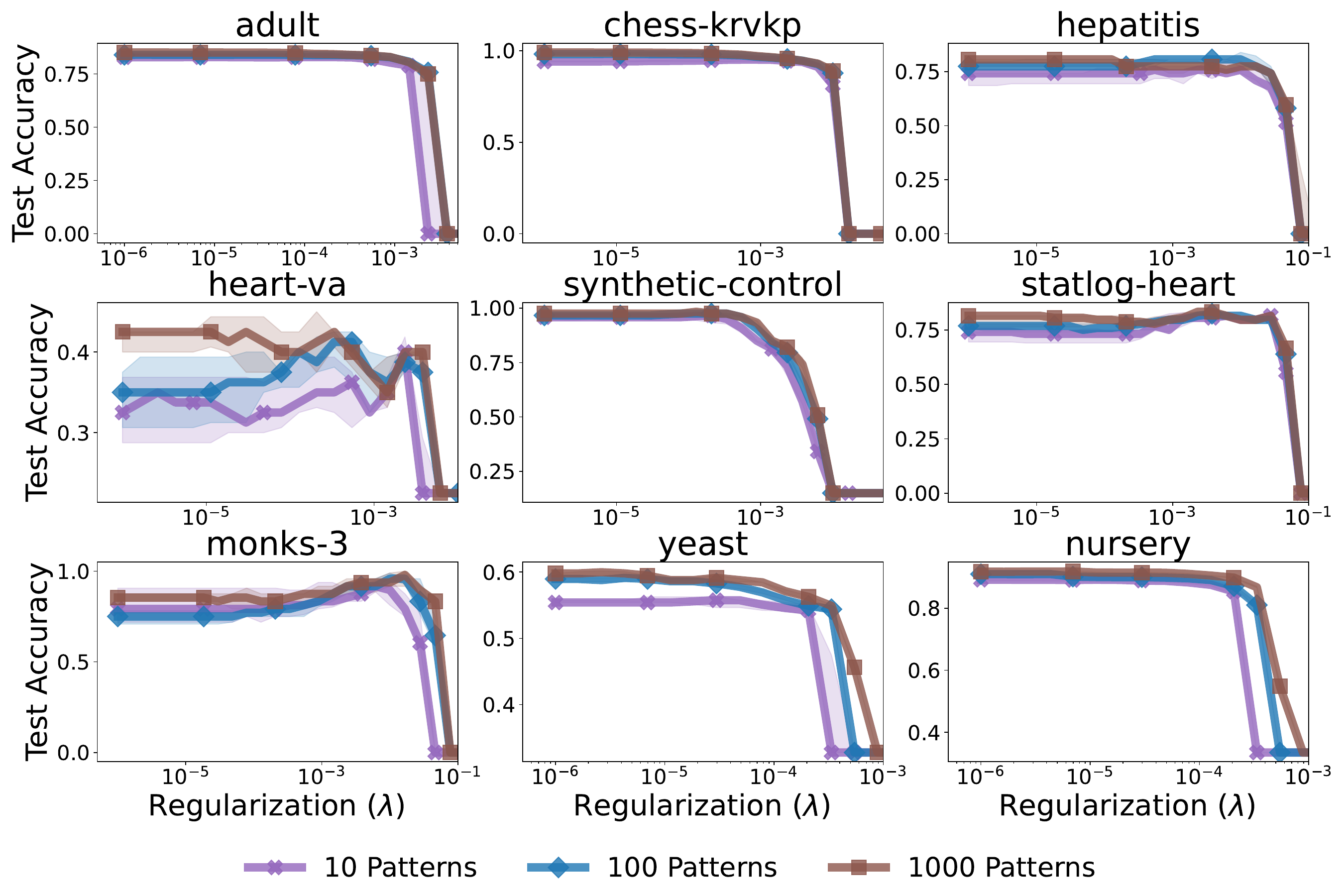}
	\caption{
		Effect of sampling activation patterns on test accuracy for neural networks trained using the \textbf{C-GReLU} problem on nine different UCI datasets.
		We consider a grid of regularization parameters and plot median (solid line) and first and third quartiles (shaded region) over 10 random samplings of \( \tilde \calD \), where \( |\tilde \calD| \) is limited to 10, 100, or 1000 patterns.
	}%
	\label{fig:hp-gated}
\end{figure}

\begin{figure}[htpb]
	\centering
	\includegraphics[width=0.85\linewidth]{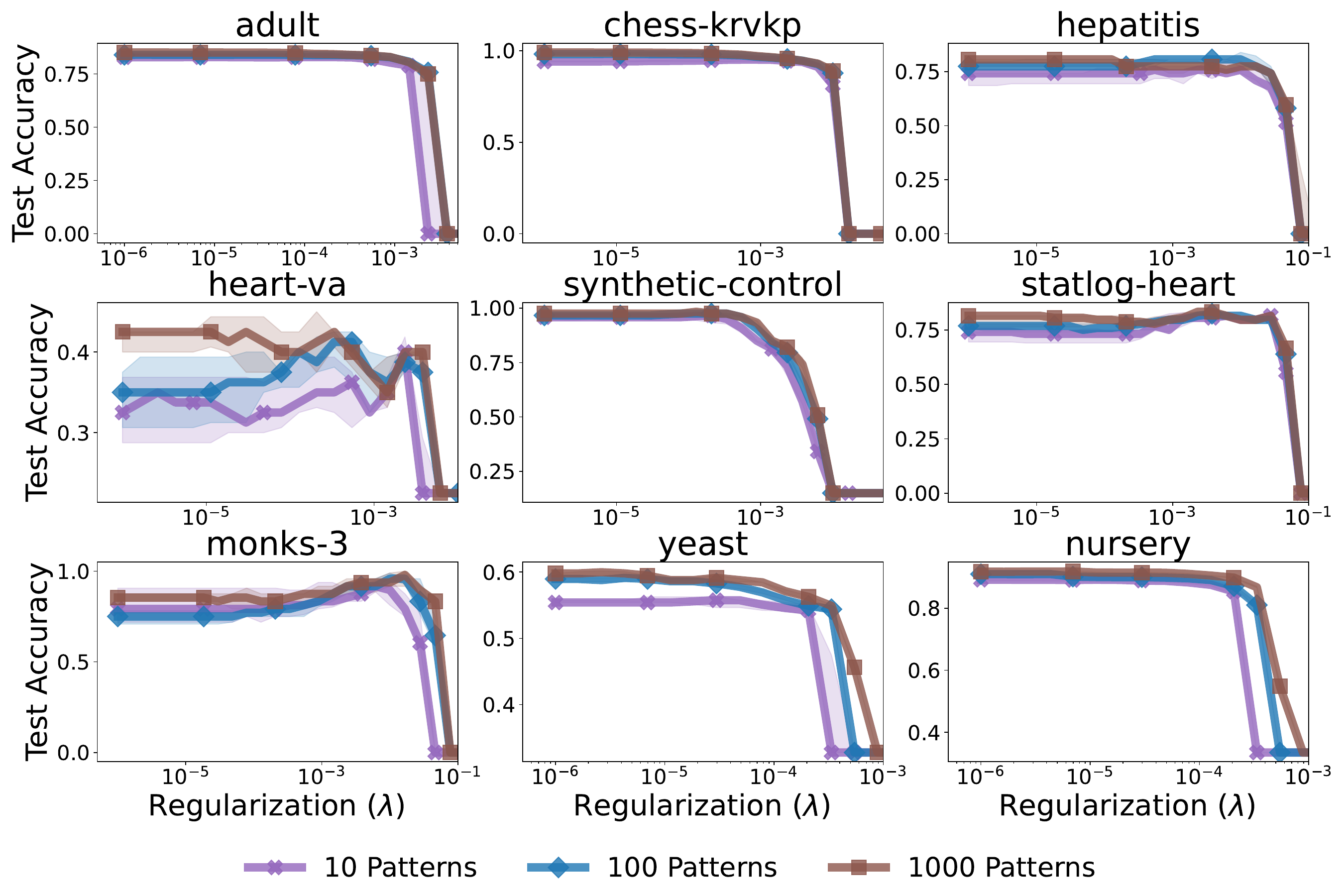}
	\caption{
		Effect of sampling activation patterns on test accuracy for neural networks trained using the \textbf{C-ReLU} problem on nine different UCI datasets. See \cref{fig:hp-gated} for details.
	}%
	\label{fig:hp-relu}
\end{figure}

\subsection{UCI Classification}\label{app:uci-accuracies}

This section gives experimental details and additional results for the experiments evaluating generalization performance of the convex reformulations.

\textbf{Experimental Details}:
We selected 37 binary classification datasets from our filtered collection of 73 datasets;
see Appendix~\ref{app:uci-datasets} for details how the 73 datasets were obtained.

We used the default parameters for each of the convex solvers as described in Appendix~\ref{app:default-parameters},
except that a tighter convergence tolerance of \( 10^{-7} \) was used for terminating our methods.
R-FISTA was limited to \( 2000 \) iterations.
For the gated ReLU problems (both C-GReLU and NC-GReLU) we sampled the same set of \( 5000 \) activation patterns for both the convex reformulation and the original non-convex model.
We used \( 2500 \) activation patterns for the C-ReLU problem.

For each dataset-method pair, we performed five-fold cross validation on the training set
to select hyper-parameters.
We considered two hyper-parameters for our methods: regularization strength,
and the proportion of examples active in each local model (ie. the number of
non-zeros in each \( D_i \) matrix).
For regularization strength, we optimized over a logarithmic grid with values
	{ \small \(
		\cbr{1\times 10^{-8},
			3.59\times 10^{-8},
			1.29\times 10^{-7},
			4.64\times 10^{-7},
			1.67\times 10^{-6},
			5.99\times 10^{-6},
			2.15\times 10^{-5},
			7.74\times 10^{-5},
			2.78\times 10^{-4},
			1.0\times 10^{-3}}
		\)
	}.
For the proportion of active examples, we considered (1) setting the bias term
for each neuron to enforce \( 50 \% \) of examples to be active or (2)
setting the bias to \( 0 \) and allowing the proportion to be random.

For the baselines, we used the implementations available from
the \texttt{scikit-learn} package.
We optimized each random forest classifier with respect to the depth of
the random trees in the ensemble (\( \cbr{2, 4, 10, 25, 50}\))
and over the number of trees in the ensemble (\( \cbr{5, 10, 100, 1000} \)).
We used the standard soft-margin SVM with and chose regularization parameter from
the range
\( \cbr{1.\times 10^{-5}, 1.\times 10^{-4}, 1.\times 10^{-3}, 1.\times 10^{-2}, 1.\times 10^{-1}, 1} \)
for linear SVMs and
\( \cbr{1\times 10^{-5},
	1.78\times 10^{-4},
	3.16\times 10^{-3},
	5.62\times 10^{-2},
	1} \)
for SVMs with an RBF kernel.
For RBF SVMs, the RBF bandwidth was optimized over the grid \( \cbr{1\times 10^{-4},
	1.58\times 10^{-3},
	2.51\times 10^{-2},
	3.98\times 10^{-1},
	6.31,
	100.0} \).
To obtain final test accuracies, we re-trained each method on the full training
set.
For our methods, we report the best test accuracy out of five random restarts.

\begin{table}[t]
	\centering
	\caption{Additional results comparing our convex solvers against random forests (RF), SVMs with a linear kernel (Linear)
		and SVMs with an RBF kernel (RBF).
		We report test accuracies on a further 19 UCI datasets.
		Combined, C-GReLU and C-ReLU obtain the best test accuracy on 10 datasets.
		Out of the baselines considered, RBF SVMs are the most competitive with our approach,
		attaining or tying for best accuracy on 9 datasets.
	}%
	\label{table:binary-uci-accuracies-alt}
	\vspace{0.1in}
	\begin{small}
		\begin{tabular}{lccccc} \toprule
			\textbf{Dataset}  & \textbf{C-GReLU} & \textbf{C-ReLU} & \textbf{RF}   & \textbf{SVM}  & \textbf{RBF}  \\ \midrule
			breast-cancer     & 73.7             & 70.2            & \textbf{75.4} & 68.4          & 68.4          \\
			congressional     & 66.7             & 65.5            & 65.5          & \textbf{67.8} & \textbf{67.8} \\
			credit-approval   & 84.1             & 84.1            & \textbf{85.5} & \textbf{85.5} & 84.8          \\
			echocardiogram    & 80.8             & 76.9            & \textbf{88.5} & 84.6          & 84.6          \\
			haberman-survival & 67.2             & \textbf{75.4}   & 70.5          & 70.5          & 70.5          \\
			hepatitis         & 80.6             & 80.6            & \textbf{83.9} & 77.4          & 77.4          \\
			horse-colic       & 88.3             & 86.7            & \textbf{93.3} & 90.0          & \textbf{93.3} \\
			ionosphere        & 90.0             & 91.4            & 91.4          & 85.7          & \textbf{97.1} \\
			molec-biol        & 76.2             & \textbf{81.0}   & 76.2          & \textbf{81.0} & 66.7          \\
			monks-2           & \textbf{69.7}    & \textbf{69.7}   & 54.5          & 57.6          & \textbf{69.7} \\
			monks-3           & \textbf{95.8}    & \textbf{95.8}   & \textbf{95.8} & 87.5          & 91.7          \\
			musk-2            & 99.7             & \textbf{99.8}   & 97.3          & 95.1          & 99.6          \\
			parkinsons        & 97.4             & 97.4            & 84.6          & 89.7          & \textbf{100}  \\
			pittsburg         & \textbf{80.0}    & 75.0            & 75.0          & \textbf{80.0} & \textbf{80.0} \\
			ringnorm          & 97.4             & 83.0            & 95.1          & 76.9          & \textbf{98.2} \\
			spect             & 46.7             & 40.0            & 53.3          & \textbf{66.7} & \textbf{66.7} \\
			statlog-austr.    & 65.9             & \textbf{66.7}   & 62.3          & 65.2          & 65.2          \\
			statlog-heart     & 81.5             & \textbf{85.2}   & 83.3          & 81.5          & 81.5          \\
			twonorm           & 97.6             & \textbf{97.7}   & 97.2          & 97.4          & 97.4          \\
			vertebral-col.    & \textbf{91.9}    & \textbf{91.9}   & 87.1          & \textbf{91.9} & 90.3          \\ \bottomrule
		\end{tabular}
	\end{small}
\end{table}

\textbf{Additional Results}:
Table~\ref{table:binary-uci-accuracies-alt} reports test results for
19 of the 37 datasets, while Table~\ref{table:binary-uci-accuracies} in the
main paper presents results for the remaining 18 datasets.
Overall, we find that two-layer neural networks trained using our convex
solvers generally perform better than the baseline methods.

\subsection{Non-Convex Solvers}\label{app:non-convex-solvers}

\begin{table*}
	\centering
	\caption{Test accuracies for convex and non-convex formulations of the Gated ReLU training problem on a subset of 20 datasets selected from the UCI dataset repository
		Results are shown as \texttt{median (first-quartile/third-quartile)} for each method.}%
	\label{table:uci-gated-accuracies}
	\vspace{0.1in}
	\begin{small}
		\begin{tabular}{lccc} \toprule
			\textbf{Dataset}     & \textbf{C-GReLU} & \textbf{NC-GReLU (Adam)} & \textbf{NC-GReLU (SGD)} \\ \midrule
			magic                & 86.9 (86.8/87.0) & 82.9 (82.9/83.1)         & 82.1 (82.1/82.2)        \\
			statlog-heart        & 79.6 (79.6/79.6) & 85.2 (83.3/85.2)         & 83.3 (83.3/83.3)        \\
			mushroom             & 100.0 (100/100)  & 97.6 (97.6/97.9)         & 96.9 (96.9/96.9)        \\
			vertebral-column     & 87.1 (83.9/87.1) & 90.3 (90.3/91.9)         & 90.3 (90.3/90.3)        \\
			cardiotocography     & 90.1 (89.9/90.4) & 85.6 (85.6/85.9)         & 85.2 (85.2/85.4)        \\
			abalone              & 63.8 (63.7/64.1) & 58.7 (58.6/58.7)         & 58.1 (58.1/58.1)        \\
			annealing            & 90.6 (90.6/91.2) & 86.2 (86.2/86.8)         & 86.2 (85.5/86.2)        \\
			car                  & 89.9 (89.9/90.1) & 83.8 (83.8/84.1)         & 83.2 (82.9/83.2)        \\
			bank                 & 89.8 (89.7/89.9) & 89.9 (89.9/90.0)         & 89.8 (89.8/90.0)        \\
			breast-cancer        & 68.4 (68.4/68.4) & 68.4 (68.4/70.2)         & 70.2 (70.2/70.2)        \\
			page-blocks          & 96.8 (96.8/96.9) & 92.1 (92.0/92.1)         & 92.4 (92.3/92.4)        \\
			contrac              & 45.9 (45.6/46.3) & 53.1 (53.1/53.1)         & 53.4 (53.1/53.7)        \\
			congressional-voting & 63.2 (63.2/63.2) & 64.4 (64.4/64.4)         & 66.7 (66.7/66.7)        \\
			spambase             & 93.4 (93.2/93.4) & 91.6 (91.6/91.6)         & 91.2 (91.2/91.3)        \\
			synthetic-control    & 97.5 (97.5/97.5) & 98.3 (98.3/98.3)         & 97.5 (97.5/98.3)        \\
			musk-1               & 93.7 (91.6/93.7) & 93.7 (93.7/93.7)         & 94.7 (92.6/94.7)        \\
			ringnorm             & 69.8 (69.5/69.9) & 77.0 (77.0/77.0)         & 77.2 (77.1/77.2)        \\
			ecoli                & 82.1 (82.1/82.1) & 79.1 (79.1/80.6)         & 4.5 (3.0/43.3)          \\
			monks-2              & 69.7 (66.7/69.7) & 66.7 (66.7/66.7)         & 60.6 (57.6/63.6)        \\
			hill-valley          & 62.0 (59.5/66.1) & 57.0 (55.4/57.9)         & 58.7 (58.7/59.5)        \\ \bottomrule
		\end{tabular}
	\end{small}
\end{table*}

This section gives experimental details and additional results for experiments comparing the generalization performance of our convex reformulations to neural networks trained by optimizing the non-convex loss with stochastic gradient methods.

\textbf{Experimental Details}:
We selected 20 datasets randomly from our filtered collection of 73 datasets; see Appendix~\ref{app:uci-datasets} for details how the 73 datasets were obtained.

We used the default parameters for each of the convex solvers as described in Appendix~\ref{app:default-parameters}.
R-FISTA was limited to \( 2000 \) iterations, while SGD and Adam were limited to \( 2000 \) epochs.
For SGD and Adam, considered step-sizes from the following grid: \( \cbr{10, 5, 1, 0.5, 0.1, 0.01, 0.001} \).
We used a ``step'' decrease schedule for the step-sizes, dividing them by \( 2 \) every 100 epochs, which we found to work much better than the classical Robbins-Monro schedule~\citep{robbins1951sgd}.
We considered the following grid of ten regularization parameters:
\( \{
1 \times 10^{-6},
2.78 \times 10^{-6},
7.74 \times 10^{-6},
2.15 \times 10^{-5},
5.99 \times 10^{-5},
1.67 \times 10^{-4},
4.64 \times 10^{-4},
1.29 \times 10^{-3}
3.59 \times 10^{-2},
1.0 \times 10^{-2} \}
\).
For each method-dataset pair, we performed five-fold cross validation on the training set and selected the best step-size and regularization parameter according to the cross-validated test accuracy.

For the gated ReLU problems (both C-GReLU and NC-GReLU) we sampled the same set of \( 5000 \) activation patterns for both the convex reformulation and the original non-convex model.
We used \( 2500 \) activation patterns for the C-ReLU problem.
To ensure a similar model space, we computed at the number of active neurons (e.g. \( v_i \neq 0 \) or \( w_i \neq 0 \)) at convergence for C-ReLU and then used this as the number of hidden units for \( NC-ReLU \) problems.
Note that we extend our convex reformulations to multi-class problems using the results in Appendix~\ref{app:multi-class-extension}.
Similarly, we use the vector-output variant of the NC-ReLU problem (Eq.~\ref{convex-forms:eq:vector-non-convex-relu-mlp}) for multi-class problems.

After selecting hyper-parameters, we obtain the final test accuracies by re-training on the full training set and testing on a held-out test set.
To control for noise in the sampling of gate vectors in the Gated ReLU problems and C-ReLU, we repeat this final testing procedure five times with different random seeds.

\begin{table*}
	\centering
	\caption{Test accuracies for convex and non-convex formulations of the ReLU training problem on a subset of 20 datasets selected from the UCI dataset repository
		Results are shown as \texttt{median (first-quartile/third-quartile)} for each method.}%
	\label{table:uci-relu-accuracies}
	\vspace{0.1in}
	\begin{small}
		\begin{tabular}{lccc} \toprule
			\textbf{Dataset}     & \textbf{C-ReLU}  & \textbf{NC-ReLU (Adam)} & \textbf{NC-ReLU (SGD)} \\ \midrule
			magic                & 85.9 (85.8/85.9) & 86.9 (86.9/86.9)        & 86.4 (86.3/86.4)       \\
			statlog-heart        & 83.3 (81.5/83.3) & 83.3 (83.3/83.3)        & 79.6 (79.6/79.6)       \\
			mushroom             & 100.0 (100/100)  & 100.0 (100/100)         & 99.9 (99.9/99.9)       \\
			vertebral-column     & 90.3 (88.7/90.3) & 90.3 (90.3/90.3)        & 88.7 (88.7/88.7)       \\
			cardiotocography     & 89.9 (89.9/89.9) & 36.5 (22.8/36.5)        & 88.9 (88.9/88.9)       \\
			abalone              & 66.2 (66.1/66.3) & 65.3 (64.9/65.4)        & 66.1 (66.1/66.1)       \\
			annealing            & 90.6 (89.9/90.6) & 93.7 (93.7/93.7)        & 88.7 (88.1/88.7)       \\
			car                  & 87.8 (87.8/87.8) & 94.8 (94.8/94.8)        & 90.1 (90.1/90.1)       \\
			bank                 & 89.8 (89.7/89.9) & 90.8 (90.8/90.9)        & 90.5 (90.5/90.5)       \\
			breast-cancer        & 68.4 (66.7/68.4) & 64.9 (64.9/64.9)        & 68.4 (68.4/68.4)       \\
			page-blocks          & 94.0 (94.0/94.0) & 97.1 (97.1/97.1)        & 96.9 (96.9/96.9)       \\
			contrac              & 55.1 (54.1/55.4) & 54.4 (54.1/54.4)        & 53.7 (53.7/53.7)       \\
			congressional-voting & 63.2 (63.2/65.5) & 62.1 (62.1/62.1)        & 67.8 (67.8/67.8)       \\
			spambase             & 93.3 (93.2/93.4) & 93.5 (93.5/93.5)        & 93.2 (93.2/93.2)       \\
			synthetic-control    & 98.3 (97.5/98.3) & 96.7 (96.7/96.7)        & 96.7 (96.7/96.7)       \\
			musk-1               & 93.7 (93.7/93.7) & 96.8 (96.8/96.8)        & 95.8 (95.8/95.8)       \\
			ringnorm             & 77.0 (76.8/77.0) & 77.3 (77.3/77.4)        & 77.4 (77.3/77.5)       \\
			ecoli                & 80.6 (80.6/80.6) & 82.1 (82.1/82.1)        & 80.6 (80.6/80.6)       \\
			monks-2              & 69.7 (69.7/72.7) & 69.7 (66.7/69.7)        & 72.7 (72.7/75.8)       \\
			hill-valley          & 65.3 (64.5/65.3) & 62.8 (62.8/62.8)        & 55.4 (55.4/55.4)       \\ \bottomrule
		\end{tabular}
	\end{small}
\end{table*}

\textbf{Additional Results}:
Tables~\ref{table:uci-gated-accuracies} and~\ref{table:uci-relu-accuracies} report median test accuracies as well as first and third quartiles for the convex and non-convex formulations with gated ReLU and ReLU activations, respectively.
Note that these results are identical to those the provided in the main paper (\cref{table:non-convex-solvers}) but for the inclusion of variance/distribution information in the form of quartiles.


\section{Image Classification}\label{app:image-datasets}

\textbf{Experimental Details}:
The MNIST and CIFAR-10 datasets are high-dimensional, with \((n, d) = (60000, 784)\) and \((50000,3072)\), respectively.
As such, we require a large number of neurons for both problems, for which we use \(m=5000\) and \(m=4000\) neurons respectively.
Both datasets are normalized column-wise, and squared loss is used as the objective.
Activation patterns are generated by sampling \( u_i \) from a distribution that samples a \(3 \times 3\) patch uniformly from the image, then sampling values for that patch from a standard Gaussian distribution, with all other values set to zero.
This technique is used for both convex and non-convex architectures.
We use the extensions of the C-GReLU and NC-ReLU to multi-class problems as given in \cref{app:multi-class-extension}.

For the NC-GReLU experiments, for all optimizers, we consider a learning rate of \(1.0, 0.1, 0.01\).
We use a momentum parameter of 0.9 for SGD.
To improve convergence, the step size was decayed by a factor of 2 every 200 epochs, and the networks were trained for a maximum of 1000 epochs.
We use a batch size of 10\% of the training data.
For the C-GReLU experiments, no R-FISTA optimizer parameters are tuned--we fix the initial step size to 0.1, with quadratic backtracking with \(\beta = 0.8\), and forward-tracking with \(\alpha = 1.2\) and \(c=5\).
For all methods, we consider regularization parameters \(\lambda \in [10^{-3}, 10^{-4}, 10^{-5}, 10^{-6}, 10^{-7}]\), and choose the one with the best accuracy on the validation set, which is chosen to be a random subset of 20\% of the training data.
All models are trained with an NVIDIA Titan X GPU with 12GB RAM.
\\\\
For G-ReLU, a value of \(\lambda=10^{-7}\) was chosen for MNIST and \(\lambda=10^{-6}\) for CIFAR-10.
For SGD, values of \((\eta, \lambda) = (1.0, 10^{-7})\) were chosen for MNIST and \((\eta, \lambda) = (1.0, 10^{-5})\) for CIFAR-10.
For Adam, values of \((\eta, \lambda) = (0.01, 10^{-6})\) were chosen for MNIST and \((\eta, \lambda) = (0.01, 10^{-4})\) for CIFAR-10.
For Adagrad, values of \((\eta, \lambda) = (0.01, 10^{-7})\) were chosen for MNIST and \((\eta, \lambda) = (0.01, 10^{-5})\)  for CIFAR-10.



\section{Default Optimization Parameters}\label{app:default-parameters}

In this section, we report the standard parameter settings for our optimizers.
We use these parameters in all experiments unless explicitly stated otherwise.
Note that data normalization (Appendix~\ref{app:data-normalization}) is applied in all experiments for both the convex and non-convex training problems.

\subsection{R-FISTA}

We set the backtracking parameter to \( \beta = 0.8 \) and the forward-tracking parameter to \( \alpha = 1.25 \).
For the step-size initialization strategy, we set the threshold to be \( c = 5.0 \).
We set the first step-size to be \( \eta_0 = 1.0 \).
The restart strategy detailed in Section~\ref{sec:efficient-fista} is always used unless it is explicitly stated otherwise.
Finally, we consider the optimizer to have (approximately) converged when the minimum-norm subgradient has \( \ell_2 \)-norm less than or equal to \( 10^{-3} \).
In practice we check the equivalent condition on the squared gradient norm with the threshold \( 10^{-6} \).
We always initialize the model weights as \( v_i = 0 \) for each \( D_i \in \tilde \calD \).

\subsection{AL Method}

We set the initial penalty parameter to be \( \delta = 100 \) and use the windowing heuristic with \( r_u = 10^{-2} \) and \( r_l = 10^{-3} \).
The dual parameters are initialization at \( 0 \), as are the primal parameters.
Note that we always warm-start the optimization of the augmented Lagrangian at the solution to the previous iteration's optimization problem.
The convergence tolerance when checking for satisfaction of the windowing heuristic is set to be \( \text{tol} = 10^{-3/2} \).
If the constraint gap is larger than \( r_u \), we increase \( \delta \) as \( \delta \gets 2 * \delta \) and repeat the procedure.
If \( c_{\text{gap}} < r_l \), we set \( \delta \gets \delta / 2 \) and also change the convergence to be \( \text{tol} \gets \text{tol} / 2 \).

The convergence tolerance for minimization of the augmented Lagrangian once the window heuristic is satisfied is \( \text{tol} = 10^{-3} \).
We consider the AL method to have approximately converged when \( c_{\text{gap}} \leq 10^{-3} \) and the minimum norm subgradient of the augmented Lagrangian (with respect to the primal parameters) is also less than \( 10^{-3} \).

To solve Equation~\eqref{eq:al-subroutine}, we use R-FISTA with the standard configuration as outlined in the previous section.
We enforce a maximum of \( 1000 \) iterations for the sub-solver, meaning that we execute a step of the AL method after at most \( 1000 \) iterations of R-FISTA regardless of the termination tolerances.
In general, we permit as many ``outer'' iterations of the AL method as needed since these do not require gradient computations, but limit the overall optimization procedure to \( 10000 \) iterations of R-FISTA.


\section{UCI Datasets}\label{app:uci-datasets}

We use the binary and multi-class classification datasets from the UCI machine learning repository~\citep{dua2019uci} as pre-processed by \citet{delgado2014hundreds}.
Note that we do not use the same training/validation/test procedure as \citet{delgado2014hundreds}, since this is known to have test-set leakage.
We applied the following selection rules to decide which datasets to retain for our experiments:
\begin{itemize}
	\item at least 150 examples and 5 features;
	\item no more than 50000 examples and 10 classes;
	\item no duplicated datasets with different targets or features.
\end{itemize}
This left the following 73 datasets from the original collection of 121:
\texttt{abalone},
\texttt{adult},
\texttt{annealing},
\texttt{bank},
\texttt{breast-cancer},
\texttt{breast-cancer-wisc-diag},
\texttt{car},
\texttt{cardiotocography-3clases},
\texttt{chess-krvkp},
\texttt{congressional-voting},
\texttt{conn-bench-sonar-mines-rocks},
\texttt{contrac},
\texttt{credit-approval},
\texttt{cylinder-bands},
\texttt{dermatology},
\texttt{ecoli},
\texttt{energy-y1},
\texttt{flags},
\texttt{glass},
\texttt{heart-cleveland},
\texttt{heart-hungarian},
\texttt{heart-va},
\texttt{hepatitis},
\texttt{hill-valley},
\texttt{horse-colic},
\texttt{ilpd-indian-liver},
\texttt{image-segmentation},
\texttt{ionosphere},
\texttt{led-display},
\texttt{low-res-spect},
\texttt{magic},
\texttt{mammographic},
\texttt{molec-biol-splice},
\texttt{monks-2},
\texttt{monks-3},
\texttt{mushroom},
\texttt{musk-1},
\texttt{musk-2},
\texttt{nursery},
\texttt{oocytes\_merluccius\_nucleus\_4d},
\texttt{oocytes\_trisopterus\_nucleus\_2f},
\texttt{optical},
\texttt{ozone},
\texttt{page-blocks},
\texttt{parkinsons},
\texttt{pendigits},
\texttt{pima},
\texttt{planning},
\texttt{primary-tumor},
\texttt{ringnorm},
\texttt{seeds},
\texttt{semeion},
\texttt{spambase},
\texttt{statlog-australian-credit},
\texttt{statlog-german-credit},
\texttt{statlog-heart},
\texttt{statlog-image},
\texttt{statlog-landsat},
\texttt{statlog-vehicle},
\texttt{steel-plates},
\texttt{synthetic-control},
\texttt{teaching},
\texttt{thyroid},
\texttt{tic-tac-toe},
\texttt{twonorm},
\texttt{vertebral-column-2clases},
\texttt{wall-following},
\texttt{waveform},
\texttt{waveform-noise},
\texttt{wine},
\texttt{wine-quality-red},
\texttt{wine-quality-white},
\texttt{yeast}

\textbf{Optimization Performance}:
For our experiments evaluating optimization performance, we considered all 73 datasets and generated \( 6 * 73 = 438 \) optimization problems by considering the following grid of regularization parameters:
\[
	\lambda \in \cbr{1 \times 10^{-5},
		6.31 \times 10^{-5},
		3.98 \times 10^{-4},
		2.51 \times 10^{-3},
		1.58 \times 10^{-2},
		1.0 \times 10^{-1}}
\]
We did a single train/test split for each dataset and report optimization metrics on the training set only.
The test was used for heuristic ``sanity checks'' of the final models.

\textbf{Model Performance}:
For our experiments evaluating generalization or test performance of different models, we randomly selected a subset of the filtered UCI datasets.
We report the regularization parameters considered for each experiment in the appropriate section.

\end{document}